\documentclass{article}

\PassOptionsToPackage{numbers, compress}{natbib}
\usepackage[final]{neurips_2020}

\usepackage[utf8]{inputenc} % allow utf-8 input
\usepackage[T1]{fontenc}    % use 8-bit T1 fonts
\usepackage{nicefrac}       % compact symbols for 1/2, etc.
\usepackage{microtype}      % microtypography

\usepackage[ruled,vlined,linesnumbered,procnumbered,longend]{algorithm2e}
\usepackage{bm}

\usepackage{amsmath,amssymb}
\usepackage{bbm}
\usepackage{mathabx}

\usepackage{etex}
\usepackage{cvpr-abbriv}
\usepackage{array}
\usepackage{url}
\usepackage{graphicx}
\usepackage{amsmath,amssymb} % define this before the line numbering.
\makeatletter
\let\proof\@undefined                        % undefine \proof
\let\endproof\@undefined                  % undefine \endproof
\makeatother

\usepackage{amsthm}
\usepackage{amsxtra}
\usepackage{stmaryrd}
\usepackage{multirow,rotating}
\usepackage{afterpage}
\usepackage{dblfloatfix}
\usepackage{setspace}
\usepackage{booktabs}
\usepackage{color, colortbl}
\usepackage{tikz}                 % create images using pgf/tikz
\usetikzlibrary{patterns,fit,external}
\usepackage{pgfplots}             % draw plots in tikz style
\usepackage{overpic}
\usepackage{tocloft}
\usepackage{bbm}
\usepackage[absolute]{textpos}
\usepackage{atbeginend}
\usepackage{units}
\usepackage{xifthen}
\usepackage{chngcntr}

\usepackage{times}
\usepackage{tikzscale}
\usepackage{marginnote}
\usepackage[textwidth=2cm,figwidth=\columnwidth,colorinlistoftodos]{todonotes}
\makeatletter
\renewcommand{\todo}[2][]{\tikzexternaldisable\@todo[#1]{#2}\tikzexternalenable}
\makeatother

\newcounter{mycomment} % Usage:  \mycomment[CR]{Schreib was gscheits}

\usepackage{booktabs}
\usepackage{multirow}
\usepackage{array}
\usepackage{longtable}
\usepackage{tabu} %% !!! this is better!
\usepackage[most]{tcolorbox}
\usepackage{pbox}
\newlength{\luw}
\newlength{\luh}

\usepackage{xifthen}

\usepackage{hyperref}
\hypersetup{draft=false, pagebackref=false,breaklinks=true,colorlinks,bookmarks=false,pdftex,hidelinks,colorlinks=false,linkbordercolor={1 1 1}}
\DeclareMathAlphabet{\mathcalligra}{T1}{calligra}{m}{n}
\DeclareMathAlphabet{\mathantt}{OT1}{antt}{li}{it}
\DeclareMathAlphabet{\mathpzc}{OT1}{pzc}{m}{it}

\DeclareMathOperator{\sign}{sign}

\DeclareMathOperator{\detach}{detach}

\renewcommand{\mid}{\,|\,}

\newcommand{\Real}{\mathbb{R}}

\newcommand{\Bool}{\mathbb{B}}

\newcommand{\E}{\mathbb{E}}
\renewcommand{\Pr}{\mathbb{P}}

\def\T{^{\mathsf T}}

\newcounter{myRomanCounter}

\newcommand{\gray}{\color[rgb]{0.5,0.5,0.5}}
\newcommand{\red}{\color[rgb]{1,0,0}}

\newlength{\figwidth}

\newcommand{\IF}{\mbox{ \rm if }}

\newcommand{\OTHERWISE}{\mbox{ \rm otherwise}}




\newcommand{\revisit}[1][]{%
\ifthenelse{\equal{#1}{}}{% no argument
\ensuremath{\red \triangle}\xspace}{%
{\ensuremath{\red \rhd}\xspace}%
{\gray #1}%
{\ensuremath{\red \lhd}\xspace}%
}%
}

\def\anchor [#1]#2{%
\phantomsection{}#1\label{#2}%
\def\arga{#2}%
\global\expandafter\def\csname#2\endcsname{%
\hyperref[#2]{#1}\xspace%
}%
}%

\def\codefunction [#1]#2{%
\phantomsection{}\label{#2}{\ttfamily #1\xspace}%
\def\arga{#2}%
\global\expandafter\def\csname#2\endcsname{%
\hyperref[#2]{\ttfamily #1}\xspace%
}%
}

\usepackage{array}
\newcolumntype{L}[1]{>{\raggedright\let\newline\\\arraybackslash\hspace{0pt}}m{#1}}
\newcolumntype{C}[1]{>{\centering\let\newline\\\arraybackslash\hspace{0pt}}m{#1}}
\newcolumntype{R}[1]{>{\raggedleft\let\newline\\\arraybackslash\hspace{0pt}}m{#1}}

\SetKwFor{forever}{while true}{}{end while}

\def\={\,{=}\,}

\newlength{\myskip}
\SetAlgoSkip{skipalgomyskip}
\SetAlgoInsideSkip{}
\makeatletter
\let\corollary\@undefined
\let\c@corollary\@undefined
\let\endcorollary\@undefined
\let\definition\@undefined
\let\c@definition\@undefined
\let\enddefinition\@undefined
\let\proof\@undefined
\let\endproof\@undefined
\let\theorem\@undefined
\let\c@theorem\@undefined
\let\endtheorem\@undefined
\let\lemma\@undefined
\let\c@lemma\@undefined
\let\endlemma\@undefined
\let\example\@undefined
\let\c@example\@undefined
\let\endexample\@undefined
\let\remark\@undefined
\let\c@remark\@undefined
\let\endremark\@undefined
\let\proposition\@undefined
\let\c@proposition\@undefined
\let\endproposition\@undefined
\let\property\@undefined
\let\endproperty\@undefined
\makeatother

\usepackage{thmtools}

\newtheoremstyle{tightItalic}% name
  {0.5\myskip}%      Space above
  {0\myskip}%      Space below
  {}%         Body font
  {}%         Indent amount (empty = no indent, \parindent = para indent)
  {\itshape}% Thm head font
  {.}%        Punctuation after thm head
  { }%     Space after thm head: " " = normal interword space;
  {}%         Thm head spec (can be left empty, meaning `normal')

\newtheoremstyle{tightBf}% name
  {0.5\myskip}%      Space above
  {0.5\myskip}%      Space below, 0 since amsart starts a paragraph
  {}%         Body font
  {}%         Indent amount (empty = no indent, \parindent = para indent)
  {\bf}% Thm head font
  {.}%        Punctuation after thm head
  {.5em}%     Space after thm head: " " = normal interword space;
  {}%         Thm head spec (can be left empty, meaning `normal')

\theoremstyle{definition}
\theoremstyle{tightBf}

\declaretheorem[style=tightBf,name=Theorem]{theorem}
\declaretheorem[style=tightBf,name=Lemma]{lemma}
\declaretheorem[style=tightBf,name=Definition]{definition}
\declaretheorem[style=tightBf,name=Corollary]{corollary}
\declaretheorem[style=tightBf,name=Proposition]{proposition}

\declaretheorem[style=tightBf,name=Example]{example}

\declaretheorem[style=tightItalic,name=Remark]{remark}

\theoremstyle{tightItalic}
\newtheorem*{proof}{Proof}
\BeforeEnd{proof}{\qed}


\usepackage{thmtools, thm-restate}
\usepackage[capitalise,nameinlink]{cleveref}
\crefformat{equation}{(#2#1#3)}
\crefname{observation}{Observation}{Observations}
\usepackage[bottom]{footmisc}

\usepackage{listings}

\definecolor{dkgreen}{rgb}{0,0.6,0}
\definecolor{gray}{rgb}{0.5,0.5,0.5}
\definecolor{mauve}{rgb}{0.58,0,0.82}
\DeclareMathOperator{\softmax}{softmax}
\DeclareMathOperator{\sgn}{sgn}
\addtocontents{toc}{\protect\setcounter{tocdepth}{0}}

\def\REINFORCE{{\sc reinforce}\xspace}
\def\ARM{{\sc arm}\xspace}
\def\PSA{{\sc psa}\xspace}
\def\ST{{\sc st}\xspace}

\def\mytitle{Path Sample-Analytic Gradient Estimators \\ for Stochastic Binary Networks}
\title{\mytitle}

\author{%
  Alexander Shekhovtsov\\
	Czech Technical University in Prague\\
  \texttt{shekhovt@cmp.felk.cvut.cz} \\
  \And 
	Viktor Yanush\\
	Lomonosov Moscow State University\\
  \texttt{yanushviktor@gmail.com} \\
  \And
	Boris Flach\\
	Czech Technical University in Prague\\
  \texttt{flachbor@cmp.felk.cvut.cz} \\	
}

\begin{document}
\maketitle

\begin{abstract}
In neural networks with binary activations and or binary weights the training by gradient descent is complicated as the model has piecewise constant response.
We consider stochastic binary networks, obtained by adding noises in front of activations.
The expected model response becomes a smooth function of parameters, its gradient is well defined but it is challenging to estimate it accurately.
We propose a new method for this estimation problem combining sampling and analytic approximation steps. The method has a significantly reduced variance at the price of a small bias which gives a very practical tradeoff in comparison with existing unbiased and biased estimators.
We further show that one extra linearization step leads to a deep straight-through estimator previously known only as an ad-hoc heuristic. 
We experimentally show higher accuracy in gradient estimation and demonstrate a more stable and better performing training in deep convolutional models with both proposed methods.
\end{abstract}

\section{Introduction}
Neural Networks with binary weights and binary activations are very computationally efficient. \citet{RastegariORF16} report up to $58{\times}$ speed-up compared to floating point computations. There is a further increase of hardware support for binary operations: matrix multiplication instructions in recent NVIDIA cards, specialized projects on spike-like (neuromorphic) computation~\cite{Cassidy2016TrueNorthAH,Loihi}, \etc.

Binarized (or more generally quantized) networks have been shown to close up in performance to real-valued baselines~\cite{Bethge-18,peters2019probabilistic,Roth-19,Tang2017HowTT,Hubara-16,courbariaux2016binarized}. We believe that good training methods can improve their performance further. 
The main difficulty with binary networks is that unit outputs are computed using $\sign$ activations, which renders common gradient descent methods inapplicable. Nevertheless, experimentally oriented works ever so often define the lacking gradients in these models in a heuristic way. We consider the more sound approach of stochastic Binary Networks (SBNs)~\cite{Neal:1992,Raiko-14}. This approach introduces injected noises in front of all $\sign$ activations. The network output becomes smooth in the expectation and its derivative is well-defined. Furthermore, injecting noises in all layers makes the network a deep latent variable model with a very flexible predictive distribution. 

Estimating gradients in SBNs is the main problem that we address. We focus on handling {\em deep dependencies through binary activations}, which we believe to be the crux of the problem. That is why we consider all weights to be real-valued in the present work. An extension to binary weights would be relatively simple, \eg by adding an extra stochastic binarization layer for them. %, since the case of binary weights and differentiable activations has been already addressed satisfactory in the literature.

\paragraph{SBN Model}

Let $x^0$ denote the input to the network (\eg. an image to recognize). 
We define a {\em stochastic binary network} (SBN) with $L$ layers with neuron outputs $X^{1 \dots L}$ and injected noises $Z^{1 \dots L}$ as follows (\cref{fig:SBN} left):
%\begin{subequations}
\begin{align}\label{SBN}
\textstyle X^0 = x^0; \ \ \ \ \ \textstyle X^{k} = \sgn(a^k(X^{k-1}; \theta^k) - Z^k);  \ \ \ \ \ \ F = f(X^L; \theta^{L+1}). 
\end{align}
%\end{subequations}
%\begin{subequations}\label{SBN}
%\begin{align}
%&\textstyle X^0 = x^0; && \textstyle  \text{(input layer)}\\
%&\textstyle X^{k} = \sgn(a^k(X^{k-1}; \theta^k) - Z^k);  && \textstyle  \text{(hidden layers)}\\
%&\textstyle  F = f(X^L; \theta^{L+1}). && \textstyle \text{(network head)}
%\end{align}
%\end{subequations}
The output $X^{k}$ of layer $k$ is a vector in $\Bool^n$, where we denote binary states $\Bool = \{-1,1\}$. 
The network input $x^0$ is assumed real-valued. The noise vectors $Z^k$ consist of $n$ independent variables with a known distribution (\eg, logistic). The network {\em pre-activation} functions $a^k(X^{k-1}; \theta^k)$ are assumed differentiable in parameters $\theta^k$ and will be typically modeled as affine maps (\eg, fully connected, convolution, concatenation, averaging, \etc). Partitioning the parameters $\theta$ by layers as above, incurs no loss of generality since they can in turn be defined as any differentiable mapping $\theta = \theta(\eta)$ and handled by standard backpropagation.

The network {\em  head function} $f(x^L; \theta^{L+1})$ denotes the remainder of the model not containing further binary dependencies. For classification problems we consider the softmax predictive probability model $p(y|x^L; \theta^{L+1}) = \softmax(a^{L+1}(x^L;\theta^{L+1}))$,
where the affine transform $a^{L+1}$ computes class scores from the last binary layer. The function $f(X^L; \theta^{L+1})$ is defined as the cross-entropy of the predictive distribution $p(y|x^L; \theta^{L+1})$ relative to the training label distribution $p^*(y|x^0)$.

\begin{figure}
\centering
\includegraphics[width=0.48\linewidth]{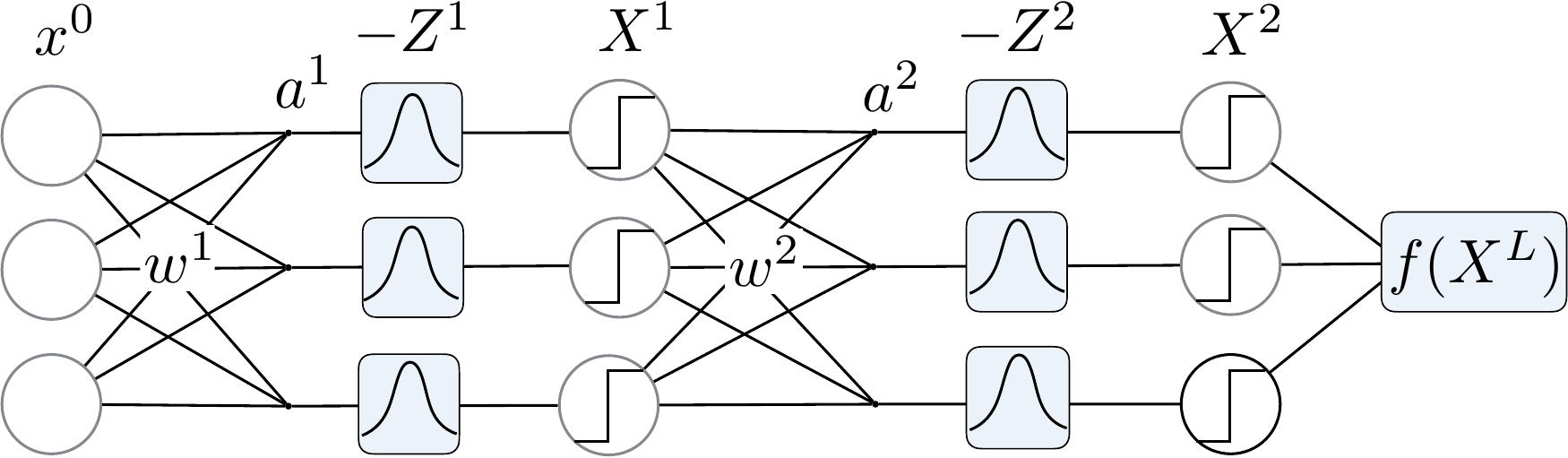}\ \ \ 
\includegraphics[width=0.50\linewidth]{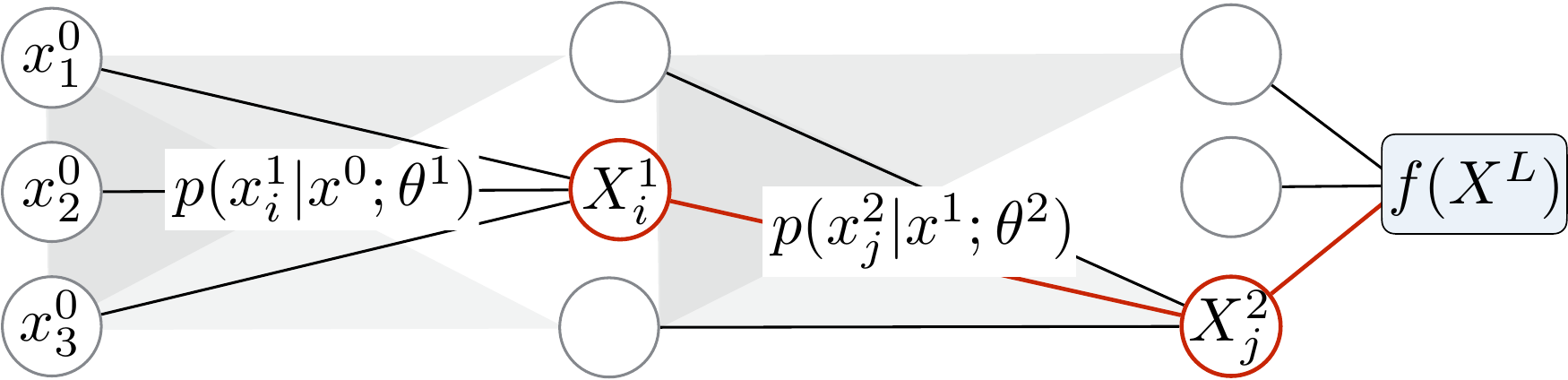}
\caption{\label{fig:SBN}%
Stochastic binary network with 2 hidden layers.
{\em Left}: latent variable model view (injected noises). {\em Right}: directed graphical model view (Bayesian network). The \PSA method performs explicit summation along paths (highlighted).
}
\end{figure}

Due to the injected noises, all states $X$ become random variables and their joint distribution given the input $x^0$ 
takes the form of a Bayesian network with the following structure (\cref{fig:SBN} right):
\begin{subequations}\label{prod-model}
\begin{align}
\textstyle p(x^{1\dots L} \mid x^0 ;\theta) = \prod_{k=1}^{L}p(x^k|x^{k-1};\theta^k),\ \ \ \ \ \ 
\textstyle p(x^k|x^{k-1};\theta^k) = \prod_{i=1}^n p(x^k_i|x^{k-1}; \theta^k). \label{prod-model-factors}
\end{align}
\end{subequations}
The equivalence to the injected noise model is established with 
%\vskip-20pt
\begin{align}\label{cond-model}
\textstyle
p(x^k_j=1|x^{k-1}; \theta^k) = \Pr\big( a^k_j{-}Z^k_j > 0 \big) = F_Z(a^k_j),
\end{align}
%\begin{align}\label{cond-model}
%\textstyle
%p(x^k_j|x^{k-1}; \theta^k) = 
%\Pr\big( \sgn(a^k_j{-}Z^k_j)\,{=}\,x^k_j \big) = 
%\begin{cases}
%\textstyle  F_Z(a^k_j), &\text{if}\ x^k_j{=}1\\
%\textstyle  1 - F_Z(a^k_j), &\text{if}\  x^k_j{=}{-}1,
%\end{cases}
%\end{align}
where $F_Z$ is the noise cdf. If we consider noises with logistic distribution, $F_Z$ becomes the common sigmoid logistic function and the network with linear pre-activations $a^k(x^{k-1})_j = \sum_{j}w^k_{ij} x^{k-1}_i$ becomes the well-known {\em sigmoid belief network}~\cite{Neal:1992}. 
\paragraph{Problem}
The central problem for this work is to estimate the gradient of the expected loss:
\begin{align}\label{problem}
\textstyle \frac{\partial }{\partial \theta} \E_Z [F(\theta)] = \frac{\partial }{\partial \theta} \sum_{x^{1\dots L}}p(x^{1\dots L}| x^0;\theta) f(x^L; \theta).
\end{align}
%where we used the Bayesian network representation on the right. % and denote $x = x^{1\dots L}$.
Observe that when the noise cdf $F_Z$ is smooth, the expected network output $\E_Z [F(\theta)]$ is differentiable in parameters despite having binary activations and a head function possibly non-differentiable in $x^L$. This can be easily seen from the Bayesian network form on the right of~\eqref{problem}, where all functions are differentiable in $\theta$. %The main difficulty to address is the composite dependence of $p$ on parameters.
The gradient estimation problem of this kind arises in several learning formulations, please see~\cref{sec:learning} for discussion.

\paragraph{Bias-Variance Tradeoff}
A number of estimators for the gradient~\eqref{problem} exist, both biased and unbiased. Estimators using certain approximations and heuristics have been applied to deep SBNs with remarkable success. The use of approximations however introduces a {\em systematic error}, \ie these estimators are {\em biased}. Many theoretical works therefore have focused on development of lower-variance unbiased stochastic estimators, but encounter serious limitations when applied to deep models. At the same time, allowing a small bias may lead to a considerable reduction in variance, and more reliable estimates overall. We advocate this approach and compare the methods using metrics that take into account both the bias and the variance, in particular the mean squared error of the estimator. When the learning has converged to 100\% training accuracy (as we will test experimentally for the proposed method) the fact that we used a biased gradient estimator no longer matters.

%We also show experimentally that despite a small bias the reduced variance allows the learning to converge in a reasonable time to 100\% training accuracy, thus .

\paragraph{Contribution}
The proposed {\em Path Sample-Analytic} (\PSA) method is a biased stochastic estimator. It takes one sample from the model and then applies a series of derandomization and approximation steps. It efficiently approximates the expectation of the stochastic gradient by explicitly computing summations along multiple paths in the network \cref{fig:SBN}\,(right). Such explicit summation over many configurations gives a huge variance reduction, in particular for deep dependencies. The approximation steps needed for keeping the computation simple and tractable, are clearly understood linearizations. They are designed with the goal to obtain a method with the same complexity and structure as the usual backpropagation, including convolutional architectures. This allows to apply the method to deep models and to compute the gradients in parameters of all layers in a single backwards pass.
%
%To study the trade-off between bias and variance, we study small models with computable exact gradients. We measure the mean squared error of estimators and the cosine similarity, which is more indicative for correct descent directions needed in optimization. Such comparison has not been conducted before. We show experimentally that the proposed method with a single base sample has RMSE superior to {\sc reinforce} with $10^4$ samples, indicating that both the bias and variance are small.

A second simplification of the method is obtained by further linearizations in \PSA and leads to the {\em Straight-Through} (ST) method.
We thus provide the first theoretical justification of straight-through methods for deep models as derived in the SBN framework using a clearly understood linearization and partial summation along paths. This allows to eliminates guesswork and obscurity in the practical application of such estimators as well as opens possibilities for improving them. Both methods perform similar in learning of deep convolutional models, delivering a significantly more stable and better controlled training than preceding techniques. % by better optimizing the chosen training loss. % and not its (possibly unknown) approximation.

%With this understanding, and a few minor but important differences, it eliminates guesswork and obscurity in the practical application of straight-through estimators. 
%Our form allows to handle more general cases as well. 
%

%, better classification accuracy
%. They outperform the SOTA methods significantly in terms of the objective, \ie, the training loss and offer a significantly more stable and controlled training than what can be achieved with preceding techniques. 
%
\subsection{Related Work}
%When only the weights are binary stochastic, there are no deep binary dependencies and a variety of methods apply, in particular the local reparametrization trick~\cite{Kingma-15-dropout} proved to be efficient~\cite{shayer2018learning}.
%Furthermore, if we are able to deal with deep binary dependencies, we can introduce binary stochastic weights as just one extra binary layer.

\paragraph{Unbiased estimators} 
A large class of unbiased estimators is based on the {\sc reinforce}~\cite{Williams1992}. Methods developed to reduce its variance include learnable input-dependent~\cite{mnih14} and linearization-based~\cite{MuProp} control variates. Advanced variance reduction techniques have been proposed: {\sc rebar}~\citep{Tucker-17-REBAR}, {\sc relax}~\cite{grathwohl18-relax}, {\sc arm}~\cite{yin18-arm}.
However the latter methods face difficulties when applied to deep belief models and indeed have never been applied to SBNs with more than two layers. One key difficulty is that they require $L$ passes through the network\footnote{See {\sc rebar} section 3.3, {\sc relax} section 3.2, {\sc arm} Alg.~2. Different propositions to overcome the complexity limitation appear in appendices of several works, including also~\cite{cong2018go}, which however have not been tested in deep models.}, leading to a quadratic complexity in the number of layers. We compare to {\sc arm}, which makes a strong baseline according to comparisons in~\cite{yin18-arm,grathwohl18-relax,Tucker-17-REBAR}, in a small-problem setting. Previous {\em direct} comparison of estimator accuracy at the same point was limited to a single neuron setting~\cite{yin18-arm}, where \PSA would be simply exact. We also compare to {\sc MuProp}~\cite{MuProp} and variance-reduced {\sc reinforce}, which run in linear time, in the deep setting in~\cref{sec:muprop}. We observe that the variance of these estimators stays prohibitive for a practical application.

In the case of one hidden layer, our method coincides with several existing unbiased techniques~\cite{cong2018go,Titsias-15,Tokui-17}. The {\sc ram} estimator~\cite{Tokui-17} can be applied to deep models and stays unbiased but scales quadratically in the number of variables. Our estimator becomes biased for deep models but scales linearly. %, which is possible thanks to linearizations and derandomizations that we apply.
\paragraph{Biased Estimators} 
Several works demonstrate successful training of deep networks using biased estimators.
One such estimator, based on smoothing the $\sign$ function in the SBN model is the {\em concrete relaxation}~\cite{maddison2016concrete,jang2016categorical}. It has been successfully applied for training large-scale SBN in~\cite{peters2019probabilistic}.
Methods that propagate moments analytically, known as {\em assumed density filtering} (ADF), \eg,~\cite{shekhovtsov18-cat,Gast18} perform the full analytic approximation of the expectation. ADF has been successfully used in~\cite{Roth-19}. 

Many experimentally oriented works successfully apply {\em straight-through} estimators ({\sc ste}).
Originally considered by~\citet{HintonST} for deep auto-encoders with Bernoulli latent variables and by~\citet{bengio2013estimating} for conditional computation, these simple, but not theoretically justified methods were later adopted for training deterministic networks with binary weights and activations~\cite{courbariaux2016binarized, zhou2016dorefa,Hubara-16}. 
The method simply pretends that the $\sign$ function has the derivative of the identity, or of some other function. There have been recent attempts to justify why this works. \citet{yin2019understanding} considers the expected loss over the training data distribution, which is assumed to be Gaussian, and show that in networks with 1 hidden layer the true expected gradient positively correlates with the deterministic {\sc ste} gradient. \citet{cheng2019straight} show for networks with 1 hidden layer that {\sc ste} is approximately related to the projected Wasserstein gradient flow method proposed there. In the case of one hidden layer~\citet[sec. 6.4]{Tokui-17} derived {\sc ste} as a linearization of their {\sc ram} estimator for SBN. We derive deep {\sc ste} in the SBN model by making extra linearizations in our {\sc psa} estimator.  %\citet{bengio2013estimating} remarks that the {\sc ste} with the gradient of $\tanh$ does not perform well. Our derivation and experiments suggest the opposite.

\section{Method}
To get a proper view of the problem, we first explain the exact chain rule for the gradient, relating it to the {\sc reinforce} and the proposed method. Throughout this section we will consider the input $x^0$ fixed and omit it from the conditioning such as in $p(x | x^0)$.

\subsection{Exact Chain Rule}
The expected loss in the Bayesian network representation~\eqref{prod-model} can be written as
\begin{align}\label{mc-model}
%\notag 
\textstyle \E_Z[F] = \textstyle \sum_{x^1 \in \Bool^n} p(x^1;\theta^1)\sum_{x^2 \in \Bool^n} p(x^2|x^1;\theta^2)
 \dots \sum_{x^L \in \Bool^n} p(x^L|x^{L-1};\theta^L) f(x^L;\theta^{L+1}) 
\end{align}
and can be related to the forward-backward marginalization algorithm for Markov chains. Indeed, let $P^k = p(x^k|x^{k-1};\theta^k)$ be transition probability matrices of size $2^n{\times}2^n$ for $k>1$ and $P^1$ be a row vector of size $2^n$. 
Then~\eqref{mc-model} is simply the product of these matrices and the vector $f$ of size $2^n$. 
The computation of the gradient is as ``easy'', \eg for the gradient in $\theta^1$ we have: 
\begin{align}\label{mc-grad}
g^1 := \textstyle \frac{\partial}{\partial \theta^1}\E_Z[F] = & D^1 P^2 P^3 \dots P^L f,
\end{align}
where $D^1$ is the size $\dim(\theta^1){\times}2^n$ transposed Jacobian $\frac{\partial p(x^1;\theta^1)}{\partial \theta^1}$. Expression~\eqref{mc-grad} is a product of transposed Jacobians of a deep model. Multiplying them in the right-to-left order requires only matrix-vector products and it is the exact back-propagation algorithm. As impractical as it may be, it is still much more efficient than the brute force enumeration of all $2^{nL}$ joint configurations.

The well-known {\sc reinforce}~\cite{Williams1992} method replaces the intractable summation over $x$ with sampling $x$ from $p(x^{1\dots L};\theta)$ and uses the stochastic estimate
\begin{align}
\textstyle \frac{\partial}{\partial \theta^1}\E_Z[F] \approx \frac{\partial p(x^1;\theta^1)}{\partial \theta^1}\frac{f(x^L)}{p^1(x^1;\theta^1)}.
\end{align}
While this is cheap to compute (and unbiased), it utilizes neither the shape of the function $f$ beyond its value at the sample nor the dependence structure of the model.
\subsection{PSA Algorithm}
We present our method for a deep model. The single layer case, is a special case which is well covered in the literature~\cite{Tokui-17,Titsias-15,cong2018go} and is discussed in~\cref{L:singlelayer}.
Let us consider the gradient in parameters $\theta^l$ of layer $l$. Starting from the RHS of~\eqref{problem}, we can move derivative under  the sum and directly differentiate the product distribution~\eqref{prod-model}: 
\begin{align}\label{direct-part-0}
g^l := \textstyle \frac{\partial}{\partial \theta^l} \sum_{x} p(x; \theta) f(x^L; \theta^{L+1}) = \textstyle\sum_{x} f(x^L) \frac{\partial }{\partial \theta^l}p(x; \theta) = \sum_{x} \frac{p(x) f(x^L)}{p(x^l | x^{l-1})} \frac{\partial}{\partial \theta^l} p(x^l| x^{l-1}; \theta^l),
\end{align}
%Directly differentiating the product model~\eqref{prod-model} we obtain:
%\begin{align}\label{direct-part-0}
%g^l = \textstyle \sum_{x} \frac{p(x)}{p(x^l | x^{l-1})} \Big(\frac{\partial}{\partial \theta^l} p(x^l| x^{l-1}; \theta^l)\Big) f(x^L),
%\end{align}
where the dependence on $\theta$ in $p$ and $f$ is omitted once it is outside of the derivative. The fraction $\frac{p(x)}{p(x^l | x^{l-1})}$ is a convenient way to write the product $\prod_{k=1| k \neq l}^L p(k^{k}|x^{k-1})$, \ie with factor $l$ excluded.

At this point expression~\eqref{direct-part-0} is the exact chain rule completely analogous to~\eqref{mc-grad}. %We now make the summation step.
A tractable back propagation approximation is obtained as follows.
Since $p(x^l| x^{l-1}; \theta^l) = \prod_{i=1}^n p(x^l_i|x^{l-1}; \theta^l)$, its derivative in~\eqref{direct-part-0} results in a sum over units $i$:
\begin{align}\label{d^l}
g^l = \textstyle \sum_{x} \sum_{i} \frac{p(x)}{p(x^l_i | x^{l-1})} D^l_i(x) f(x^L),\ \ \ \ \ \ \
\text{where}\ \textstyle D^l_i(x) = \frac{\partial}{\partial \theta^l} p(x^l_i| x^{l-1}; \theta^l).
\end{align}
We will apply a technique called {\em derandomization} or Rao-Blackwellization~\citep[ch. 8.7]{Owen-13} to $x^l_i$ in each summand $i$. Put simply, we sum over the two states of $x^l_i$ explicitly. The summation computes a portion of the total expectation in a closed form and as such, naturally and guaranteed, reduces the variance. The estimator with this derandomization step is an instance of the general {\em local expectation gradient}~\cite{Titsias-15}.
The derandomization results in the occurrence of the differences of products:
\begin{align}\label{prod-diff}
\textstyle \prod_{j} p(x^{l+1}_j | x^l_{\downarrow i}) -\textstyle \prod_{j} p(x^{l+1}_j | x^l),
\end{align}
where $x^l_{\downarrow i}$ denotes the state vector in layer $l$ with the sign of unit $i$ flipped. 
We approximate this difference of products by {\em linearizing} it (making a 1st order Taylor approximation) in the differences of its factor probabilities, \ie, replacing the difference~\eqref{prod-diff} with 
\begin{align}\label{approx-prod}
\textstyle \sum_j \prod_{j' \neq j} p(x^{k+1}_{j'} | x^k) \Delta^{k+1}_{i,j}(x), \ \  \ \ \text{where } \Delta^{k+1}_{i,j}(x) = p(x^{k+1}_{j}{|}x^{k}) - p(x^{k+1}_{j}{|}x^k_{\downarrow i}).
\end{align}
The approximation is sensible when $\Delta^{k+1}_{i, j}$ are small. This holds \eg, in the case when the model has many units in layer $k$ that all contribute to preactivation so that the effect of flipping a single $x^k_i$ is small.

Notice that the approximation results in a sum over units $j$ in the next layer $l{+}1$, which allows us to isolate a summand $j$ and derandomize $x^{l+1}_j$ in turn. Chaining these two steps, derandomization and linearization, from layer $l$ onward to the head function, we obtain summations over units in all layers $l\dots L$, equivalent to considering all paths in the network, and derandomize over all binary states along each such path (see~\cref{fig:SBN} right). The resulting approximation of the gradient $g^l$ gives 
\begin{align}\label{matrix-prod}
\textstyle \tilde g^l = \sum_x p(x) D^l(x) \Delta^{l+1}(x) \cdots \Delta^{L}(x) df(x)  =: \sum_x p(x) \hat g^l(x),
\end{align}
where $D^l$, defined in~\eqref{d^l}, is the Jacobian$\T$ of layer probabilities in parameters (a matrix of size $\dim(\theta^l){\times}n$); $\Delta^k_{i,j}$ defined in~\eqref{approx-prod} are $n{\times}n$ matrices, which we call {\em discrete Jacobians$\T$}; and $df$ is a column vector with coordinates $f(x^L) - f(x^L_{\downarrow i})$, \ie a {\em discrete gradient} of $f$. 
Thus a stochastic estimate $\hat g^l(x)$ 
%is obtained by drawing one sample $x\sim p(x)$ from the model and computing 
%\begin{align}\label{approx-prod-sample}
%\textstyle \hat g^l = D^l(x) \Delta^{l+1}(x) \cdots \Delta^{L}(x) df(x),
%\end{align}
%which 
is a product of Jacobians$\T$ in the ordinary activation space and can therefore be conveniently computed by back-propagation, \ie multiplying the factors in~\eqref{matrix-prod} from right to left.

This construction allows us to approximate a part of the chain rule for the Markov chain with $2^n$ states as a chain rule in a tractable space and to do the rest via sampling. 
This way we achieve a significant reduction in variance with a tractable computational effort. This computation effort is indeed optimal is the sense that it matches the complexity of the standard back-propagation, which matches the complexity of forward propagation alone, \ie, the cost of obtaining a sample $x\sim p(x)$.

\paragraph{Derivation}
The formal construction is inductive on the layers. Consider the general case $l{<}L$. We start from the expression~\eqref{d^l} and apply to it~\cref{prop-delta} below recurrently, starting with $J^l = D^l$. The matrices $J^k$ will have an interpretation of composite Jacobians$\T$ from layer $l$ to layer $k$.
\begin{restatable}{proposition}{Pdelta}\label{prop-delta}
Let $J^k_i(x)$ be functions that depend only on $x^{1\dots k}$ and are {\em odd} in $x^k_i$: $J^k_i(x^k) = -J^k_i(x^k_{\downarrow i})$ for all $i$.
Then
\begin{align}\label{direct-prop-eq1}
\textstyle \sum_{x} p(x) \sum_i \frac{J^k_i(x) f(x^L) }{p(x^k_i | x^{k-1})} 
\approx 
\textstyle \sum_{x} p(x) \sum_j \frac{J^{k+1}_j(x) f(x^L)}{p(x^{k+1}_j | x^{k})}, \ \ \ \ \text{where $J^{k+1}_{j} = \sum_i J^k_i(x) \Delta^{k+1}_{i, j}(x)$} 
\end{align}
and the approximation made is the linearization~\eqref{approx-prod}. Functions $J^{k+1}_{j}$ are odd in $x^{k+1}_j$ for all $j$. %and $\Delta^{k+1}_{i, j}(x)$ are defined as in~\eqref{approx-prod}.
\end{restatable}
The structure of~\eqref{direct-prop-eq1} shows that we will obtain an expression of the same form but with the dividing probability factor from the next layer, which allows to apply it inductively. To verify the assumptions at the induction basis, observe that according to~\eqref{cond-model}, 
%\begin{align}%\label{D-expandend}
$\textstyle D^l_i(x) = \frac{\partial}{\partial \theta^l} p(x^l_i| x^{l-1}; \theta^l) 
= p_Z(a^l_i) x^l_i \frac{\partial}{\partial \theta^l} a^l_i(x^{l-1}; \theta^l),
$
%\end{align}
hence it depends only on $x^{l-1}$, $x^l$ and is odd in $x^l_i$.

In the last layer, the difference of head functions occurs instead of the difference of products. Thus instead of~\eqref{direct-prop-eq1}, we obtain for $k=L$ the expression
\begin{align}\label{last-layer-case}
\textstyle \sum_{x} p(x) \sum_i \frac{J^L_i(x) f(x^L) }{p(x^L_i | x^{L-1})} 
= \sum_{x} \sum_i p(x) J^L_i(x) df_i(x).
\end{align}
Applying~\cref{prop-delta} inductively, we see that the initial $J^l$ is multiplied with a discrete Jacobian$\T$ $\Delta^{k+1}$ on each step $k$ and finally with $df$. Note that neither the matrices $\Delta^k$ nor $df$ depend on the layer we started from.
The final result of inductive application of~\cref{prop-delta} is exactly the expected matrix product~\eqref{matrix-prod}. The key for computing a one-sample estimate $\hat g(x)^l = D^l \Delta^{l+1} \cdots \Delta^{L} df$
is to perform the multiplication from right to left, which only requires matrix-vector products. Observe also that the part $\Delta^{k} \cdots \Delta^{L} df$ is common for our approximate derivatives in all layers $l<k$ and therefore needs to be evaluated only once.

\paragraph{Algorithm} Since backpropagation is now commonly automated, we opt to define forward propagation rules such that their automatic differentiation computes what we need. The algorithm in this form is presented in~\cref{alg:PSA}. First, we have substituted the noise model~\eqref{cond-model} to compute $\Delta^{k}_{i,j}$ as in~\cref{Delta in alg}. The {\tt detach} method available in PyTorch~\cite{pytorch} obtains the value of the tensor but excludes it from back-propagation. It is applied to $\Delta^{k}$ since it is already the Jacobian$\T$ and we do not want to differentiate it.
The recurrent computation of $q^k$ in \cref{def2:q^k} serves to generate the computation of $\hat g^l(x)$ on the backward pass. Indeed, variables $q^k$ always store just zero values as ensured by~\cref{q^k-detach} but have non-zero Jacobians as needed. A similar construct is applied for the final layer. 
It is easy to see that differentiating the output $E$ \wrt $\theta^l$ recovers exactly $\hat g(x)^l = D^l \Delta^{l+1} \cdots \Delta^{L} df$.
%$\hat g^l$ exactly as needed in~\eqref{approx-prod-sample}.

This form is simpler to present and can serve as a reference. The large-scale implementation detailed in~\cref{sec:impl} defines custom backward operations, avoiding the overhead of computing variables $q^k$ in the forward pass. The overhead however is a constant factor and the following claim applies.

\begin{figure}[t]
\centering
\begin{minipage}[t]{0.5\linewidth}
\begin{algorithm}[H]
\small
\KwIn{Network parameters $\theta$, input $x^0$}
\KwOut{The expression $E$ generating the derivative}
Initialize: $q^0 = 0$\;
\For{layer $k$ with Bernoulli output}{
		$a^k_{j} = a^k_{j}(x^{k-1}; \theta^k)$\;
		Sample layer output state $x^k_j\in\{-1,1\}$ with probability of $1$ given by $F_Z(a^k_{j})$\;
		Compute discrete Jacobians$\T$:\\
		$\Delta^{k}_{i,j}{=}\detach(x^k_j (F_Z(a^k_{j}){-}F_Z(a^k_{j}(x^{k-1}_{\downarrow i}; \theta^k))))$\label{Delta in alg};\hskip-5pt\\
		Generate chain dependence:\\
		\label{def2:q^k}
		{$q^k_j = x^k_j F_Z(a^k_{j})	 + \sum_{i} \Delta^k_{i,j} q^{k-1}_i$}\;
		\label{q^k-detach}
		{$q^k : = q^k - \detach(q^k)$}\;
}
Last layer:\\
 \label{last-layer}
	$E = f(x^L; \theta^{L+1}) + \sum_i \detach(f(x^L) - f(x^L_{\downarrow i})) q^L_i$\;
\Return{E}
\caption{Path Sample-Analytic (\bf PSA)\label{alg:PSA}}
\end{algorithm}
\end{minipage}
\ \ \ \ \ \begin{minipage}[t]{0.4\linewidth}
\begin{algorithm}[H]
\small
\KwIn{Network parameters $\theta$, input $x^0$}
\KwOut{The expression $E$ generating the derivative}
\For{layer $k$ with Bernoulli output}{
		$a^k_{j} = a^k_{j}(x^{k-1}; \theta^k)$\;
		Sample layer output states $x^k_j\in\{-1,1\}$ with probability of $1$ given by $F_Z(a^k_{j})$\;
		Compute $\tilde x^k_j = 2 F_Z(a^k_j)$\;
		Binary state with a derivative generator:
		$x^k := x^k + \tilde x^k -\detach(\tilde x^k)$\;
	%	}
}
Last layer's output is the derivative generator:\\
\Return{$E = f(x^L)$}
\caption{Straight-Through ({\bf ST})\label{alg:ST}}
\end{algorithm}
\end{minipage}
\end{figure}

\begin{restatable}{proposition}{Pcomplexity}\label{prop:complexity}
The computation complexity of \PSA gradient in all parameters of a network with convolutional and fully connected layers is the same as that of the standard back-propagation. % in the network with the same pre-activations and differentiable non-linearities.
\end{restatable}

The proof is constructive in that we specify algorithms how the required computation for all flips as presented in \PSA can be implemented with the same complexity as standard forward propagation. Moreover, in~\cref{sec:impl} we show how to implement transposed convolutions for the case of logistic noise to achieve the FLOPs complexity as low as 2x standard backprop.

We can also verify that in the cases when the exact summation is tractable, such as when there is only one hidden unit in each layer, the algorithm is exploiting it properly:
\begin{restatable}{proposition}{Punbiased}\label{prop:unbiased}
\PSA is an unbiased gradient estimator for networks with only one unit in layers $1\dots L-1$ and any number of units in layer $L$.
\end{restatable}
%The latter result has no practical relevance, but is satisfactory in the sense that for 
%
\subsection{The Straight-Through Estimator}
\begin{restatable}{proposition}{PST}\label{prop:ST}
Assume that pre-activations $a^k(x^{k-1};\theta)$ are multilinear functions in the binary inputs $x^{k-1}$, and the objective $f$ is differentiable. Then, approximating $F_Z$ and $f$ linearly around their arguments for the sampled base state $x$ in \cref{alg:PSA}, we obtain the straight-through estimator in~\cref{alg:ST}.
\end{restatable}

By applying the stated linear approximations, the proof of this proposition shows that the necessary derivatives and Jacobians in~\cref{alg:PSA} can be formally obtained as derivatives of the noise cdf $F_Z$ \wrt parameters $\theta$ and the binary states $x^{k-1}$. Despite not improving the theoretical computational complexity compared to \PSA, the implementation is much more straightforward.
We indeed see that~\cref{alg:ST} belongs to the family of straight-through estimators, using hard threshold activation in the forward pass and a smooth substitute for the derivative in the backward pass. The key difference to many empirical variants being that it is clearly paired with the SBN model and the choice of the smooth function for the derivative matches the model and the sampling scheme. As we have discovered later on, it matches exactly to the original description of such straight-through by~\citet{HintonST} (with the difference that we use $\pm1$ encoding and thus scaling factors $2$ occur).

In the case of logistic noise, we get $2 F_Z(a) = 2\sigma(a) = 1 + \tanh(a/2)$, where the added $1$ does not affect derivatives. This recovers the popular use of $\tanh$ as a smooth replacement, however, the slope of the function in our case must match the sampling distribution. 
In the experiments we compare it to a popular straight-through variant (\eg,~\cite{Hubara-16}) where the gradient of {\em hard tanh} function, $\min(\max(x,-1),1)$ is taken. The mismatch between the SBN model and the gradient function in this case leads to a severe loss of gradient estimation accuracy.
The proposed derivation thus eliminates lots of guesswork related to the use of such estimators in practice, allows to assess the approximation quality and understand its limitations. %Finally, it is rather general. In particular, the multilinear assumption is not a limitation since any function of binary variables can be represented in such form~\cite{BorosHammer02}.

\section{Experiments}\label{sec:experiments}

We evaluate the quality of gradient estimation on a small-scale dataset and the performance in learning on a real dataset. In both cases we use SBN models with logistic noises, which is a popular choice. Despite the algorithm and the theory is applicable for any continuous noise distribution, we have optimized the implementation of convolutions specifically for logistic noises.
\begin{figure*}[t]
\centering
\small
\begin{tabular}{ccc}
\small Layer 1 &  \small Layer 2 & \small Layer 3 \\  
\includegraphics[width=0.33\linewidth]{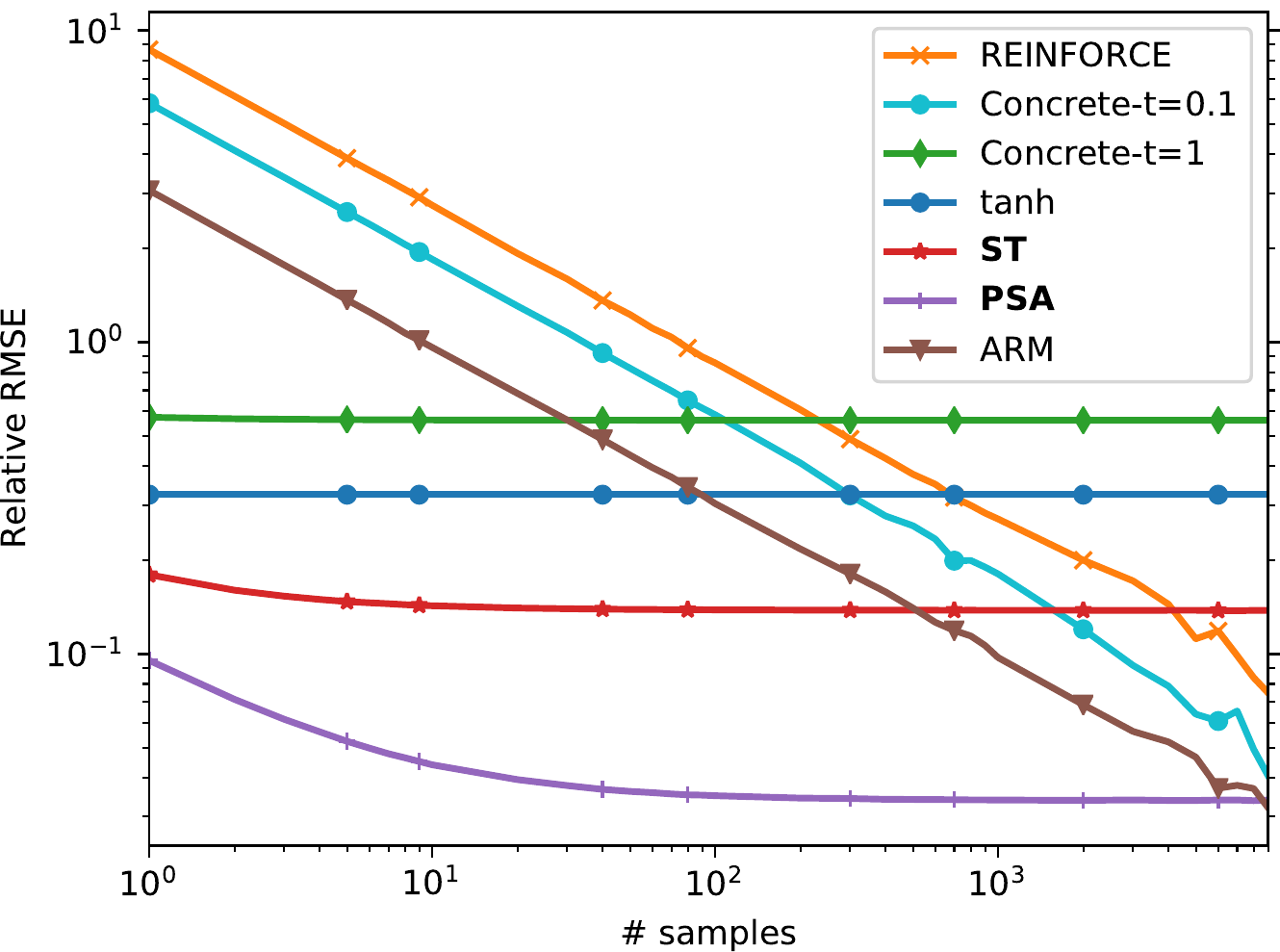}&
\includegraphics[width=0.33\linewidth]{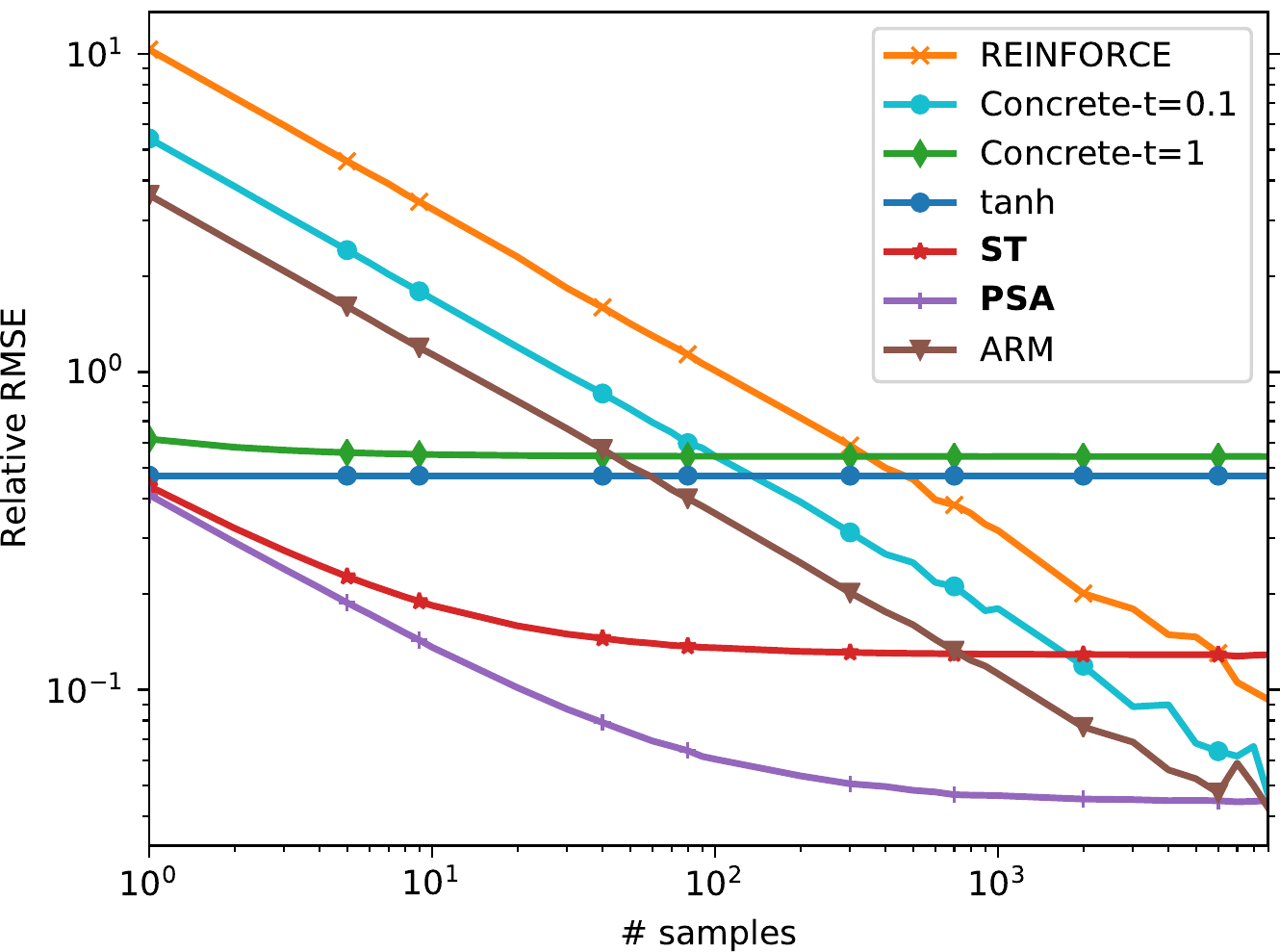}&
\includegraphics[width=0.33\linewidth]{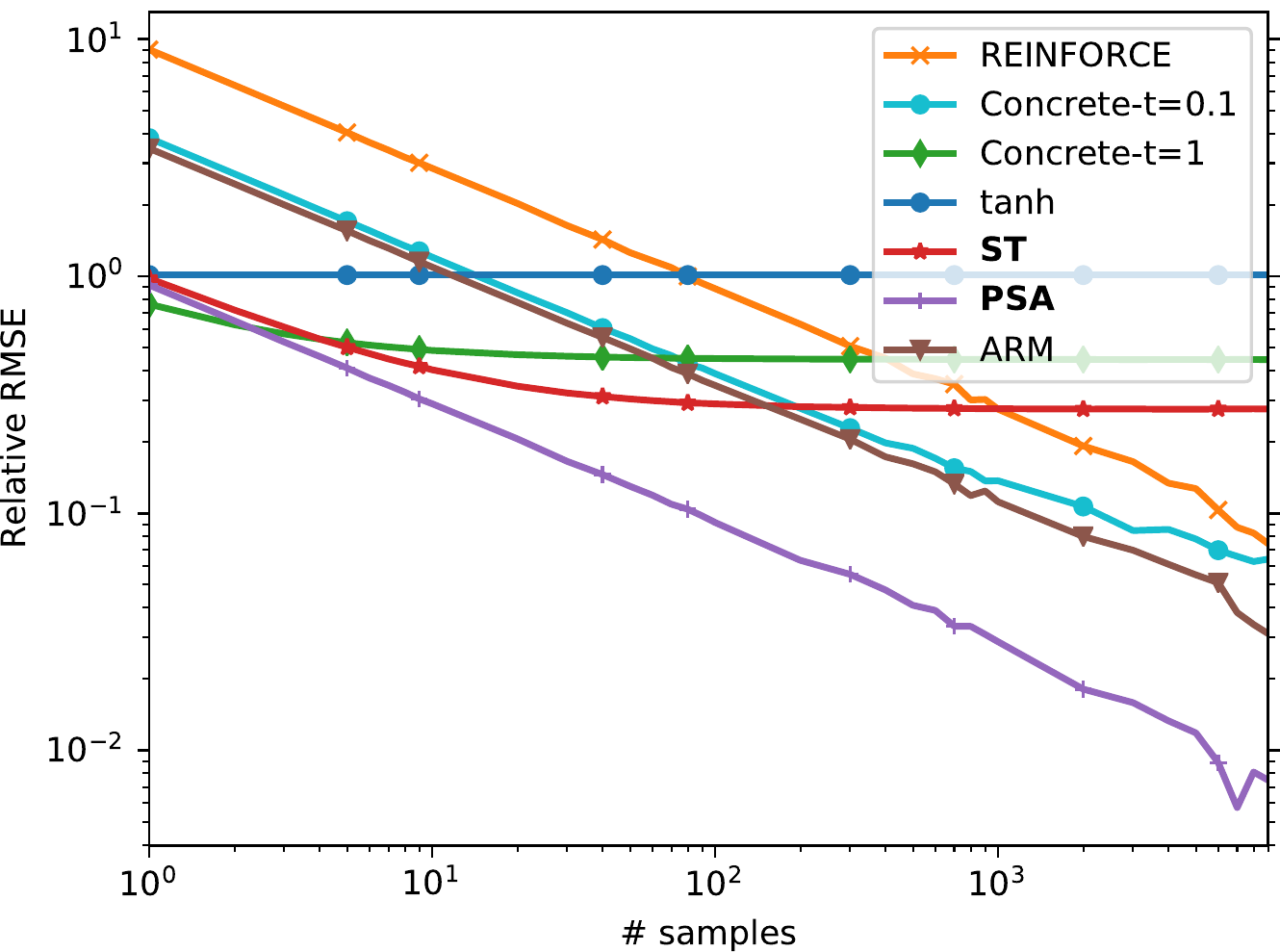}
\end{tabular}
\vskip-0.5\baselineskip
\caption{\label{fig:rmse}
Root mean squared error of the gradient in layers 1 to 3 relative to the true gradient length after epoch 1 of training with \REINFORCE.
Layer 1 parameters correspond to $\theta^1$ in~\eqref{SBN} -- weights and biases defining preactivations of layer $1$ Bernoulli states.
Unbiased estimators always improve with more samples. Biased estimators only improve up to a point.
However, biased methods may be more accurate when using fewer samples and the discrepancy significantly increases with layer depth. %, \ie, in layers further from the head it is more difficult to estimate gradients.
}
\end{figure*}
\begin{figure*}[t]
\setlength{\figwidth}{0.33\textwidth}
\centering
\begin{tabular}{ccc}
\small After 1 epoch &  \small After 100 epochs & \small After 2000 epochs \\  
\includegraphics[width=0.33\linewidth]{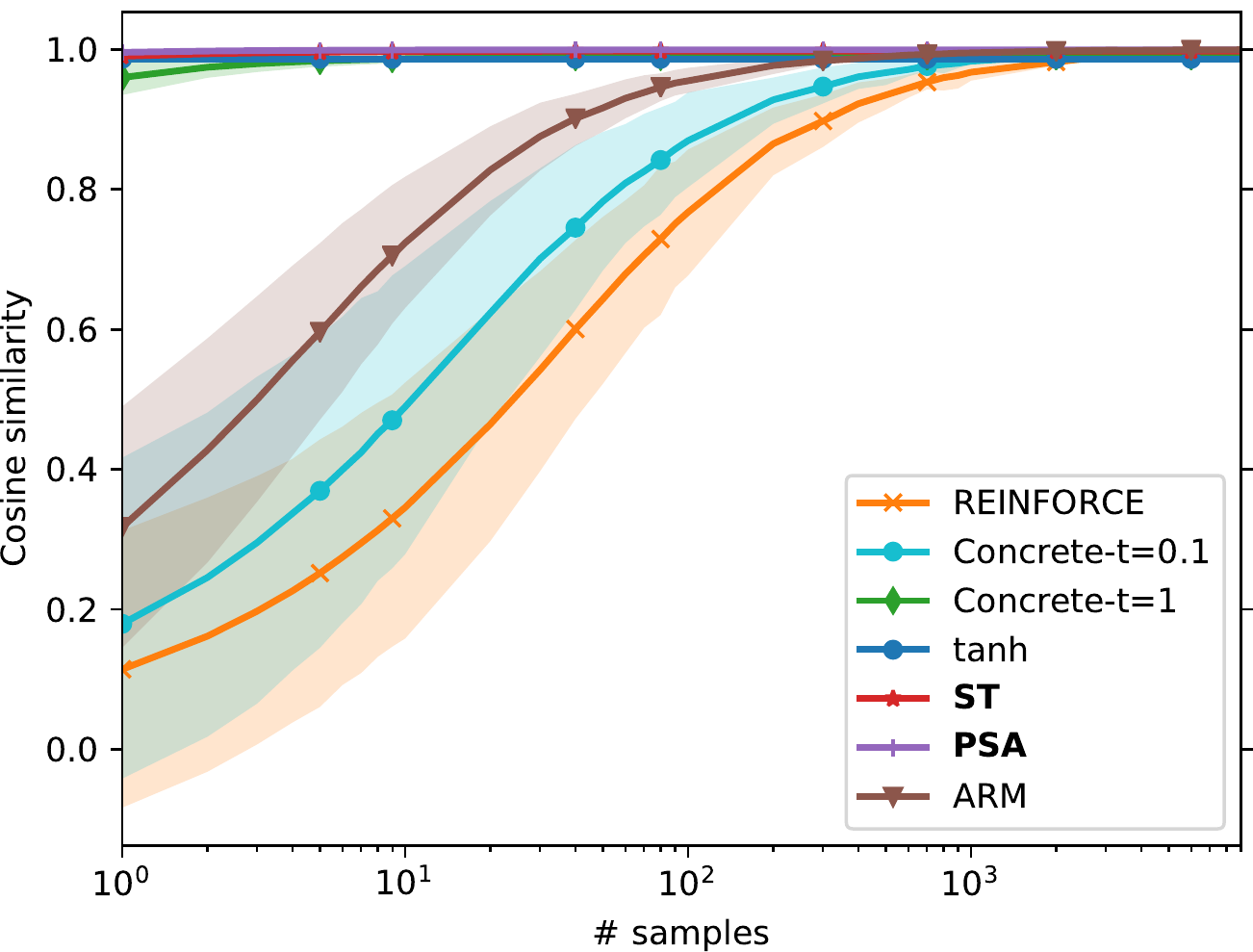}&
\includegraphics[width=0.33\linewidth]{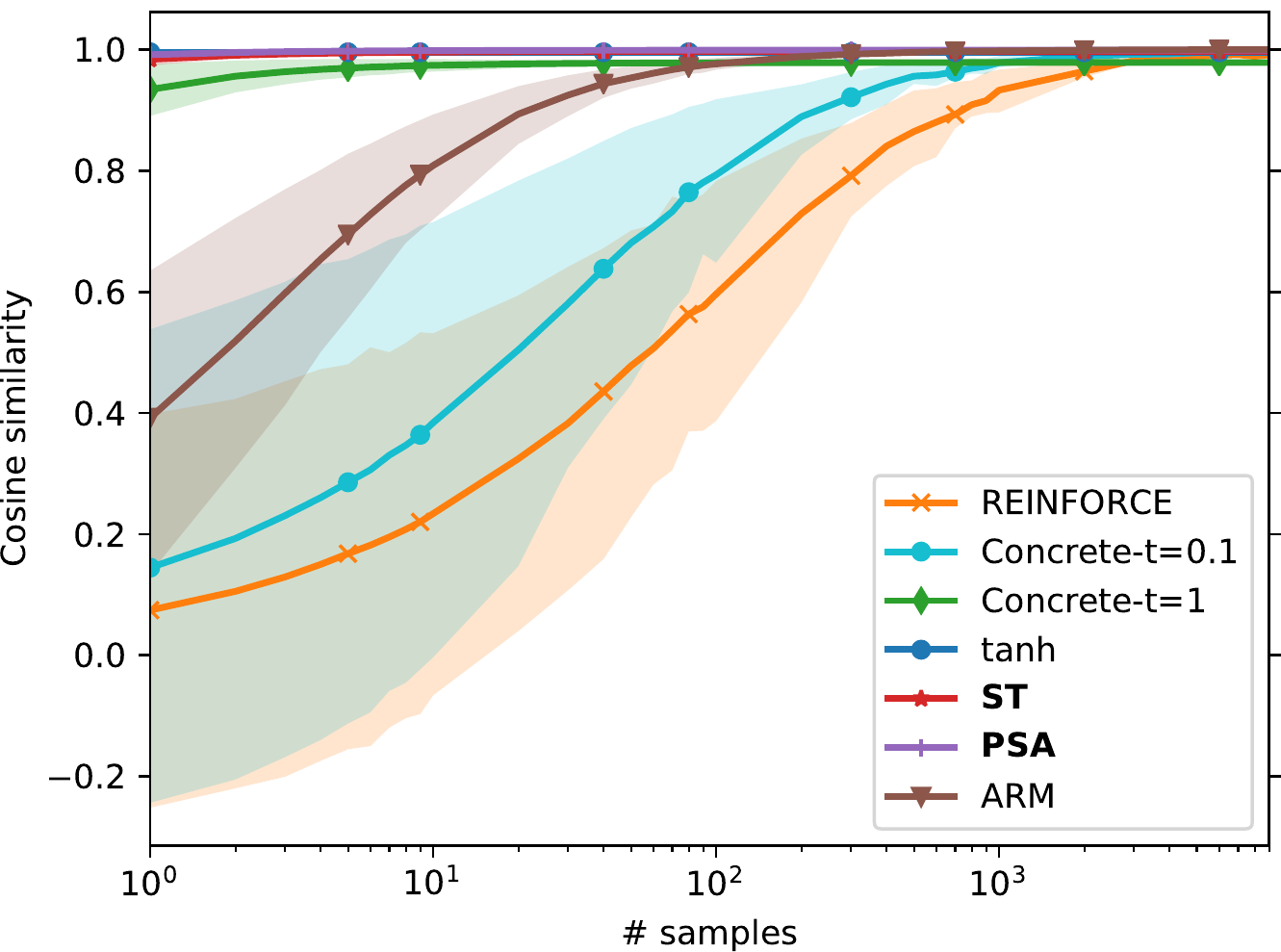}&
\includegraphics[width=0.33\linewidth]{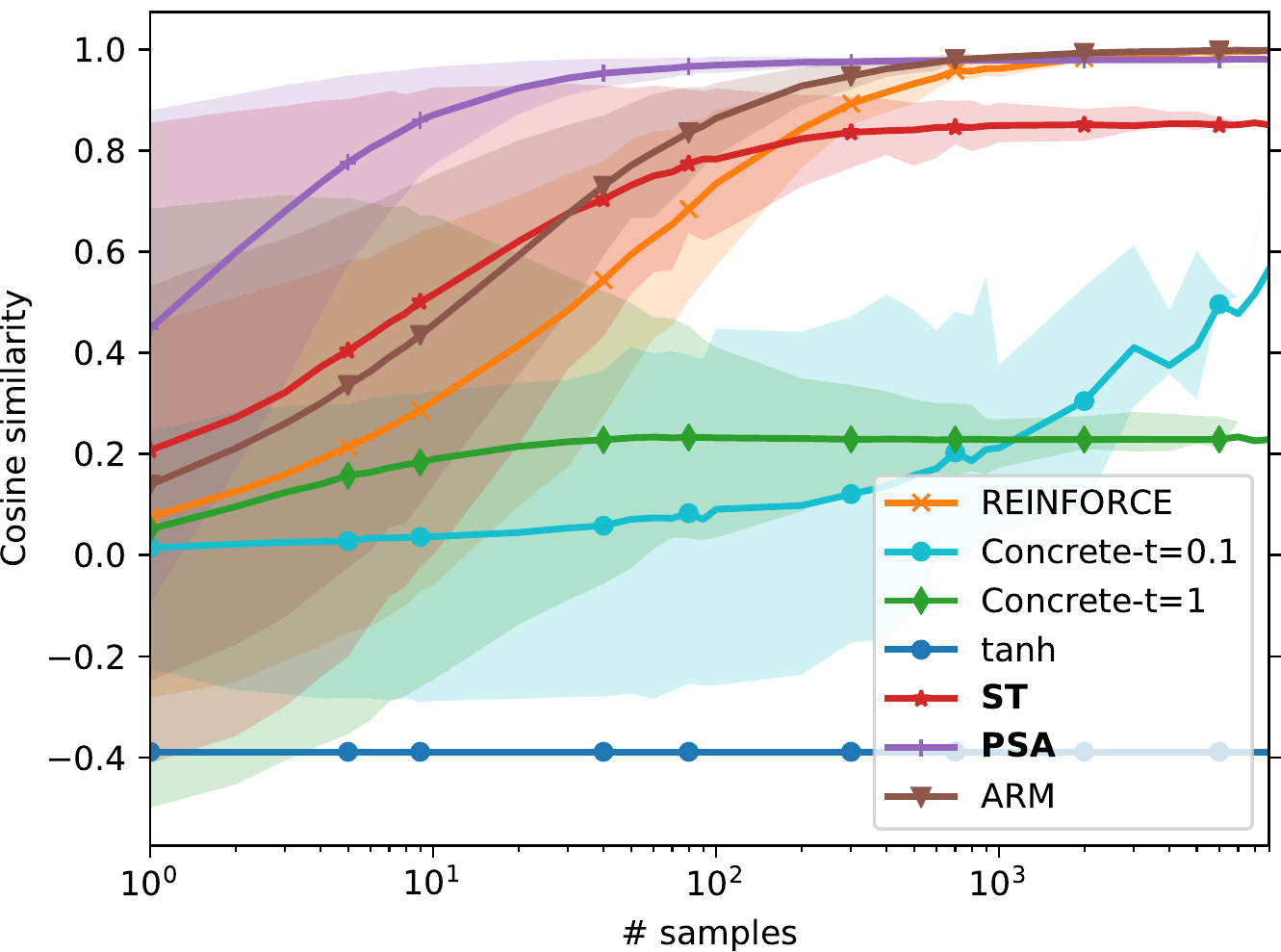}\\
\end{tabular}
\vskip-0.5\baselineskip
\caption{\label{fig:accuracy}
Cosine similarity of the estimated gradient to the true gradient in layer 1 at different points during training. 
The lines show the mean of the cosine similarity of the $N$-sample estimator. The shaded areas shows the interval containing 70\% of the trials, illustrating the scatter of values that can be obtained in a random trial. It is seen that for some estimators there are good chances of failing to produce a positive cosine, \ie a valid descent direction. 
}
\end{figure*}
\begin{figure*}[t]
\small
\setlength{\figwidth}{0.33\textwidth}
\centering
\begin{tabular}{cccc}
&\small No Augmentation & \small Affine Augmentation & \small Flip \& Crop Augmentation \\
\parbox{3mm}{\rotatebox[origin=c]{90}{\small Training Loss}}&
\begin{tabular}{c}
\includegraphics[width=\figwidth]{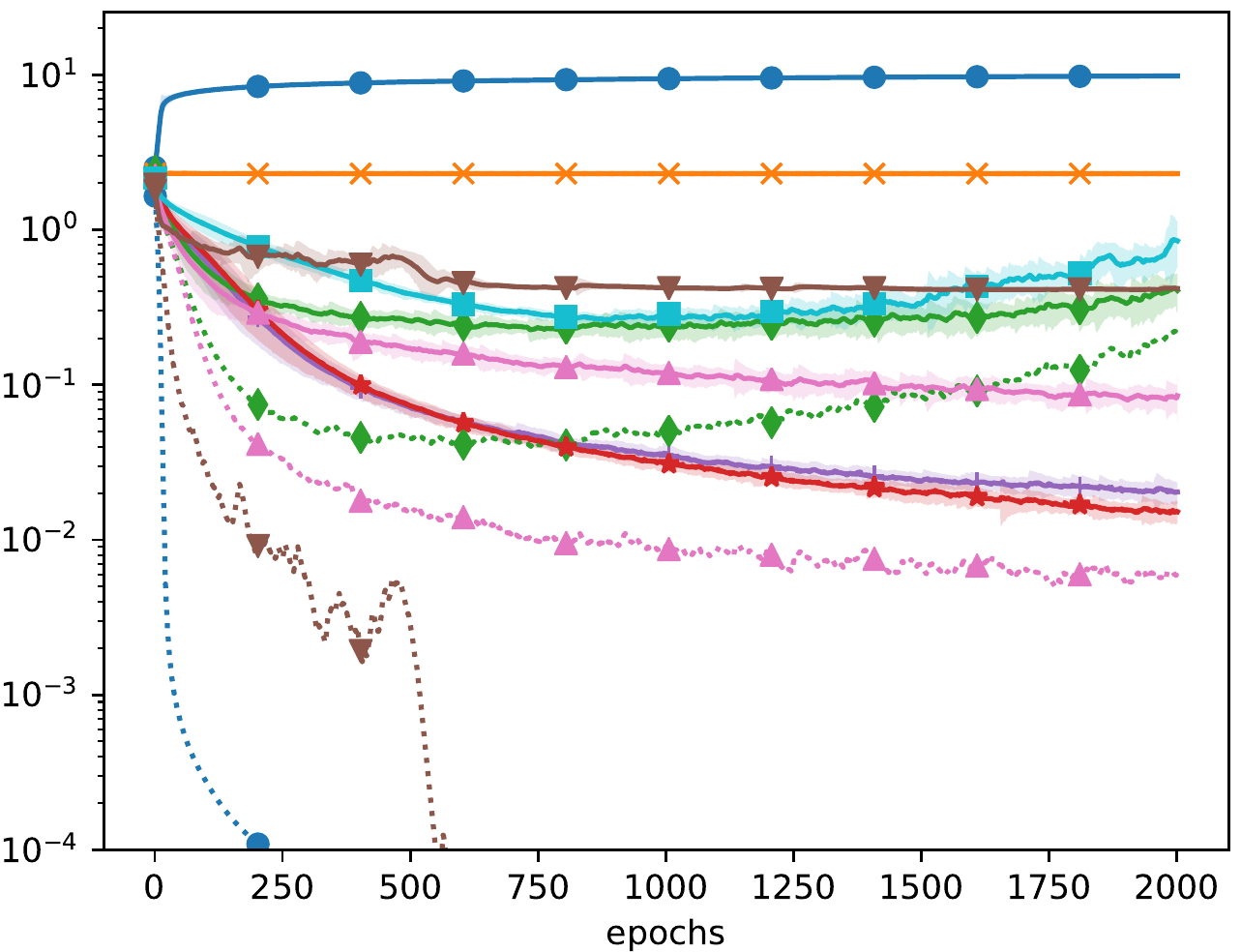}%
\end{tabular}&\begin{tabular}{c}
\includegraphics[width=\figwidth]{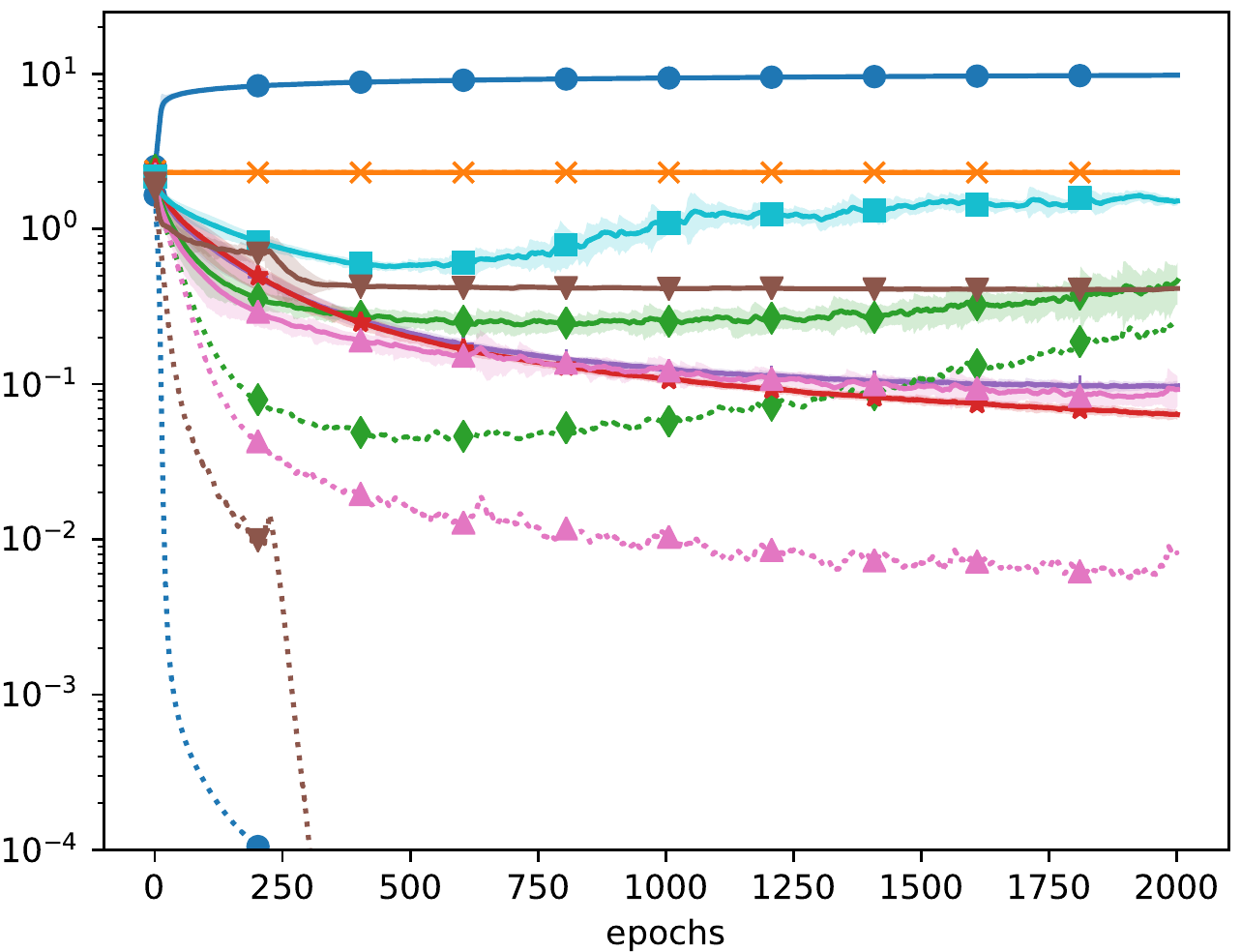}%
\end{tabular}&\begin{tabular}{c}
\includegraphics[width=\figwidth]{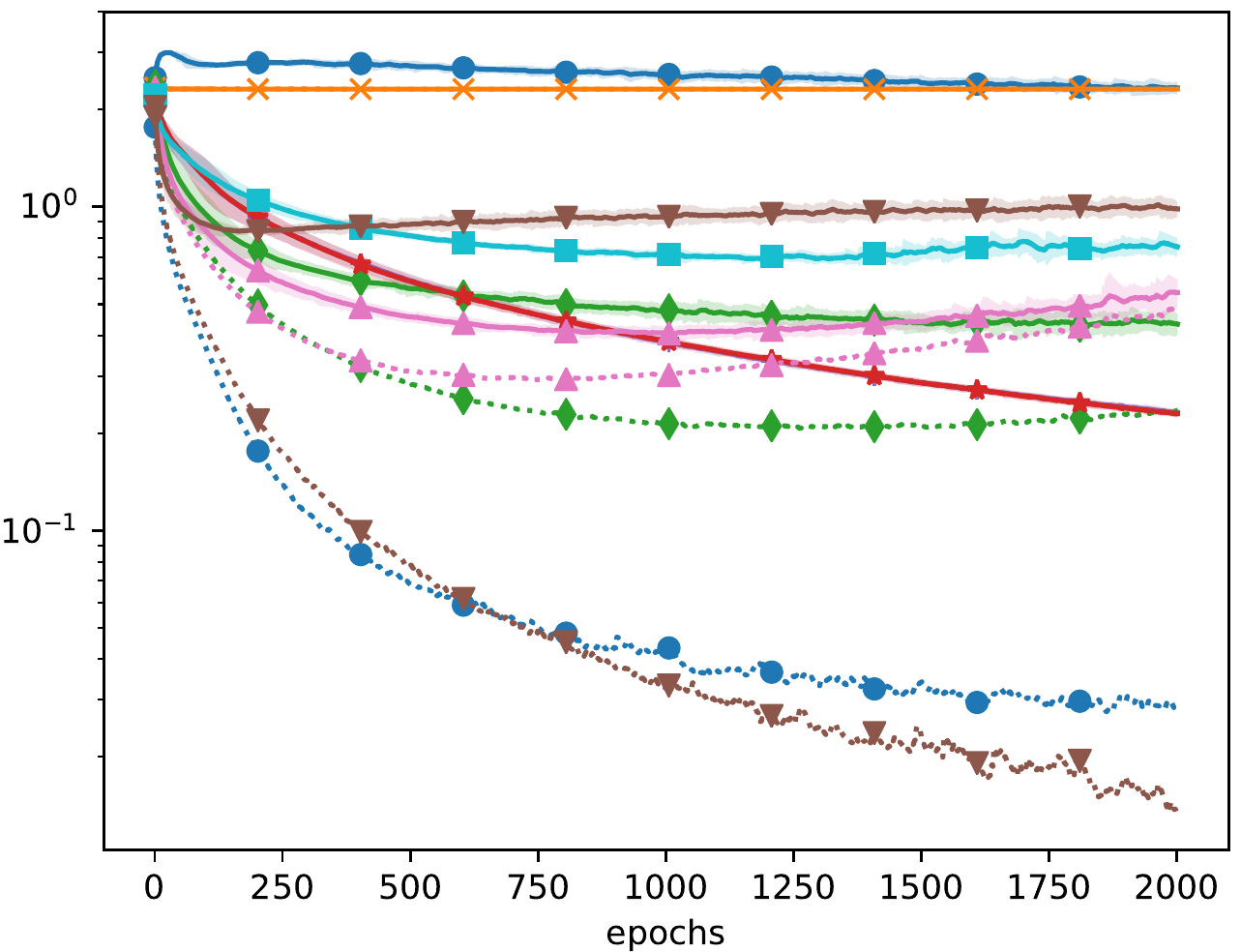}
\end{tabular}\\
%\small Training Loss & \small Training Accuracy & \small Test Accuracy vs Samples \\
\parbox{3mm}{\rotatebox[origin=c]{90}{\small Validation Accuracy, \%}}&
\begin{tabular}{c}
\includegraphics[width=\figwidth]{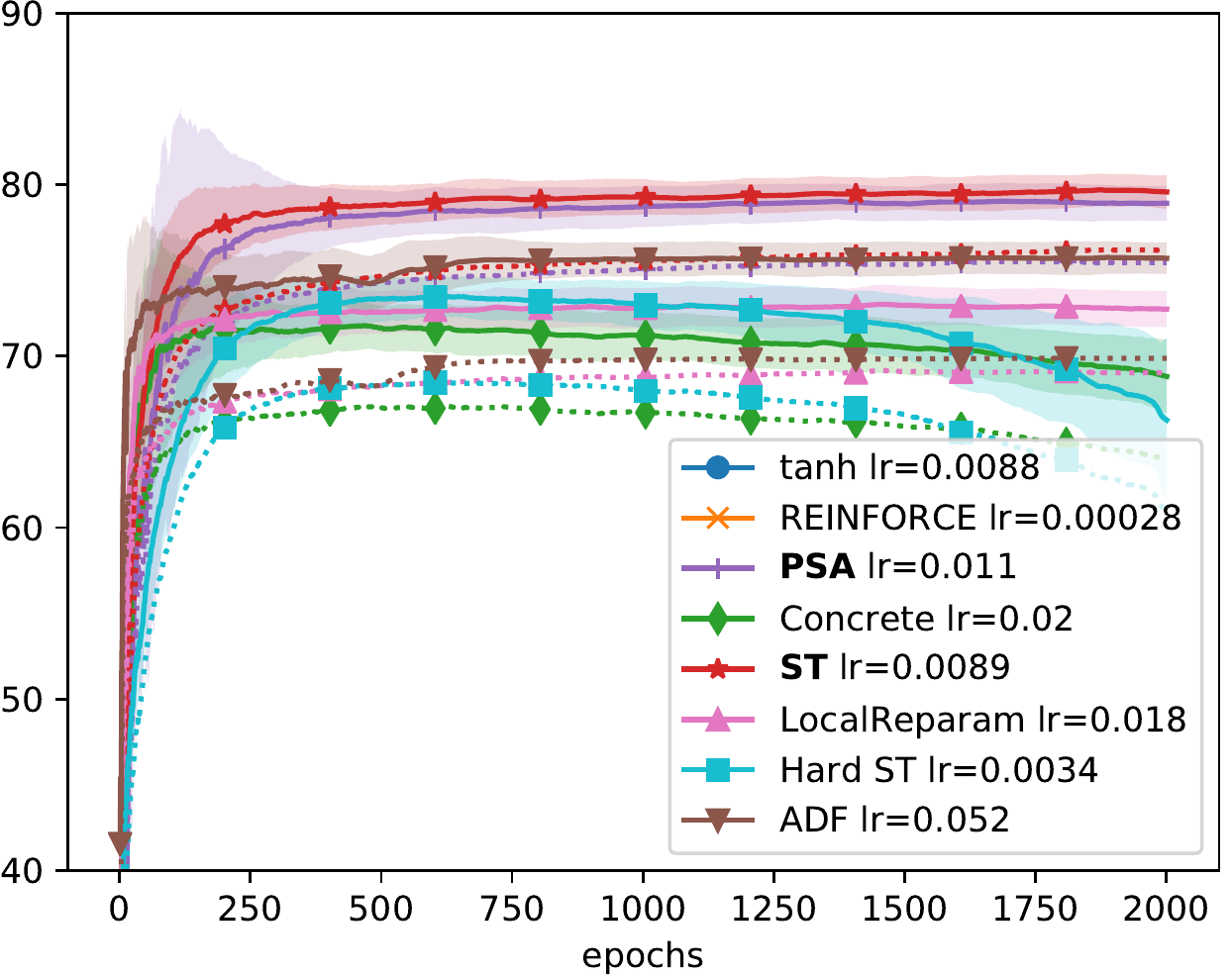}%
\end{tabular}&\begin{tabular}{c}
\includegraphics[width=\figwidth]{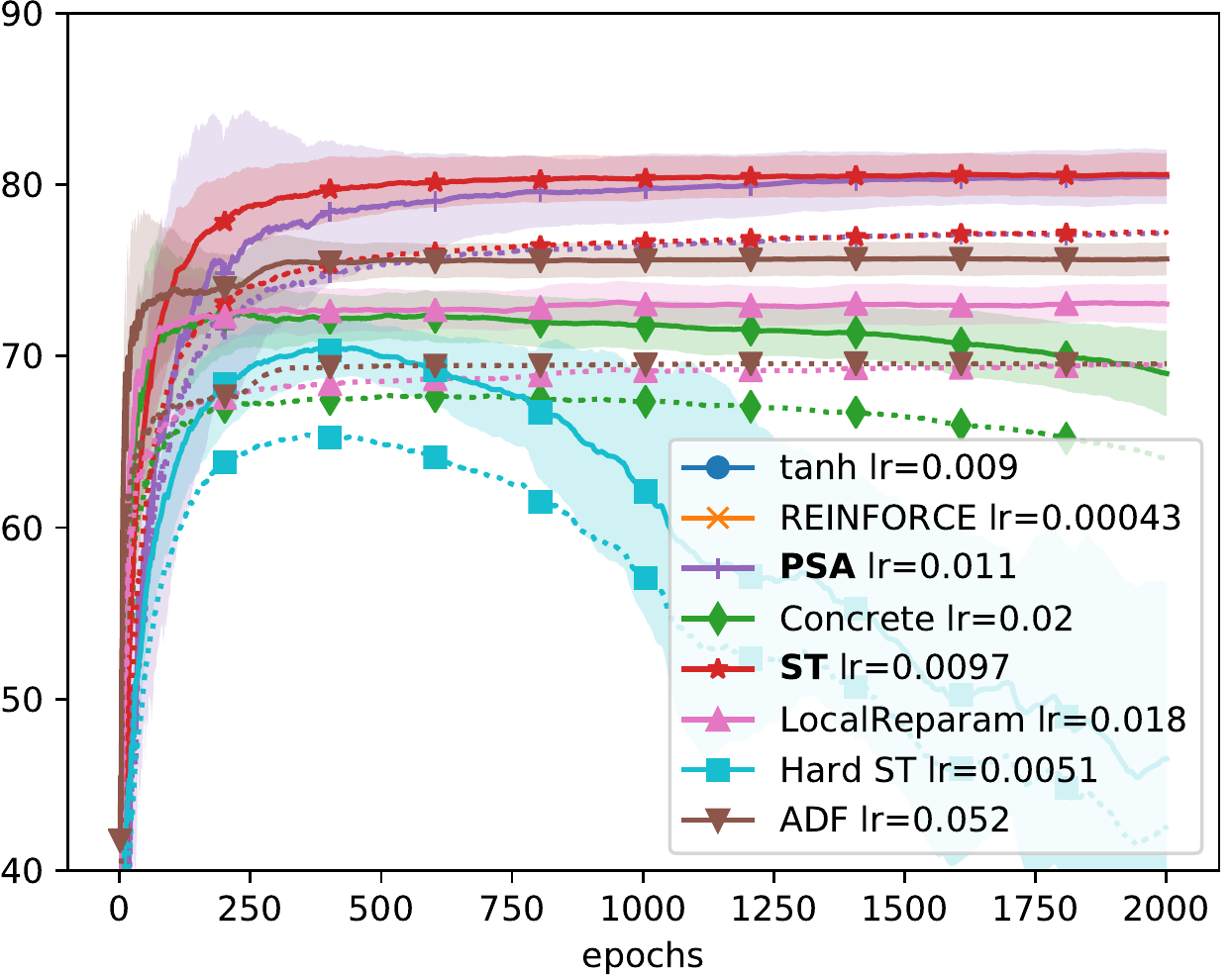}%
\end{tabular}&\begin{tabular}{c}
\includegraphics[width=\figwidth]{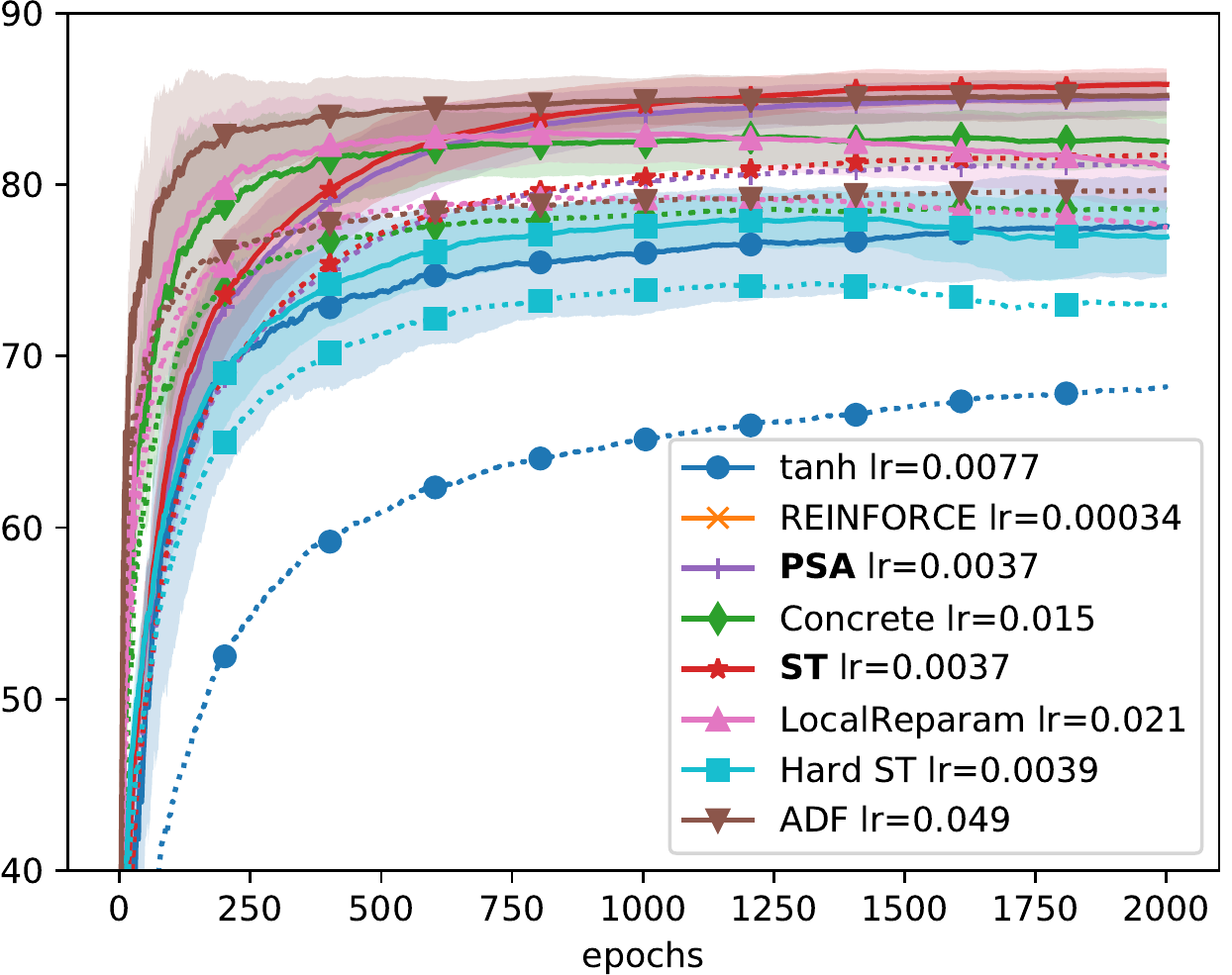}%
\end{tabular}\\
%
%\parbox{3mm}{\rotatebox[origin=c]{90}{\small Validation Loss}}&
%\begin{tabular}{c}
%\includegraphics[width=0.33\linewidth]{fig1/No-augment/val_loss.pdf}%
%\end{tabular}&\begin{tabular}{c}
%\includegraphics[width=0.33\linewidth]{fig1/Affine/val_loss.pdf}%
%\end{tabular}&\begin{tabular}{c}
%\includegraphics[width=0.33\linewidth]{fig1/FlipCrop/val_loss.pdf}%
%\end{tabular}
%
\end{tabular}\\
\raggedleft
\begin{tabular}{p{0.18\linewidth}p{0.68\linewidth}}
\hline
{\tt \REINFORCE} & Unbiased estimator~\cite{Williams1992}.\\
{\tt Tanh} & Replace $\sign(a - Z)$ by $E_Z[\sign(a-Z)]  = \tanh(a/2)$.\\
{\tt Concrete}-$t$ & Concrete Relaxation~\cite{maddison2016concrete} with the relaxation parameter $t$.\\
{\tt HardST} & STE with the gradient of clamped identity, $\max(\min(a,1),-1)$. \\
{\tt ADF} & Assumed density filtering, \eg,~\cite{shekhovtsov18-cat}, the equivalent of {\sc PBNet} method in~\cite{peters2019probabilistic} for real weights.\\
{\tt LocalReparam} & Approximating pre-activations with normal distribution and sampling them.\\ 
\hline
\end{tabular}
\caption{\label{fig:CIFAR-NLL}
Learning comparison on CIFAR-10. Solid loss curves measure the SBN expected loss. Doted loss curves indicate the relaxed objectives used by respective methods (where applicable). Solid accuracy curves are using 10-sample expected predictive probabilities of SBN and dotted curves only 1-sample predictive probabilities. All curves are smoothed over iterations and shaded areas denote $3{\times}$std \wrt smoothing. The automatically found learning rates are displayed in the legend.
}
\end{figure*}
\paragraph{Gradient Estimation Accuracy}
To evaluate the accuracy of gradient estimators, we implemented the exact method, feasible for small models. We use the simple problem instance shown in \cref{fig:example}(a) with 2 classes and 100 training points per class in 2D and a classification network with 5-5-5 Bernoulli units. %\cref{fig:example}(b-c) illustrates the properties of the stochastic binary model and the interpretation as ensemble, trained with \REINFORCE. 
To study the bias-variance tradeoff we vary the number of samples used for an estimate and measure the Mean Squared Error (MSE). For unbiased estimators using $N$ samples leads to a straightforward variance reduction by $1/N$, visible in the log-log plot in~\cref{fig:rmse} as straight lines.
To investigate how the gradient quality changes with the layer depth we measure Root MSE in each of the 3 layers separately.
%. \cref{fig:rmse} shows the curves of RMSE error relative to the length of the true gradient in each of the 3 layers of the model. 
It is seen in~\cref{fig:rmse} that the proposed \PSA method has a bias, which is the asymptotic value of RMSE when increasing the number of samples. However, its RMSE accuracy with 1 sample is indeed not worse than that of the advanced unbiased \ARM method with $10^3$ samples.
We should note that in more deterministic points, where the network converges to during the training, the conditions for approximation hold less well and the bias of \PSA may become more significant while unbiased methods become more accurate (but the gradient signal diminishes).
\cref{fig:rmse} also confirms experimentally that \PSA is always more accurate than ST and has no bias in the last hidden layer, as expected.
Additional experiments studying the dependence of accuracy on network width, depth and comparison with more unbiased baselines are given in~\cref{sec:moreexp}~(\cref{fig:acc-width,fig:acc-depth,fig:rmse1}). Both \PSA and ST methods are found to be in advantage when increasing depth and width.

The cosine similarity metric measured in~\cref{fig:accuracy} is more directly relevant for optimization.
%Another metric, more directly relevant for optimization is the cosine similarity. 
If it is close to one, we have an accurate gradient direction estimate. If it is positive, we still have a descent direction. Negative cosine similarity will seriously harm the optimization. 
For this evaluation we take the model parameters at epochs 1, 100 and 2000 of a reference training and measure the cosine similarity of gradients in layer 1. %Results in \cref{fig:accuracy} show that the ranking of the methods by cosine similarity changes during training. 
We see that methods with high bias may systematically fail to produce a descent direction while methods with high variance may produce wrong directions too often. Both effects can slow down the optimization or steer it wrongly.

The proposed \PSA method achieves the best accuracy in the practically important low-sample regime. The ST method is often found inferior and we know why: the extra linearization does not hold well when there are only few units in each layer. We expect it to be more accurate in larger models.
%
%\begin{table}
%\caption{List of baseline methods\label{tab:baselines}}
%%\small
%\begin{tabular}{p{0.2\linewidth}p{0.8\linewidth}}
%\hline
%{\tt \REINFORCE} & Unbiased estimator~\cite{Williams1992}.\\
%{\tt Tanh} & Replace $\sign(a - Z)$ by $E_Z[\sign(a-Z)]  = \tanh(a/2)$.\\
%{\tt Concrete}-$t$ & Concrete Relaxation~\cite{maddison2016concrete} with the relaxation parameter $t$.\\
%{\tt HardST} & STE with the gradient of clamped identity, $\max(\min(a,1),-1)$. \\
%{\tt ADF} & Assumed density filtering, \eg,~\cite{shekhovtsov18-cat}, the equivalent of {\sc PBNet} method in~\cite{peters2019probabilistic} for real weights.\\
%{\tt LocalReparam} & Approximating pre-activations with normal distribution and sampling them.\\ 
%\hline
%\end{tabular}
%\end{table}

\paragraph{Deep Learning}
To test the proposed methods in a realistic learning setting we use CIFAR-10 dataset and network with 8 convolutional and 1 fully connected layers (\cref{sec:moreexp}). The first and foremost purpose of the experiment is to assess how the methods can optimize the training loss. We thus avoid using batch normalization, max-pooling and huge fully connected layers. 
When comparing to a number of existing techniques, we find it infeasible to optimize all hyper-parameters such as learning rate, schedule, momentum \etc per method by cross-validation. However, compared methods significantly differ in variance and may require significantly different parameters.
We opt to use SGD with momentum with a constant learning rate found by an automated search per method (\cref{sec:moreexp}). 
While this may be suboptimal, it nevertheless allows to compare the behavior of algorithms and analyze the failure cases. We further rely on SGD to average out noisy gradients with a suitable learning rate and therefore use 1-sample gradient estimates. To modulate the problem difficulty, we evaluate 3 data augmentation scenarios: no augmentation, affine augmentation, Flip\&Crop augmentation.  

\cref{fig:CIFAR-NLL} and \cref{fig:CIFAR-extra} show the training performance of evaluated methods. Both \PSA and \ST achieve very similar and stable training performance. This verifies that the extra linearization in \ST has no negative impact on the approximation quality in large models. For methods {\tt Tanh}, {\tt Concrete}, {\tt ADF}, {\tt LocalReparam} we can measure their relaxed objectives, whose gradients are used for the training (\eg the loss of a standard neural network with $\tanh$ activations for {\tt Tanh}). The training loss plots reveal a significant gap between these relaxed objectives and the expected loss of the SBN. While the relaxed objectives are optimized with an excellent performance, the real objective stalls or starts growing. This agrees with our findings in~\cref{fig:accuracy} that biased methods may fail to provide descent directions. {\tt HardST} method, seemingly similar to our ST, performs rather poorly. Despite its learning rate is smaller than that of ST, it diverges in the first two augmentation cases, presumably due to wrongly estimated gradient directions. 
As we know from preceding work, good results with these existing methods are possible, in particular we also see that {\tt ADF} with Flip\&Crop augmentation achieves very good validation accuracy despite poor losses. We argue that in methods where bias may become high there is no sufficient control of what the optimization is doing and one needs to balance with empirical guessing.
Finally, the \REINFORCE method requires a very small learning rate in order not to diverge and the learning rate is indeed so small that we do not see the progress. We investigate if further by attempting input-dependent baselines in~\cref{sec:muprop}. While it can optimize the objective in principle, the learning rate stays very small.

%Further details of implementation, training setup and additional experiments can be found in~\cref{sec:impl,sec:moreexp}.
%Our implementation of PSA in PyTorch with custom convolutions in CUDA is available at~\url{http://gitlab.com}.
Please see further details on implementation and the training setup in~\cref{sec:impl,sec:moreexp}. The implementation is available at~{\small \url{https://github.com/shekhovt/PSA-Neurips2020}}.

\section{Conclusion}
We proposed a new method for estimating gradients in SBNs, which combines an approximation by linearization and a variance reduction by summation along paths, both clearly interpretable. We experimentally verified that our \PSA method has a practical bias-variance tradeoff in the quality of the estimate, runs with a reasonable constant factor overhead, as compared to standard back propagation, and can improve the learning performance and stability. The {\sc st} estimator obtained from \PSA gives the first theoretical justification of straight-through methods for deep models, opening the way for their more reliable and understandable use and improvements. Its main advantage is implementation simplicity. While the estimation accuracy may suffer in small models, we have observed that it performs on par with \PSA in learning large models, which is very encouraging. However, the influence of the systematic bias on the training is not fully clear to us, we hope to study it further in the future work.
%
%\clearpage
\section*{Broader Impact}
The work promotes stochastic binary networks and improves the understanding and efficiency of training methods. We therefore believe ethical concerns are not applicable. 
At the same time, developing more efficient training methods for binary networks, we believe may further increase the researchers and engineers interest in low-energy binary computations and aid progress in embedded applications such as speech recognition and vision. In the field of stochastic computing, which is rather detached at the moment, the stochasticity is treated as a source of errors and accumulators are used in every layer just to mimic smooth activation function~\cite{Kim-16,Lee:2017}. It appears to us that when stochastic binary computations are made useful instead, the related hardware designs can be made more efficient and stable \wrt to errors. 

\begin{ack}
We thank the anonymous peers for pointing out related work and helping to clarify the presentation.
We gratefully acknowledge our funding agencies.
A.S. was supported by the project ``International Mobility of Researchers MSCA-IF II at CTU in Prague'' (CZ.02.2.69/0.0/0.0/18\_070/0010457). V.Y. was supported by Samsung Research, Samsung Electronics. B.F. was supported by the Czech Science Foundation grant no. 19-09967S.
\end{ack}
	
\small
\bibliography{../bib/strings,../bib/neuro-generative,../bib/our}
\bibliographystyle{icml2020}

\newpage
\newpage
\onecolumn
%\newgeometry{left=2cm, righ=2cm,top=2cm,bottom=2cm}
%\newgeometry{includefoot,margin=2cm,bottom=1in}
\appendix
\numberwithin{figure}{section}
%\addtocontents{toc}{\protect\setcounter{tocdepth}{2}}
\pagestyle{plain}
%
% Reset counters
\setcounter{figure}{0}
\setcounter{table}{0}
\counterwithin{figure}{section}
\counterwithin{table}{section}
\counterwithin{theorem}{section}
\counterwithin{proposition}{section}
\counterwithin{lemma}{section}
\counterwithin{corollary}{section}
%
%\title{\mytitle (Appendix)}
%\author{Paper ID \SubNumber \\[10pt]}
%\author{ACML Submission}
%\title{Appendix}
%\author{}
%\maketitle
%
%%\appendices
%\appendix
\addtocontents{toc}{\protect\setcounter{tocdepth}{2}}
%\pagestyle{plain}
%
%% Reset counters
%\setcounter{figure}{0}
%\setcounter{table}{0}
%\counterwithin{figure}{section}
%\counterwithin{table}{section}
%
%\twocolumn[{%
 %\centering
 %\LARGE \mytitle \\ %Supplementary Material \\[1.5em] 
 %(CVPR Submission \#\cvprPaperID\ Supplementary Material) \\[1.5em]
 %\normalsize
%}]
%
% TOC
%remove dots and page numbers
\let\Contentsline\contentsline
\renewcommand\contentsline[3]{\Contentsline{#1}{#2}{}}
\makeatletter
\renewcommand{\@dotsep}{10000}
\makeatother

\let\authcount\relax
\def\authcount#1{}
{\centering
 \LARGE %\mytitle \\ %Supplementary Material \\[1.5em] 
 \mytitle\ (Appendix) \\[1.5em]
\normalsize
}
\tableofcontents
\section{Learning Formulations}\label{sec:learning}
%Before deriving our methods, 
Here we give a clarification on the learning formulation used in this work. During the training we consider the expected loss of a randomized predictor:
\begin{align}\label{exp-loss}
\E_{(x^0, y)} \E_Z\big[F(\theta)\big] = -\E_{(x^0, y)} {\red \E_Z} \big[\log p(y| X^L; \theta)\big],
\end{align}
where $(x^0, y) \sim \texttt{data}$. However at the test time we evaluate the expected predictor $\E_Z[ p(y| X^L; \theta)]$, considering $Z$ as latent variables, which can be interpreted as an {\em ensemble} of binary networks (see ~\cref{fig:example} c-f). The loss of this marginal predictor would be rather given by:
\begin{align}\label{marg-loss}
- \E_{(x^0, y)}\big[ \log {\red \E_{Z}} \big[ p(y| X^L; \theta) \big]\big].
\end{align}
This setup is similar to dropout~\cite{srivastava14a} with latent multiplicative Bernoulli noises. The expected loss~\eqref{exp-loss} upper bounds the marginal loss~\eqref{marg-loss} (by Jensen's inequality), so that minimizing it also minimizes~\eqref{marg-loss}. However, unlike with dropout, the following observation holds for SBN models.
\begin{restatable}{proposition}{Pdetermenistic}\label{prop:deterministic}
In the family of models~\eqref{SBN} with free scale and bias in all coordinates of $a^k$ there is always an effectively deterministic model (with no injected noises) that achieves the same or better expected loss~\eqref{exp-loss}.
\end{restatable}

This means that the model will tend to be deterministic and fit the classification boundary but not the data ambiguity (see~\cref{fig:example} b,c). Being aware of this, we note that it is nevertheless a common way to train classification models (\eg, dropout). Furthermore, the upper bound may be tightened by considering a multi-sample bound~\citep{Tang-13, Raiko-14} or variational bounds (applied for shallow models in~\cite{yin18-arm,grathwohl18-relax}). These extensions are left for future work as they require the ability to estimate gradients of the expected loss~\eqref{exp-loss} in the first place. We can nevertheless see from the example in~\cref{fig:example} that the SBN family can be expressive when trained appropriately.

\section{Proofs}\label{sec:proofs}
%\Lsinglelayer*
This section contains proofs, technical details and extended discussion that did not fit in the main paper.
\subsection{Training with Expected Loss tends to Deterministic Models}
\Pdetermenistic*
\citet{Raiko-14} give a related theorem, but do not show the preferred deterministic strategy to be realizable in the model family.
%\begin{proof}
%When the scale and bias of all coordinates $a^k_i$ of activations in all layers can be set arbitrary, %it is then equivalent to consider that all injected noises $Z^k_i$ have a free scale and bias. Equivalently, 
%the model becomes equivalent to
%\begin{align}
%X^{k}_i = \sgn(a^k_i(X^{k-1}; \theta) - s^k_i Z^k_i  + b^k_i),
%\end{align}
%where $s^k_i >0 $ and $b^k_i$ are new free scale and bias parameters per noise component. Since the $\sign$ function is invariant to multiplication of the argument by a positive number, these scale and bias parameters can be equivalently incorporated in $a$. 
%
%
%Let $\theta$ be parameters of the model optimizing~\eqref{randomized-predictor}.
%Let then $z^*$ be a maximizer of $\E_{(x^0, y^*) \sim \text{data}} \Big[f(x^L(z),y^*;\theta) \Big]$. 
%%Consider the case of a linear layer $a(x) = W^T x + b$ with the output $\sgn(W^T x + b - Z)$. 
%%Chose as new parameters $W' = s W$, $b' = s(b - z^*)$ for $s \rightarrow \infty$. 
%We can then take the limit $s^k_i \rightarrow 0$.
%Since $Z$ has a finite variance, this ensures that $\sgn(a^k_i(X^{k-1}; \theta) - s^k_i Z^k_i  + b^k_i) \rightarrow$. The network with new weights is deterministic as it efficiently scales all noises to zero and it achieves same or better expected loss. The same argument applies whenever $a$ has a free scale and bias degrees of freedom.
%\end{proof}
%
\begin{proof}
Let $\theta$ be parameters of the model optimizing~\eqref{exp-loss}. Let then $z^*$ be a maximizer of $\E_{(x^0, y^*) \sim \text{data}} \Big[f(x^L(z),y;\theta) \Big]$. Consider the case of a linear layer $a(x) = W^T x + b$ with the output $\sgn(W^T x + b - Z)$. Chose as new parameters $W' = s W$, $b' = s(b - z^*)$ for $s \rightarrow \infty$. Since $Z$ has a finite variance, this ensures that $\sgn( W'^T X + b' - Z ) = \sgn( W^T X + b - z^* - Z/s ) \rightarrow \sgn( W^T X + b - z^* )$. The network with new weights is deterministic as it efficiently scales all noises to zero and it achieves same or better expected loss. The same argument applies whenever $a$ has a free scale and bias degrees of freedom.
\end{proof}
\setlength{\figwidth}{0.165\textwidth}
\begin{figure*}[t]
\centering
\begin{tabular}{c}
\includegraphics[width=\figwidth]{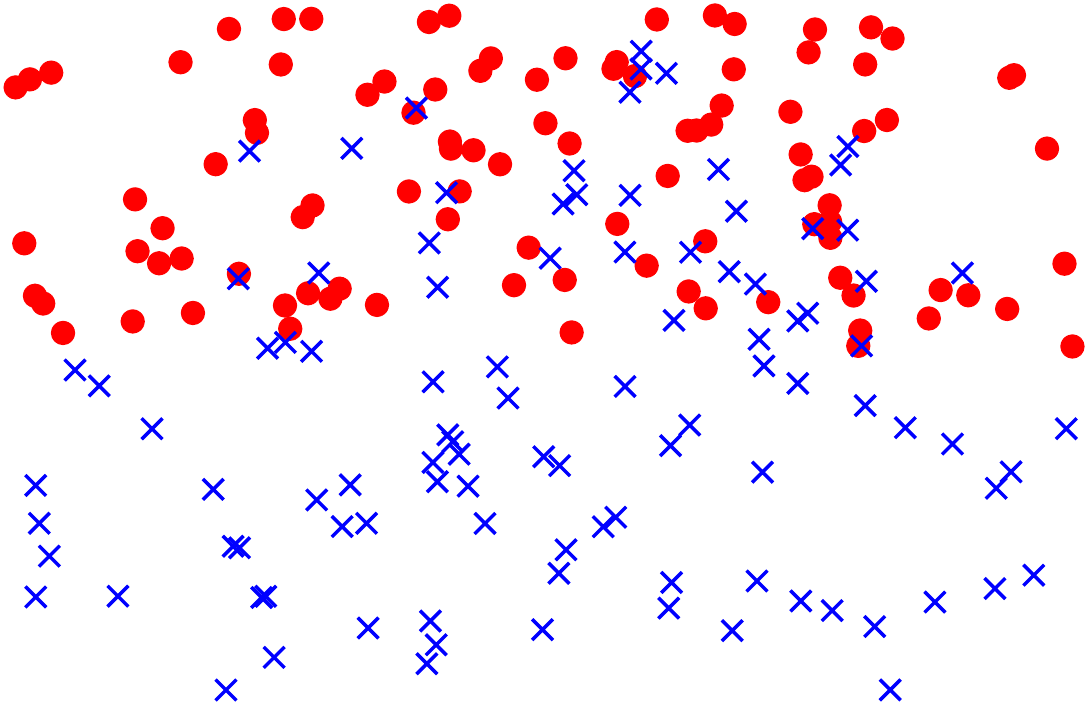}\\%
(a)
\end{tabular}%
\begin{tabular}{c}
\includegraphics[width=\figwidth]{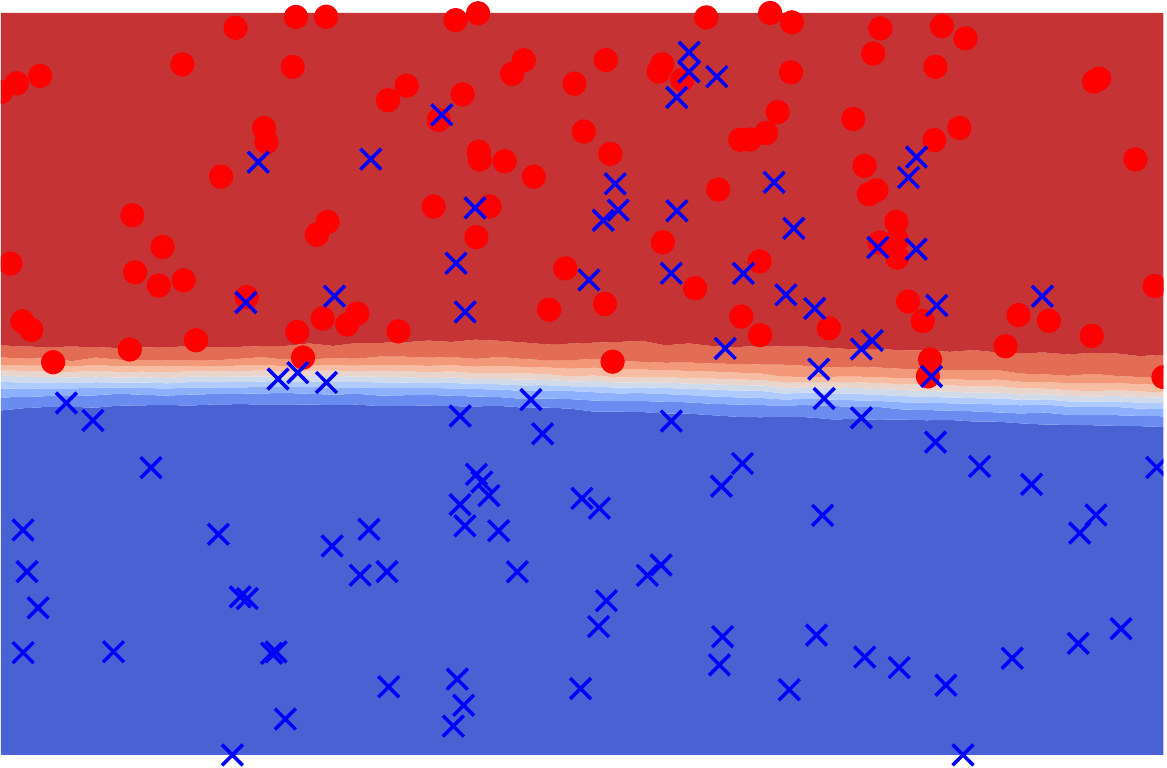}\\%
(b)
\end{tabular}%
\begin{tabular}{c}
\includegraphics[width=\figwidth]{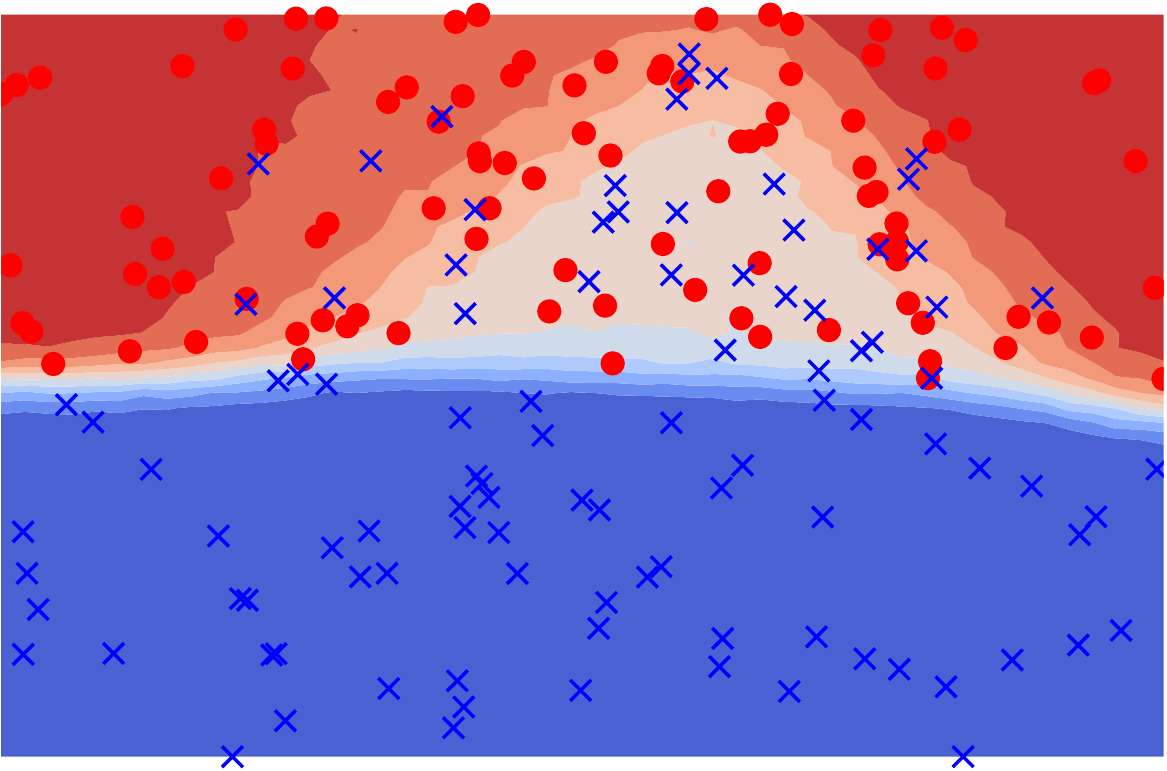}\\
(c)%
\end{tabular}%
\begin{tabular}{c}
\includegraphics[width=\figwidth]{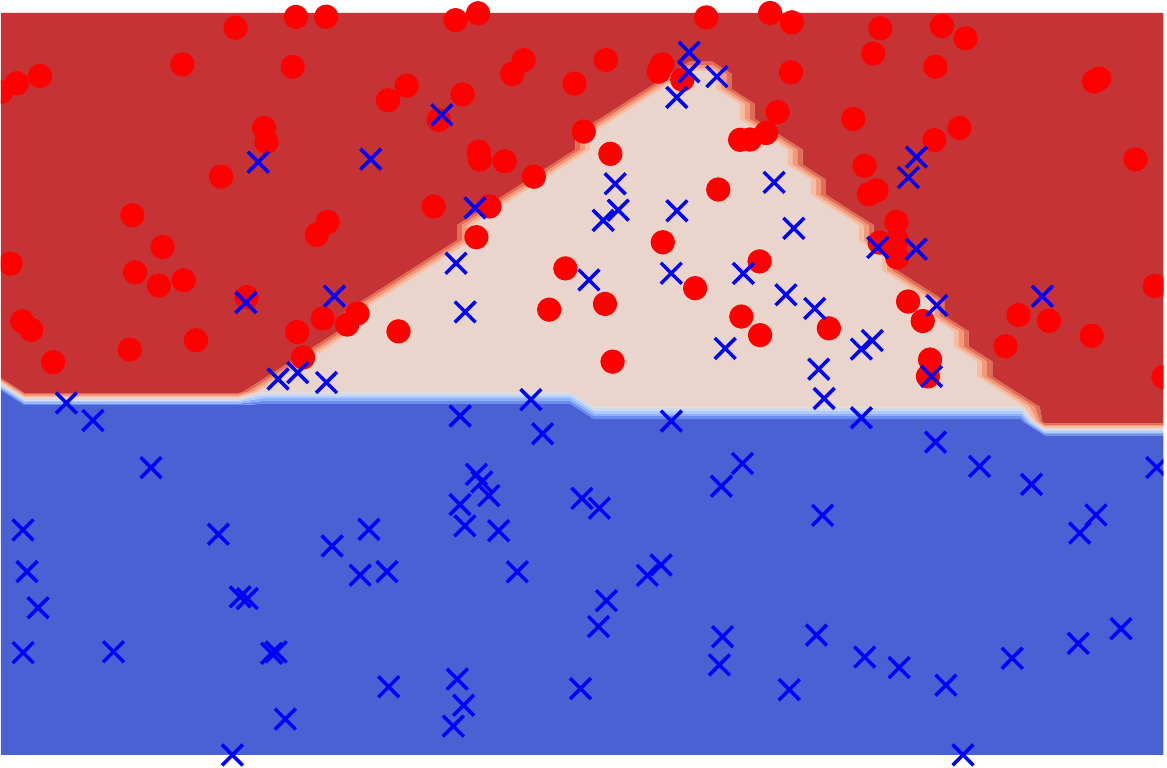}\\
(d)
\end{tabular}%
\begin{tabular}{c}
\includegraphics[width=\figwidth]{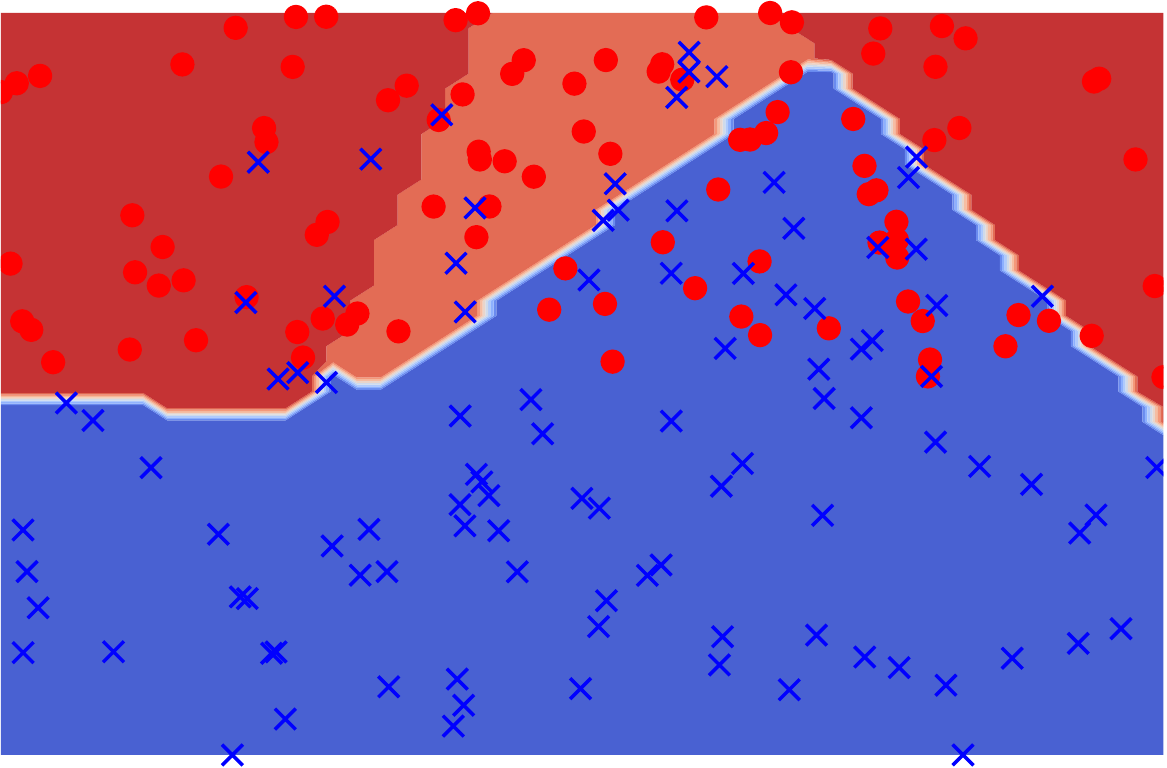}\\
(e)
\end{tabular}%
\begin{tabular}{c}
\includegraphics[width=\figwidth]{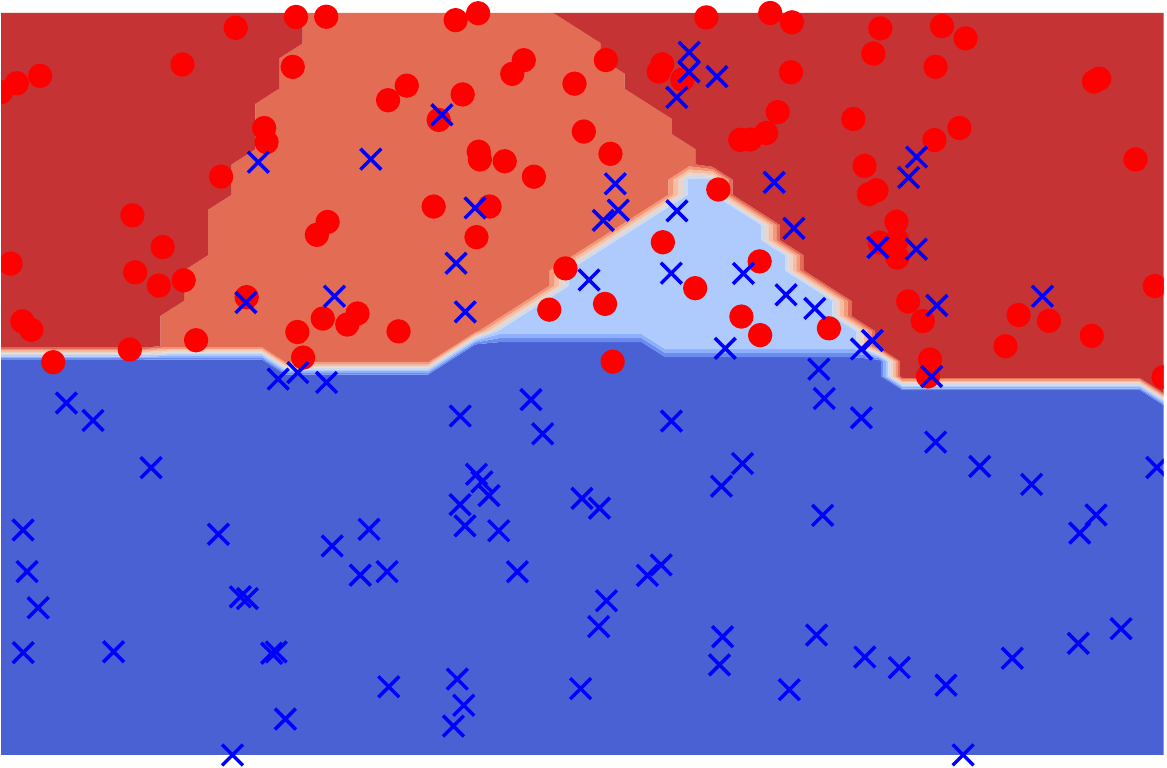}\\
(f)
\end{tabular}%
\vskip-0.5\baselineskip
\caption{\label{fig:example}
Example problem to classify points in 2D with overlapping distributions. (a) Data points. (b) Classification model trained with the expected loss~\eqref{exp-loss}: the optimal solution tends to deterministic prediction. (c) Same model trained with a 10-sample bound~\cite{Raiko-14}, closer to the marginal likelihood~\eqref{marg-loss}. The model fits the uncertainty of the data.
(d-f) Examples of the ensemble members obtained by fixing a particular realization of the noise variables $Z$ in all layer for the model in (c).
}
\end{figure*}
Let us remark that the conditions are not met in the following cases:
\begin{itemize}
\item The pre-activation does not have some degree of freedom, \eg, there is no bias term. This case is obvious.
\item Pre-activations of different outputs do not have independent degrees of freedom per output. \eg, in a convolutional network we can suppress all the noises by scaling them down, however since the noises of all pre-activations are independent (not spatially identical), we cannot represent the bias from $z^*$ with the convolution bias which is spatially homogenous.
\item The network uses parameter sharing in some other way, \eg, a Siamese network for matching. % or our construction~\cref{MSE1} for the MSE loss combining two samples from the same model.
\end{itemize}

These exceptions actually imply that training with expected loss a convolutional network in~\cref{sec:experiments} tends to be deterministic but will not collapse to a fully deterministic state as it is suboptimal.

\subsection{PSA Derivation and Properties}
\Pdelta*
\begin{proof}
%The approximation consists of an equivalent transform step and an elementary approximation step. 
Starting from LHS of~\eqref{direct-prop-eq1} we take the sum in $x^k_i$ explicitly. The factors involving $x^k_i$ (after cancellation of the denominator with the respective term in $p(x)$) are
\begin{align}\label{term-to-sum}
p(x^{k+1} | x^k) J^k_i(x^{k}),
\end{align}
where we omit the dependance of $J^k_i$ on $x^{1 \dots k-1}$, not relevant for the sum in $x^k_i$. Using the oddness of $J^k_i$, the sum of~\eqref{term-to-sum} in $x^k_i$ can be written as
\begin{align}\label{derandomize_x}
\textstyle p(x^{k+1} | x^k) J^k_i(x^k) + p(x^{k+1} | x^k_{\downarrow i}) J^k_i(x^{k}_{\downarrow i}) 
= \big(p(x^{k+1} | x^k) - p(x^{k+1} | x^k_{\downarrow i})\big) J^k_i(x^{k}).
\end{align}
%where $x^k_{\downarrow i}$ denotes flipping the state of $x^k_i$,
%\begin{align}
%& = \Big( p(x^{k+1} | x^k) - p(x^{k+1} | x^k_{\downarrow i}) \Big) \frac{\partial}{\partial \theta} p(x^k_i| x^{k-1}; \theta).
%\end{align}
Though this expression formally depends on $x^k_i$, it is by design invariant to $x^k_i$. Thus $x^k_i$ has been {\em derandomized}.
We multiply~\eqref{derandomize_x} with $1 = \sum_{x^k_i}p(x^k_i | x^{k-1})$ to obtain
\begin{align}\label{derandomize-expect}
\textstyle \sum_{x^k_i}p(x^k_i | x^{k-1}) \big( p(x^{k+1} | x^k) - p(x^{k+1} | x^k_{\downarrow i}) \big) J^k_i(x^{k}),
\end{align}
%Though this expression is 
%This is needed in order to put this part back into joint expectation over $x$.
%This step we do in order to preserve the form $\sum_x p(x)$ in \eqref{direct-part-1} in order to replace it by a sample later on. 
%Substituting 
%This allows to substitute~\eqref{derandomize-expect} back 
which allows to put this expression back as a part of the joint sum in $x$ in~\eqref{direct-prop-eq1}. We thus obtain in~\eqref{direct-prop-eq1}:
%The factors not involving $x^k_i$ in \eqref{direct-part-1} are
\begin{align}\label{prop-direct-proof-eq2}
\textstyle \sum_{i, x} p(x^{1 \dots k}) p(x^{k+2 \dots L} | x^{k+1}) 
\big( p(x^{k+1} | x^k) - p(x^{k+1} | x^k_{\downarrow i}) \big) J^k_i(x) f(x^L).
\end{align}
Recalling that $p(x^{k+1} | x^k) = \prod_{j} p(x^{k+1}_j | x^k)$, the product linearization~\eqref{approx-prod} gives
\begin{align}
\hskip-3pt \textstyle p(x^{k+1} | x^k){-}p(x^{k+1} | x^k_{\downarrow i}) \approx \sum_j \frac{p(x^{k+1} | x^k)}{p(x^{k+1}_j | x^k)} \Delta^{k+1}_{i, j}(x),
\end{align}
where the division is used to represent the factor that needs to be excluded.
Substituting this into~\eqref{prop-direct-proof-eq2} we get the resulting expression in~\eqref{direct-prop-eq1}. Finally, $\Delta^{k+1}_{i, j}(x)$ is odd in $x^{k+1}_j$ and $J^k(x)$ does not depend on $x^{k+1}_j$, and therefore $J^{k+1}_{j}$ is odd in $x^{k+1}_j$.
\end{proof}

%We now make the presented construction precise. First consider the simple case of derivative in last layer's parameters.
\paragraph{Case $\bm{l{=}L}$} 
%\paragraph{Last Layer}
%The sum over $x$ can be expressed as
We now prove the expression~\eqref{last-layer-case} claimed as the derandomization result for the last layer when propagating $J^L$. %, identical to derandomization of the gradient in the last layer $D^L$. 
Let us consider gradient in parameters of the last layer, $g^L$. In this case, the gradient expression~\eqref{d^l} becomes %we can be rewritten with singling out the summation in $x^{L}$ as
\begin{align}
\textstyle \sum\limits_{x^{1\dots L-1}} p(x) \sum_i \sum\limits_{x^{L}} \frac{d^L_i(x)}{p(x^L|x^{L-1})} f(x^L).
\end{align}
%The following lemma applies derandomization to the summation over $x^L$.
Then derandomization over $x^L$ takes a particular simple form, which can be described by the following standalone lemma.
%Analogous results exist in the literature, \eg~\cite{cong2018go} considers general discrete and continuous distributions.
%
\begin{restatable}{lemma}{Lsinglelayer}\label{L:singlelayer}
Let $X_i$ be independent $\Bool$-valued Bernoulli with probability $p(x_i; \theta)$ for $i=1\dots n$ and $f\colon \Bool^n \to \Real$. Let $x$ be a joint sample and $x_{\downarrow i}$ denote the joint state with $x_i$ flipped. Then
\begin{align}\label{last-layer-grad}
\textstyle \sum_{i}\frac{\partial }{\partial \theta} p(x_i; \theta) \big(f(x) - f(x_{\downarrow i})\big)
\end{align}
is an unbiased estimate of $\frac{\partial}{\partial \theta} \sum_{x}p(x; \theta) f(x)$.
\end{restatable}
Analogous results exist in the literature, \eg~\cite{cong2018go} considers general discrete and continuous distributions. 
%The proof below demonstrates how to compute expectation over $x^L_i$ explicitly and at the same time retain $x_i$ in sampling.
\begin{proof}
We differentiate the product of probabilities in the expectation:
\begin{align}
\textstyle \frac{\partial }{\partial \theta} \sum\limits_x \prod\limits_i p(x_i; \theta) f(x)
= \sum\limits_{x} \sum\limits_{i} \frac{p(x)}{p(x_i)} \frac{\partial }{\partial \theta} p(x_i;\theta) f(x).
\end{align}
%Note that the {\sc reinforce} estimator would just approximate $\sum_x p(x)$ with a sample. 
%Instead, 
We then compute the sum over $x_i$ for each summand $i$ explicitly, obtaining
\begin{align}
\notag \textstyle \sum\limits_i \sum\limits_{x_{\lnot i}} p(x_{\lnot i}) \frac{\partial }{\partial \theta} \Big( p(x_i;\theta) f(x) + (1- p(x_i;\theta)) f(x_{\downarrow i}) \Big),
\end{align}
where $x_{\lnot i}$ denotes excluding the component $i$. Since the expression in the brackets is invariant of $x_i$, we multiply by the factor $1 = \sum_{x_i} p(x_i)$ and get
\begin{align}\label{last-layer-grad-sum}
& \textstyle \sum_{x} p(x) \sum_i \big( f(x) - f(x_{\downarrow i}) \big) \frac{\partial }{\partial \theta} p(x_i).
\end{align}
Thus~\eqref{last-layer-grad} is a single sample unbiased estimate of~\eqref{last-layer-grad-sum}. %Differently from {\sc reinforce}, this computes part of the sum explicitly and thus reduces the variance. It needs however $n$ discrete derivatives of the function $f$.
\end{proof}

\Punbiased*
\begin{proof}
The product linearization is not used anywhere in the method when we have a single binary unit in each hidden layer with $l<L$, nor it is used in the last layer. We therefore make no approximations and the 1-sample estimate is unbiased.
\end{proof}

\subsection{Last Layer Enhancement}
In \cref{alg:PSA}~\cref{last-layer} we defined $E$ so that the gradient in  $\theta^{L+1}$ is the common stochastic estimate $\frac{\partial f(x^L;\theta^{L+1})}{\partial \theta^{L+1}}$. We now propose an improvement to this estimate. % estimate of $\sum_x p(x) \frac{\partial }{\partial \theta} f(x^L; \theta)$. 
Intuitively, we want to utilize the values $f(x^{L}_{\downarrow i}; \theta^{L+1})$ for all $i$ that we compute anyway. %at low cost and reducing the variance.
%An idea why this should be possible is that we can potentially at low cost utilize also values $\frac{\partial }{\partial \theta} f(x^{L}_{\downarrow i}; \theta)$ for all $i$ to reduce the variance.

Estimating $\sum_x p(x) \frac{\partial }{\partial \theta} f(x^L; \theta)$ means to estimate the expected value of the function $g(x^L) = \frac{\partial }{\partial \theta} f(x^L; \theta)$, without further derivatives involved.
We have the following lemma that applies derandomization over units in the last layer.

\begin{restatable}{lemma}{Lexpectedval}\label{Bernoulli-lemma-2}
Let $X_i$ be independent $\Bool$-valued Bernoulli with probability $p(x_i)$ for $i=1\dots n$ and $g\colon \Bool^n \to \Real$. Let $x$ be a joint sample. Then
\begin{align}
\textstyle g(x) + \gamma \sum_{i}\big(g(x_{\downarrow i}) - g(x)\big) (1-p(x_i))%q_i
\end{align}
is an unbiased estimate of $\E_X[g(X)]$. %,%where $q_i = 1-p(x_i)$ is the probability of the opposite state to $x_i$.
\end{restatable}
%The proof is similar to the proof of~\cref{L:singlelayer} and is given in~\cref{sec:proofs} and 
%See~\cref{sec:proofs} for details.
%
\begin{proof}
We expand for some fixed $i$:
\begin{align}
\textstyle \E_X[g] = \sum_{x}p(x) g(x) &\textstyle = \sum_{x_{\lnot i}}p(x_{\lnot i})\sum_{x_i} p(x_i) g(x)\\
\textstyle &\textstyle =\sum_{x_{\lnot i}}p(x_{\lnot i})\Big(p(x_i) g(x) + p(-x_i) g(x_{\downarrow i}) \Big).
\end{align}
Observe that the bracket does not depend on $x_i$. We can therefore rewrite the expression as
\begin{align}
\textstyle &\textstyle \sum_{x}p(x)\Big(p(x_i) g(x) + p(-x_i) g(x_{\downarrow i}) \Big) \\
\textstyle = &\textstyle \sum_{x}p(x)\Big((1-p(-x_i)) g(x) + p(-x_i) g(x_{\downarrow i}) \Big) \\
\textstyle = & \textstyle \sum_{x}p(x) \Big[ g(x) + \big( g(x_{\downarrow i}) -g(x)) p(-x_i)  \Big].
\end{align}
This shows that
\begin{align}\label{var-red-last-estimate}
\sum_{x}p(x) \big( g(x_{\downarrow i}) -g(x)) p(-x_i) = 0.
\end{align}
So it serves as a variance reduction baseline. Moreover, if we have access to the two values $g(x_{\downarrow i})$ and $g(x)$, it is the perfect baseline, as adding it results in the complete sum in $x_i$. Taking the average over $i$, \ie choosing $\gamma = \frac{1}{n}$ gives an estimate with a decreased variance, however it is not straightforward which value of $\gamma$ gives the best variance reduction as estimates~\eqref{var-red-last-estimate} for all $i$ are not independent.
\end{proof}
Setting $\gamma = \frac{1}{n}$ is a natural choice that guarantees a reduction in variance.
This improvement can be implemented as a simple replacement of the last line of the algorithm to:
\begin{subequations}
\begin{align}
& \textstyle \pi^L_i = 1 - \detach( (x^L_i+1)/2 + x^L_i F_Z(a^L_i))\\
& \textstyle E = f(x^L; \theta) + \sum\limits_i (f(x^L;\theta){-}f(x^L_{\downarrow i};\theta))(q^L_i{-}\frac{\pi^L_i}{n} ).
\end{align}
\end{subequations}

\subsection{Straight-Through}
\PST*
\begin{proof}
Consider the last layer. The linear approximation to $f$ at $x^L$ allows to express
\begin{align}\label{STE-f}
& \textstyle f(x^L_{\downarrow i}) \approx f(x^L) + \frac{\partial f(x^L)}{\partial x^L_i} (-2 x^L_i);\\
& \textstyle df_i = f(x^L) - f(x^L_{\downarrow i}) \approx 2 x^L_i \frac{\partial f(x^L)}{\partial x^L_i}.
\end{align}
%Further, for pre-activation $a^k(x^{k-1}; \theta)$ linear in $x^{k-1}_i$ we have $\frac{\partial}{\partial \theta}a^k(x^{k-1}; \theta) = x^{k-1}_i a^k(x^{k-1}; \theta)$ because $(x^k_i)^2 = 1$. 
%Note also that for symmetric noise we have $\frac{\partial}{\partial \theta}F_Z(x^k_j a^k(x^{k-1}; \theta) ) = x^k_j\frac{\partial}{\partial \theta}F_Z(a^k(x^{k-1}; \theta) )$.
We use the expression for the derivative of layer probabilities in parameters~\eqref{d^l}
\begin{align}
D^l_i(x) & \textstyle = p_Z(a^l_i) x^l_i \frac{\partial}{\partial \theta^l} a^l_i(x^{l-1}; \theta^l)
\end{align}
and its oddness in $x^l_i$. The gradient in parameters of the last layer becomes
\begin{align}
\textstyle \sum_i D^L_i df_i = \sum_j p_Z(a^L_i) \frac{\partial}{\partial \theta^L} a^L_i(x^{L-1}; \theta^L) 2 \frac{\partial f(x^L)}{\partial x^L_i},
%\frac{\partial}{\partial \theta}F_Z(x^L_j a^L_j(x^{L-1}; \theta) )
% p_Z(a^L(x^{L-1})) \frac{\partial a^L(x^{L-1})}{\partial \theta}
\end{align}
where we %used that $p_Z$ is symmetric to remove $x^L_j$ from inside $F_Z$ and 
have canceled $x^L_i x^L_i = 1$.

%It is further seen that 
%\begin{align}
%2 p_Z({w^L_j}\T x^{L-1}) x^{L-1}_i = \frac{\partial}{\partial w^L_{i,j}} \tanh({w^L_j}\T x^{L-1}).
%\end{align}
%
%For the inner layers, the linearization of $F$ reads $ F_Z(a^k(x^k_{\downarrow})) \approx F_Z(a^k(x^k)) + p_Z(a^k(x^k))(a^k(x^k_{\downarrow}) - a^k(x^k))$.

With the linearization of $F_Z$ at $a^k$, the Jacobians $\Delta$ express linearly as follows:
\begin{align}
\textstyle \Delta^{k}_{i, j} = x^k_j (F_Z(a^k_{j}){-}F_Z(a^k_{j\downarrow i}))
%\textstyle \Delta^{k}_{i, j} & = F_Z(x^{k}_{j} a^{k} (x^{k-1})) - F_Z(x^{k}_{j} a^{k} (x^{k-1}_{\downarrow i})) \\
\textstyle  \approx x^{k}_{j} p_Z(a^k_j) \big(a^k_j - a^k_{j\downarrow i}\big),
\end{align}
where $p_Z$ is the noise density, \ie derivative of $F_Z$, and we have denoted $a^k_{j\downarrow i} = a^k_{j}(x^{k-1}_{\downarrow i}; \theta^k)$.
Because $a^k$ is linear in $x^{k-1}_i$, we have similarly to~\eqref{STE-f} that
\begin{align}
\textstyle a^k_j - a^k_{j\downarrow i} = 2 x^{k-1}_i \frac{\partial}{\partial x^{k-1}_i} a^k_j(x^{k-1}).
\end{align}
This allows to express
\begin{align}
\textstyle  \Delta^{k}_{i, j} = 2 x^{k}_{j} x^{k-1}_i p_Z(a^k_j)\frac{\partial}{\partial x^{k-1}_i} a^k_j(x^{k-1}) = 2 x^{k}_{j} x^{k-1}_i \frac{\partial}{\partial x^{k-1}_i} F_Z(a^{k} (x^{k-1})).
\end{align}
%
%\begin{align}
%& = p_Z(x^{k+1}_{j} {w^k_j}\T x^k) x^{k+1}_{j} 2 w^k_{i_j} x^k_{i}\\
%& = 2 p_Z({w^k_j}\T x^k) x^{k+1}_{j} x^k_{i} w^k_{i_j}\\
%& = x^{k+1}_{j} x^k_{i} \frac{\partial}{\partial x^k_{i}} \tanh({w^k_j}\T x^k).
%\end{align}
%Finally due to noise symmetry we also have that
%\begin{align}
%\textstyle d^l\,{=}\,\frac{\partial}{\partial \theta} F_Z(x^l_j a^{l}(x^{l-1};\theta))\,{=}\,x^l_j \frac{\partial}{\partial \theta} F_Z(a^{l}(x^{l-1};\theta)).
%\end{align}
Note the occurrence of the formal derivative of $F_Z(a^{k} (x^{k-1}))$ in $x^{k-1}_i$, that will be the only derivative that is used in the ST algorithm.

Finally observe that for the derivative in parameters of layer $l$, estimated with the product $D^l \Delta^{l+1} \cdot \dots \cdot \Delta^{L} df$ in PSA, for each $k\geq l$ the factors $x^k_{i_k}$ appear exactly twice and thus cancel.
We recover the product of Jacobians without extra multipliers by $x^k_{i_k}$, which can be implemented with automatic differentiation as proposed in~\cref{alg:ST}.
\end{proof}
%In the special case of a single hidden layer, linearizing $f$ only suffices, and
The connection to STE pointed out by~\cite{Tokui-17} can be seen as a special case of our construction for a single hidden layer, where local expectation gradients~\cite{Titsias-15} apply and it suffices to linearize $f$.

\subsection{Complexity and Efficient Implementation of PSA}\label{sec:impl}
We have made the following complexity claim.
\Pcomplexity*
For fully connected networks this complexity is indeed linear in the total number of inputs and weights. For convolutional networks it is a bit more tricky because convolutions can generate big output tensors. In this case the complexity can be stated as linear in the total number of inputs, weights and hidden units.

%\begin{proof}
We prove the claim by giving the algorithms how to implement all necessary computations with the same complexity as standard back-propagation.

%In this section we show that the computation complexity of PSA~\cref{alg:PSA} is the same as of standard back-propagation.
First, we observe that the numbers $q$ in~\cref{alg:PSA} are only used to determine the derivative and set to zero value by line~\eqref{q^k-detach}. The pre-activations $a$ in~\cref{alg:PSA} are of the same form as in the standard network, evaluating $F_Z$ component-wise does not increase complexity. Therefore backpropagation for these parts takes the same time. We only need to additionally compute the matrices $\Delta$ on the backward pass and explicitly implement the transposed multiplication (resp. transposed convolution) with them to define a custom backprop operation for the update~\eqref{def2:q^k}. %The complexity of steps for the last layer~\eqref{last-layer} is handled similarly. We now detail different tp

\paragraph{Fully Connected Layers}
Consider the case of a fully connected layer with pre-activation $a(x) = W x + b$.
Then $\Delta_{ij}$ has the same size as the matrix $W$. Recall it expresses as
%\begin{align}
%\Delta_{i,j} = F_Z(x^k_j a_{j}) - F_Z(x^k_j a_{j}(x^{k-1}_{\downarrow i})).
%\end{align}
%It can be further simplified for a symmetric noise $Z$ as
\begin{align}\label{delta-summands}
\textstyle \Delta_{i,j} = x^k_j \Big(F_Z( a_{j}) - F_Z( a_{j}(x^{k-1}_{\downarrow i}))\Big).
\end{align}

The first summand can be computed in linear time once $x^k$ and $a$ are known. The second summand is slightly more complex as it involves pre-activations for inputs with a flipped component $i$. For linear layers we have:
\begin{align}
a_{j}(x^{k-1}_{\downarrow i}) = W x^{k-1}_{\downarrow i} + b = a_j - 2 W_{j,i} x_i,
\end{align}
\ie all the numbers $a_{j}(x^{k-1}_{\downarrow i})$ can be computed in time $O(n_{k} n_{k-1})$ for a matrix $W \in \Real^{n_{k} \times n_{k-1}}$.
It follows that computing matrix $\Delta$ takes $O(n_k n_{k-1})$ time, the same as the matrix-vector multiplication $W x^{k-1}$ for forward pass or the transposed multiplication for the backward pass.

\paragraph{Convolutional Layer}
With a convolutional layer, the implementation is more tricky, because computing $\Delta$ in a matrix form is no longer efficient. Consider the convolution pre-activation
\begin{align}
a_{o,j} = \sum_{c, @i} w_{o,c,i-j} x_{c, i},
\end{align}
where $c$ and $o$ are input and output channels, $i$ and $j$ are 2d indices of spatial locations and $@i$ denotes that the range of $i$ is given by the output location $j$ and the weight kernel size: $j-h/2\leq i \leq j + h/2$. This notation makes it more easier to match with the equations in the matrix form.
%$k$ ranges over the weight kernel spatial size: $-h/2\leq l \leq h/2$. The $@$ symbol is just a hint for this is a convolution.
%in the summation is just 
%is to denote that the range of $k$ is given by the weight kernel size: $-h/2\leq i \leq h/2$, \ie corresponds to the convolution.

For the gradient in the standard network, a transposed convolution occurs:
\begin{align}
 \frac{\partial}{\partial x^{k-1}_{c,i}} = \sum_{o, @j} w_{o,c,j-i} \frac{\partial}{\partial x^k_{o,j}}.
\end{align}

For the gradient in PSA, we need to implement the following sum with $\Delta$:
\begin{align}\label{delta-t-prod}
 \frac{\partial}{\partial q^{k-1}_{c,i}} = \sum_{o, j} \Delta_{o,c,j,i} \frac{\partial}{\partial q^k_{o,j}},
\end{align}
where in the case of convolution and symmetric noise, $\Delta$ is given by
\begin{align}\label{delta-conv-summmands}
\Delta_{o,c,j,i} = x^k_{o,j} \Big( F_Z( a_{o,j}) - F_Z( a_{o,j}(x^{k-1}_{c, \downarrow i})) \Big).
\end{align}
Using that $a_{o,j}(x^{k-1}_{c, \downarrow i})$ itself is given by the convolution, we have 
\begin{align}
a_{o,j}(x^{k-1}_{c, \downarrow i}) = 
\begin{cases}
a_{o,j} - 2 w_{o,c,i-j} x^{k-1}_{c,i}, & \IF -h/2\leq i \leq h/2;\\
a_{o,j}, & \OTHERWISE.
\end{cases}
\end{align}
Therefore $\Delta_{o,c,j,i}$ has the same support in indices $i,j$ as the convolution, but it is different in that it cannot be represented as just a function of $i-j$. It can be interpreted as a convolution with a kernel, which is spatially varied, \ie at all locations a different kernel is applied. 

Due to the structure of $\Delta$, the sum~\eqref{delta-t-prod} takes the same range of indices as the convolution, we may write:
\begin{align}\label{delta-t-conv}
 \frac{\partial}{\partial q^{k-1}_{c,i}} = \sum_{o, @j} \Delta_{o,c,j,i} \frac{\partial}{\partial q^k_{o,j}}.
\end{align}

The elements of the kernel $\Delta_{o,c,j,i}$ are computed in $O(1)$ time, therefore the full backprop operation~\eqref{delta-t-prod} has the same complexity as backprop with standard network (assuming small kernel size, where convolution uses straightforward implementation and not FFT).

%It remains to show that the elements of the spatially-varying kernel are e

We take one step further, to show that convolution with $\Delta$ for logistic noise needs literally the same amount of operations.
Computing the convolution~\eqref{delta-t-prod} with the first summand of $\Delta$ is easy. As it does not involve the index $i$, it reduces to multiplication in the output space (indices $o,j$), summation over $o$ and a convolution with just the support indicator in $j$.
%it has the complexity of the size of the output tensor $x^k$. We now focus on the second summand in~\eqref{delta-conv-summmands}.

It remains to compute the convolution~\eqref{delta-t-conv} with the second summand of $\Delta$, \ie:
\begin{align}\label{delta-t-conv}
\sum_{o, @j} x^k_{o,j} F_Z(a_{o,j} - 2 w_{o,c,i-j} x^{k-1}_{c,i}) g_{o,j},
\end{align}
where we denoted the gradient \wrt the output as $g$. The factor $x^k_{o,j}$ is easily accounted for, by introducing $\tilde g_{o,j} = g_{o,j} x^k_{o,j}$. We can further expand for logistic noise:
\begin{align}\label{delta-t-conv}
\sum_{o, @j} \frac{1}{1 + e^{a_{o,j}} e^{-2 w_{o,c,i-j} x^{k-1}_{c,i}}} \tilde g_{o,j}.
\end{align}
Since $x^{k-1}_{c,i}$ takes only two possible values, for each input gradient coordinate $c,i$ we need the convolution with $e^{\pm 2 w_{o,c,i-j}}$. 
We can precompute $A_{o,j} = \exp(a_{o,j})$, $W^{\pm}_{o,c,l} = \exp(\pm 2 w_{o,c,i-j})$, \ie the expensive $\exp$ operations need to be performed only for the output and the kernel alone, and not inside the convolution. For the dominant complexity part involving $c,o,i,@j$ indices, we only need to compute and aggregate
\begin{align}\label{ratio_conv}
\sum_{o, @j} \frac{\tilde g_{o,j}}{1 + A_{o,j} W^{\pm}_{o,c,i-j}}.
\end{align}
Compared to standard convolution, this costs only one extra addition and division operation. We call the operation~\eqref{ratio_conv} a {\em ratio convolution} and implemented it in CUDA. Since our implementation is not fully optimized and we need to load twice as much data ($g$ and $A$) for the input and $W^{\pm}$ for the ``kernel'', the actual run-time is 3-4 times slower than that of cuDNN standard convolution.
%
%
%
%The $\Delta$ has two summands as in~\eqref{delta-summands}. We use that $x^_j$ is binary and for 

\paragraph{Head Function}
For the head function $f(x^L)$ that is a composition of a linear layer $W x^L +b$ and some fixed function $h$, in order to compute $f(x^L) - f(x^L_{\downarrow i})$ we need again to form all pre-activations $a_j(x^L_{\downarrow i})$ that takes $O(n_L K)$ time, where $K$ is the number of classes (more generally, the dimensionality of the network output). This is of the same size as the matrix $W$. Assuming that the final loss function $h \colon \Real^K \to \Real$ (\eg, cross-entropy) takes time $O(K)$, the computation in the last layer has complexity $O(n_L K)$ as we need to call this function for all input flips. This is however still of the same complexity as size of the matrix $W$. %, but with more expensive operations.
%\end{proof}

%\input{tex/estimators.tex}
%\input{tex/concrete.tex}
%\input{tex/ste.tex}
\section{Details of Experiments and Additional	 Comparisons}\label{sec:moreexp}

In this section we describe details of the experimental setup and measuring techniques and offer some auxiliary experiments.

\subsection{Gradient Estimation Accuracy}
We use the simple problem instance shown in \cref{fig:example}(a) with 2 classes and 100 training points per class in 2D and a classification network with 5-5-5 Bernoulli units. The data was generated as follows. Points of class 1 (resp. 2) are uniformly distributed above $y=0$ (resp. below $y= \cos(x)$) for $x \in [-\pi/2, \pi/2]$. The implementation is available in {\tt gradeval/expclass.py}.
We have experimented with several configurations varying the number of units and layers. Generally, with a smaller number of units {\sc arm} gets more accurate and {\sc st} gets less accurate, but the overall picture stays. We therefore demonstrate the comparison on the 5-5-5 configuration.

The training progress is shown in~\cref{fig:toy-train}. In this problem unbiased estimators perform well and an extra variance reduction of PSA is not essential. 
To measure RMSE and cosine similarity errors in~\cref{fig:rmse}, we collect $T = 10^4$ total samples for each estimator. For each value of the number of samples $M$ shown on the x-axis, we calculate the mean and variance of an $M$-sample estimator by using the $T/M$ sample groups to estimate these statistics. The same $T$ samples are used to estimate the values for varying $M$. Towards $M\rightarrow T$ the estimates become more noisy, so the rightmost parts of the plots shouldn't be considered reliable.

Comparisons with additional unbiased methods is proposed in~\cref{fig:rmse1} and \cref{fig:accuracy1}. We compare to the following techniques: 1. \REINFORCE with the constant input-dependent baseline set to the true expected function value of the loss objective per data sample. This choice represents the best of what one can expect to get with input-dependent baselines constant or trained with a neural network as NVIL~\cite{mnih14}. 2. MuProp~\cite{MuProp}, which uses a linear baseline around deterministically propagated points. The constant part of the baseline may also optionally set to the input-dependent true value. It is seen from the results that the linear baseline of MuProp does not improve accuracy in deep layers, possibly connected to the fact that the loss depends only on the state of the last layer and thus its linear approximation using states of the first layers is not helpful. It is also seen in~\cref{fig:accuracy1} that the linear baseline of MuProp degrades at 2000 epochs. For this experiment we used only $T=2000$ samples.

%\cref{fig:accuracy-suppl} shoes more detail to the cosine similarity metric~\cref{fig:accuracy}, evaluated separately \wrt gradient components in different layer. Consistently with RMSE metric at initialization~\cref{fig:rmse}, all estimators are more accurate in later layers (towards the head) also in cosine similarity and during the learning.
%
\subsection{Deep Learning}
\begin{figure*}[t]
\small
\setlength{\figwidth}{0.33\textwidth}
\centering
\begin{tabular}{cccc}
&\small No Augmentation & \small Affine Augmentation & \small Flip \& Crop Augmentation \\
\parbox{3mm}{\rotatebox[origin=c]{90}{\small Validation Loss}}&
\begin{tabular}{c}
\includegraphics[width=0.33\linewidth]{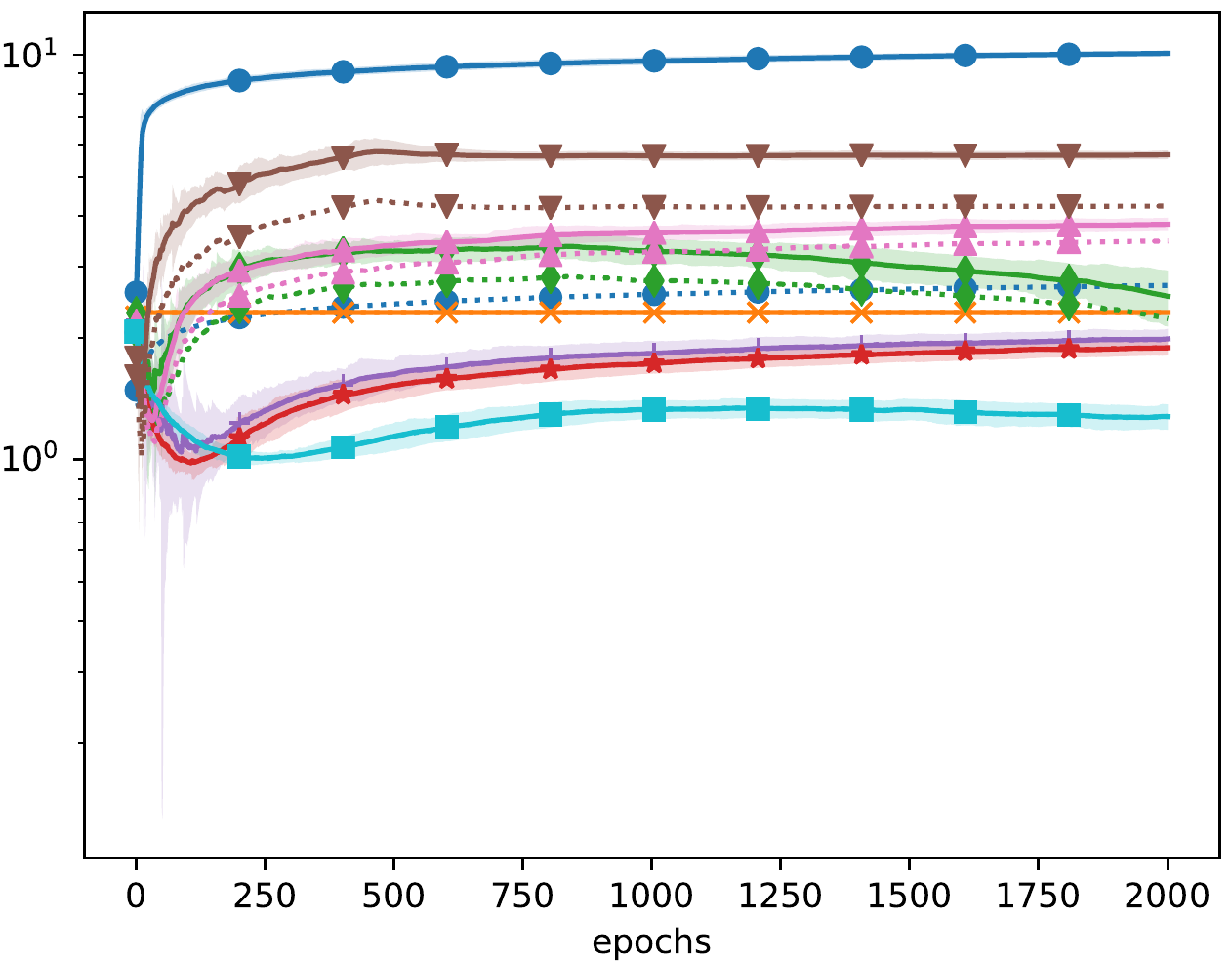}%
\end{tabular}&\begin{tabular}{c}
\includegraphics[width=0.33\linewidth]{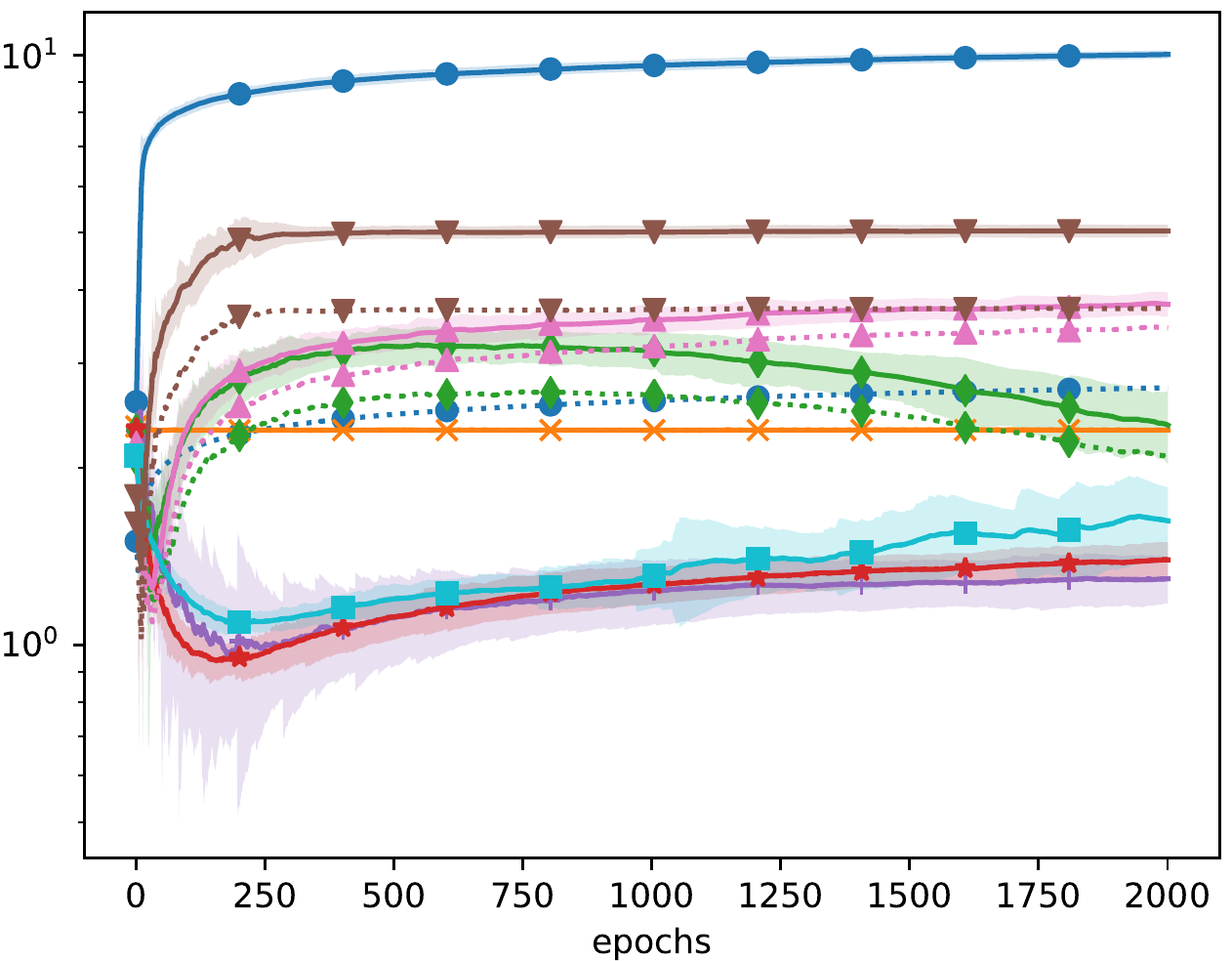}%
\end{tabular}&\begin{tabular}{c}
\includegraphics[width=0.33\linewidth]{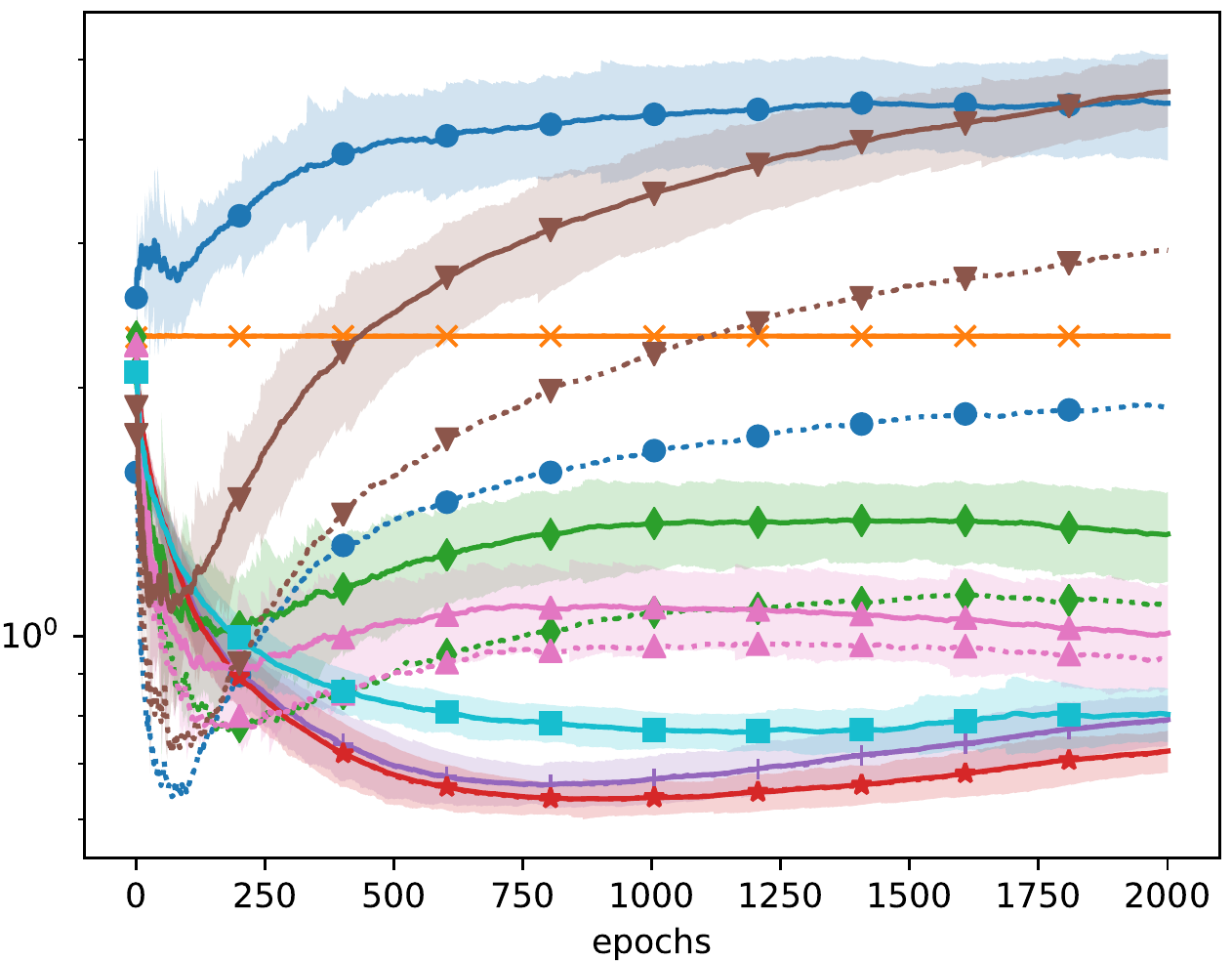}%
\end{tabular}\\
\parbox{3mm}{\rotatebox[origin=c]{90}{\small Training Accuracy}}&
\begin{tabular}{c}
\includegraphics[width=0.33\linewidth]{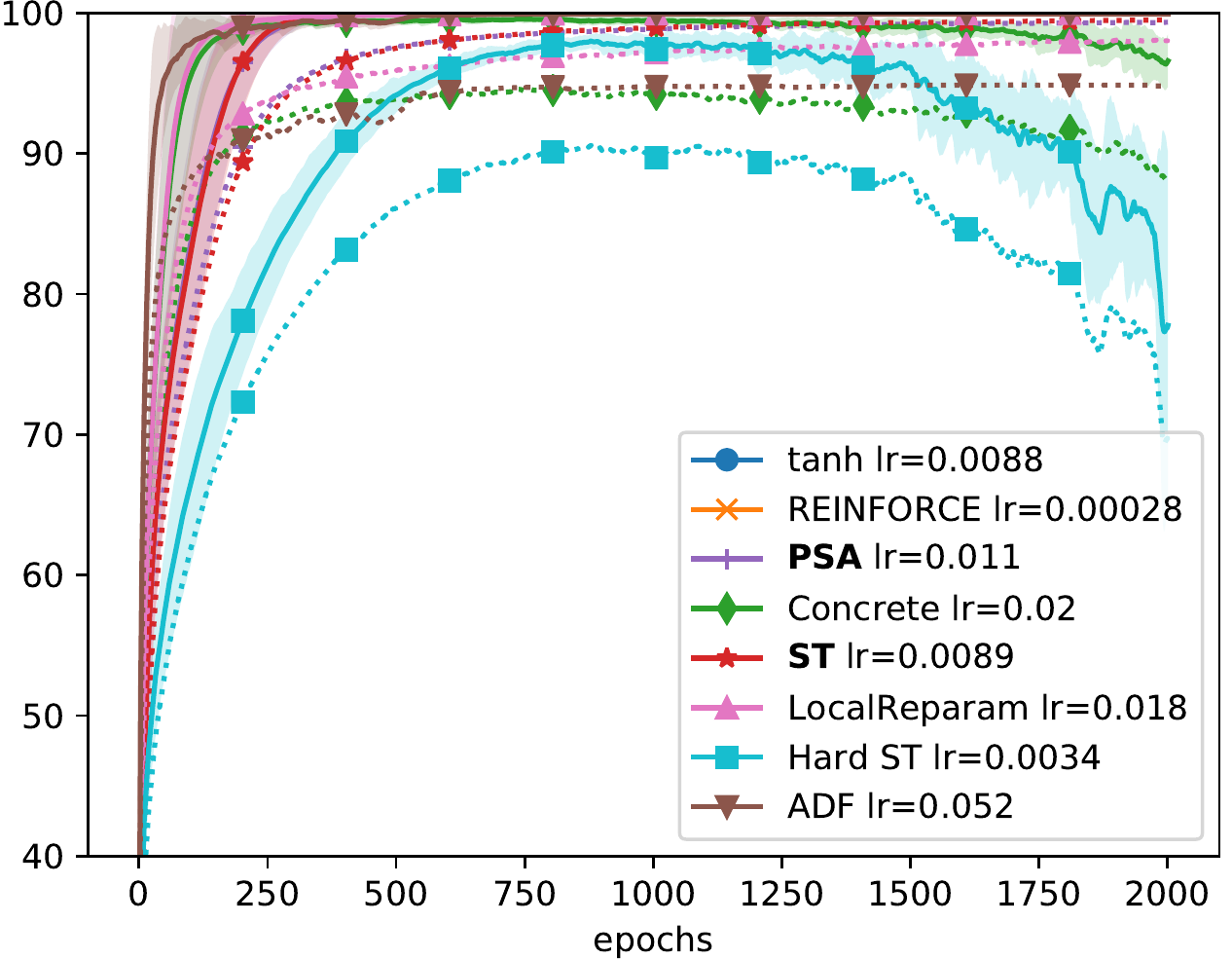}%
\end{tabular}&\begin{tabular}{c}
\includegraphics[width=0.33\linewidth]{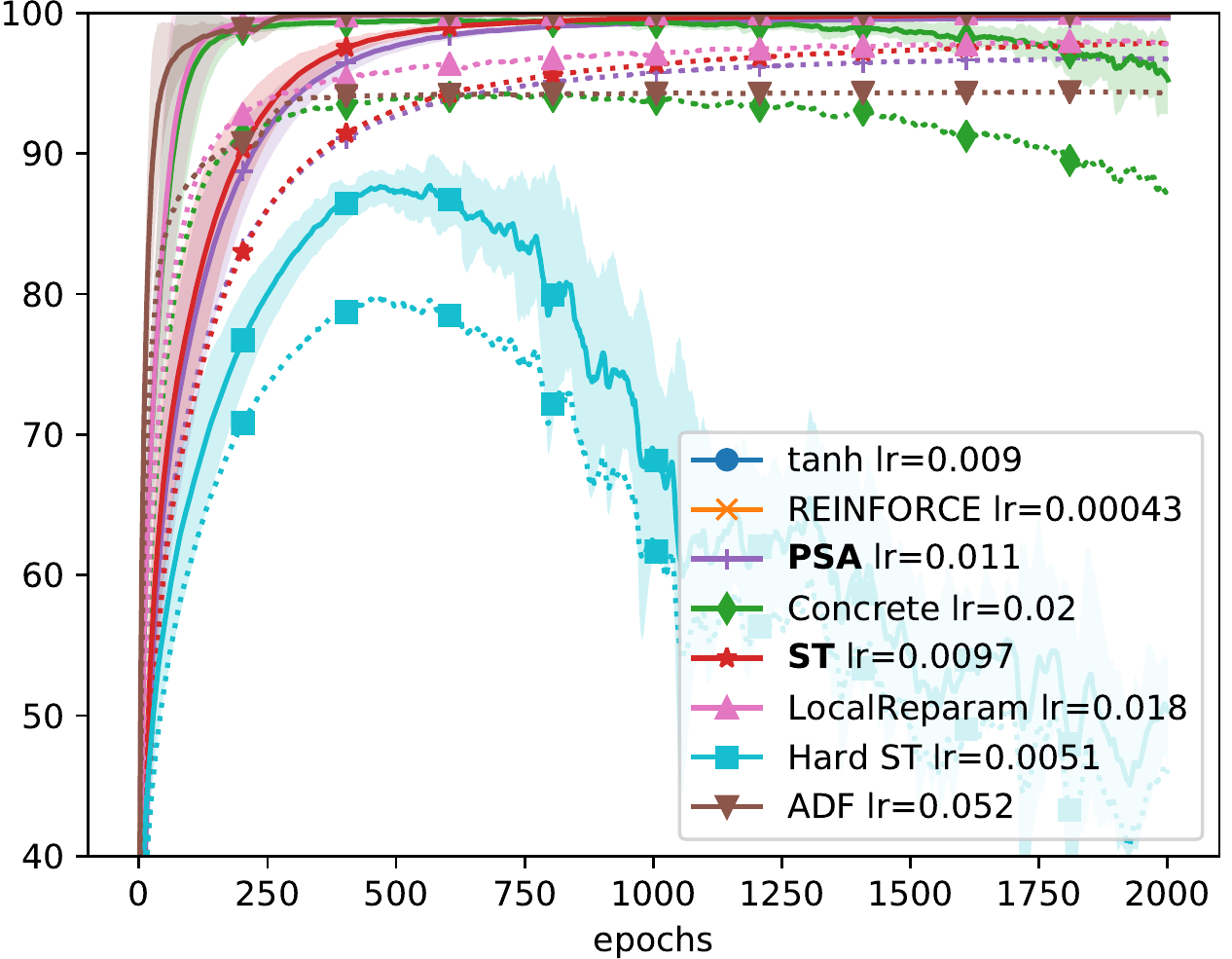}%
\end{tabular}&\begin{tabular}{c}
\includegraphics[width=0.33\linewidth]{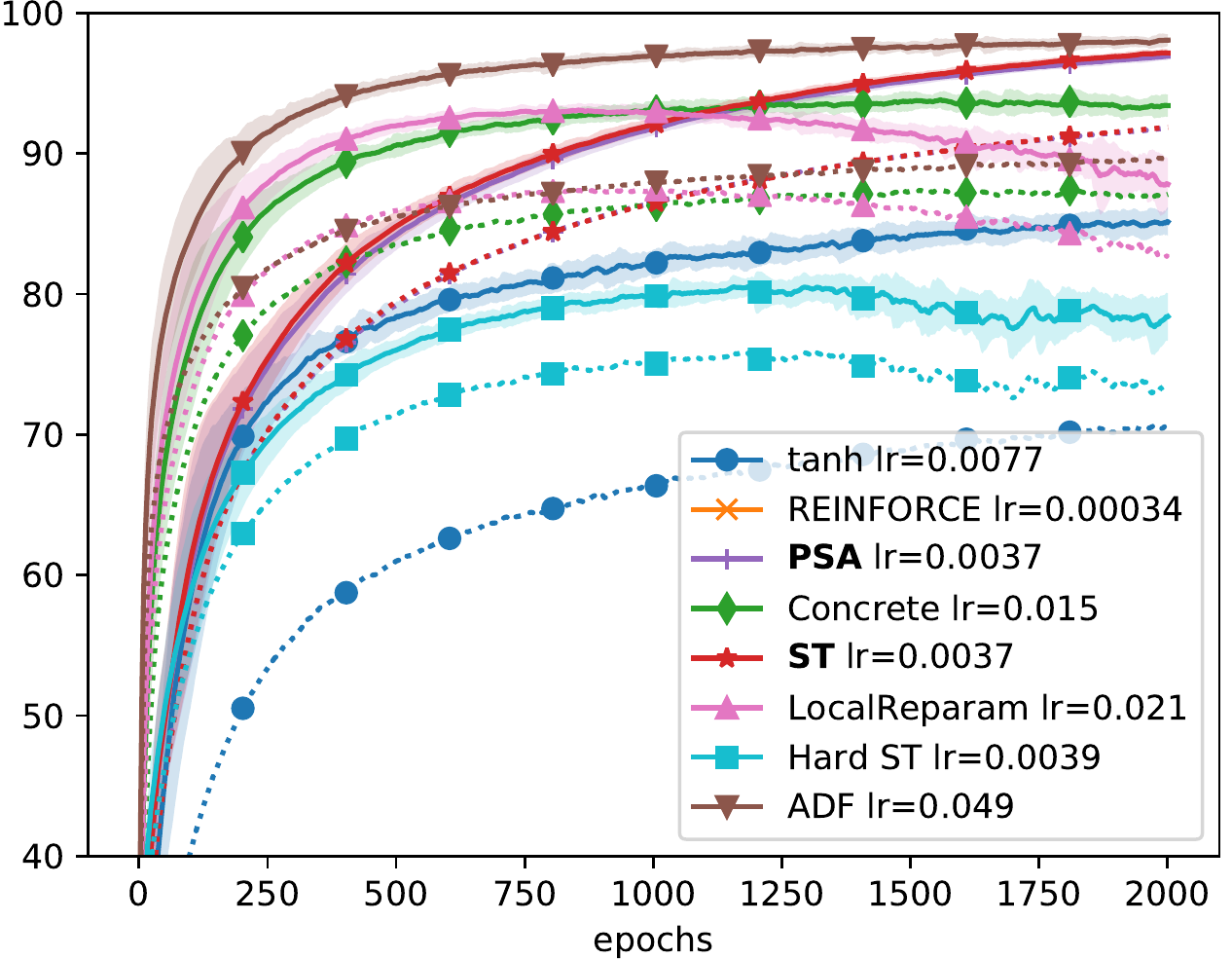}%
\end{tabular}
\end{tabular}
\caption{\label{fig:CIFAR-extra}
Additional plots to~\cref{fig:CIFAR-NLL}: validation losses and training accuracies.
}
\end{figure*}
\paragraph{Dataset}
The learning experiments are performed on CIFAR10 dataset\footnote{\url{https://www.cs.toronto.edu/~kriz/cifar.html}}. The dataset contains a training set and a test set. Following the common approach, we withhold 5000 samples (10 percent) of the training set as a validation set. Since we do not perform any hyper-parameter tuning or model selection based on the validation set, it provides an independent and unbiased estimates.

\paragraph{Augmentation}
For Affine augmentation experiment we used random affine transforms that included shifts by $\pm5\%$ and rotation by $\pm10\deg$, linearly interpolated.
Flip\&Crop is the commonly applied augmentation for this dataset. It consists of random horizontal flips and random shifts with zero padding by $\pm4$ pixels ({\tt \small transforms.RandomCrop(32, 4)}). Fro the training plots and loss values achieved we can hypothesize that this augmentation is more "diverse" and harder to fit that the Affine augmentation above.

\paragraph{Network}
For simplicity, we used a variant of All convolutional network (Springenberg et al. 2014), with strides instead of max pooling.
The network structure is as follows:
\begin{lstlisting}
ksize = [3,  3,  3,  3,   3,   3,   3,   1  ]
stride= [1,  1,  2,  1,   1,   2,   1,   1  ]
depth = [96, 96, 96, 192, 192, 192, 192, 192]
\end{lstlisting}
There is no padding and the output is a $192 \times 2 \times 2$ binary tensor, which is then flattened and passed to the head function consisting of an affine transform to the $10$-dimensional class logits. This network is smaller than VGG-type networks commonly used~\cite{Hubara-16}, however significantly smaller is size esp. considering the last fully connected layer. Our FC has size $768{\times}10$, the ones in~\cite{Hubara-16} (there are three) are: $8192{\times}1024$, $1024{\times}1024$, $1024{\times}10$.
No batch normalization is used for the purity of comparison experiments. If we were chasing the highest accuracy, we can confirm that BN improves the results.

\paragraph{Optimizer}
For the optimization we used batch size 64 and SGD optimizer with Nesterov Momentum 0.9 (pytorch default) and a constant learning rate. Because different methods have different variance and biases, for a fair comparison we tried to find the optimal learning rate for each method individually.
We selected the learning rate by a numerical optimization based on the performance of the model in 5 epochs as measured by the objective optimized by the method (\ie the sample-based estimate of the expected loss or the approximated expected loss) on the training set. We used exponentially weighted average on the objective value to reduce its variance. We used
\begin{lstlisting}
scipy.optimize.minimize_scalar(f, method='bounded', bounds=(-6, 0),maxiter=10)
\end{lstlisting}
for the numerical search of the optimal $\log$ of the learning rate. Arguably, this learning rate selection optimizing the short-horizon performance may be sub-optimal in a longer run, but is the first best approximation to deal with this issue.

Parameters of linear and convolutional layers were initialized as uniformly distributed. We then perform one iteration, computing mean and variance statistics over a batch and spatial dimensions and whitening pre-activations using these statistics, similar to batch normalization. This is performed only as a data-dependent initialization to make sure that activations are in a reasonable range on average and the gradients are initially non-vanishing.

\paragraph{Test Metrics}
In our experiments all hyperparameters including the learning rates are tuned exclusively on the training set as detailed above. Hence the validation set provides an unbiased estimate of the test error and we report only it.

\paragraph{Methods}
%We tested our Path Sample-Analytic ({\tt PSA})~\cref{alg:PSA} and Straight-Through ({\tt ST})~\cref{alg:ST} against several baselines in~\cref{tab:baselines}. 
Here we specify additional details on the baseline methods used in the learning experiment.
The ADF method, called AP2 in \cite{shekhovtsov-18-norm}, propagates means and variances through the hidden layers fully analytically. This method is the equivalent of the {\sc PBNet} method in~\cite{peters2019probabilistic} when the weights are deterministic. We only sample the states of the last binary layer are sampled, as a general solution suitable with different head functions. The ADF family of methods includes expectation propagation~\cite{Minka-2001} designed for approximate variational inference in graphical models. It computes an approximation to summation~\eqref{mc-model} by fitting and propagating a fully factorized approximation to marginal distributions $p(x^k|x^0)$ with forward KL divergence. The gradient of this approximation is then evaluated. The PSA method differs in that it approximates the gradient directly and does not make a strong factorization assumption. From the experiments we observe that ADF performs very well in the beginning, when all weights are initially random and then it over-fits to the relaxed objective.

In the {\tt LocalReparam} method we sample pre-activations from their approximated Normal distribution and computed probabilities of the outputs analytically. This is related but different from the {\sc PBNet-S} method~\cite{peters2019probabilistic}, which samples activations from the concrete relaxation distribution and is not directly applicable to real-valued deterministic weights. The implementation of our methods and these baselines is available.

\paragraph{Infrastructure}
The experiments were run on linux servers with NVIDIA GTX1080 cards and Tesla P100 cards.

\paragraph{Running Time}
We measure running time for a single batch of size 64 and the running time for one epoch of training / validation. The training loop includes measuring statistics using 10 samples of the model per data point, which increases amount of forward computations performed.
Please consider that these numbers are indicative only, as we did not optimize for performance beyond providing the CUDA kernel for \PSA.
We also have lots of overhead from implementing padding and strides in pytorch, externally to the kernel.
Times are given in seconds.

\setlength{\tabcolsep}{5pt}
\begin{tabular}{lllp{0.15\linewidth}p{0.15\linewidth}}
Method & Forward batch &  Backward batch & Train Epoch with measurements & Val Epoch with measurements\\
\hline
{\tt PSA} &  0.013 & 0.056 & 105 & 14 \\
{\tt Gumbel} & 0.009 & 0.012 & 77 & 13 \\
{\tt Tanh} & 0.005 & 0.011 & 77 & 12.5 \\
{\tt REINFORCE} & 0.014 & 0.009 & 77 & 14 \\
{\tt ADF} & 0.012 & 0.031 & 91.5 & 14 \\
{\tt LocalReparam} & 0.015 & 0.029 & 98 & 14 \\
{\tt ST} & 0.009 & 0.011 & 77 & 13 \\
{\tt HardST} & 0.009 & 0.011 & 77 & 13 \\
\hline
\end{tabular}

\subsection{MuProp and REINFORCE with Baselines}\label{sec:muprop}
For fairness of comparison we have additionally tried to use \REINFORCE with a centering variance reduction and MuProp. As a baselines for these methods we use the exponentially weighted average (EWA) of the loss function per data point with momentum=0.9. Since the learning rate needed for these methods is rather small, the EWA is close to the true expected loss value per data point. Since the EWA is kept per data point, this is an accurate input-dependent constant baseline. We made several learning trials with automatically selected learning rates (denoted with stars in~\cref{fig:muprop}) as well as manually set learning rates. Since automatic learning rates appear to be too high as they lead do divergence, it only makes sense to decrease these learning rates. We tried setting the learning rates smaller. The results in~\cref{fig:muprop} show that with some careful choice of parameters, \REINFORCE and MuProp can make progress. However, the learning rates required are 1-2 orders smaller than with biased  methods, which leads to much slower performance. Indeed, comparing~\cref{fig:muprop} and~\cref{fig:CIFAR-NLL} (note the scale difference of $y$-axis), we observe that variance-reduced \REINFORCE and MuProp in 1000 epochs have barely made the progress of biased estimators in 10 epochs.

%The automatically found learning rate was $0.0035$ and no visible improvement in the results with the training loss stagnating around $2.31$.
%
%
%\paragraph{Additional Plots}

\begin{figure*}[t]
\small
\centering
\includegraphics[width=0.4\linewidth]{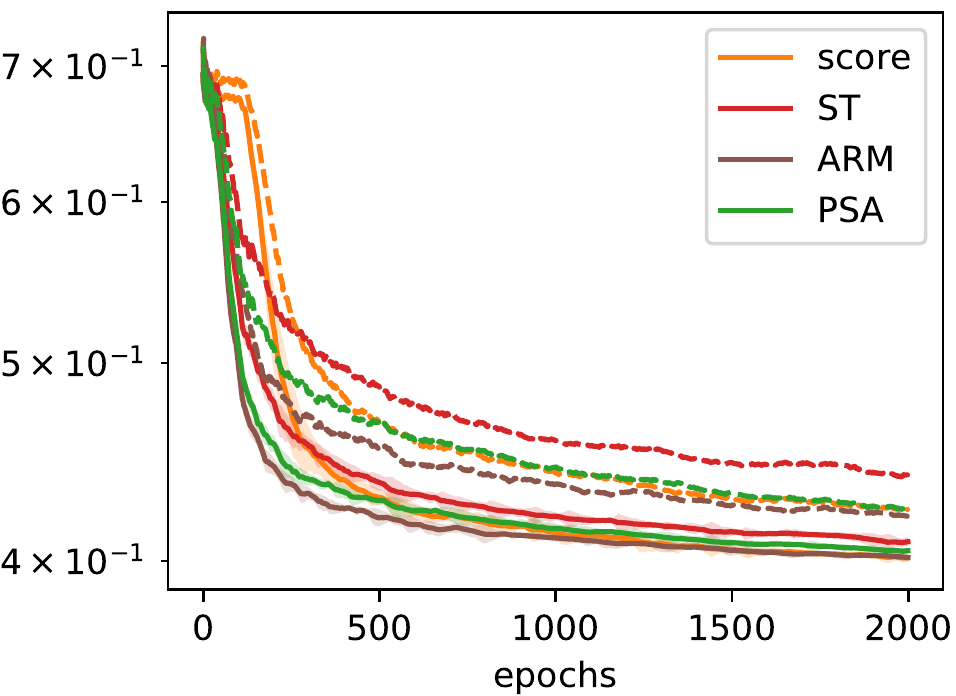}%
\caption{\label{fig:toy-train}
Training losses when training a small network 5-5-5 on the data in~\cref{fig:example}(a) with several methods. Here the dashed curves show the SBN objective, the expected loss of randomized predictor~\eqref{exp-loss}, which upper bounds the the loss of the ensemble~\eqref{exp-loss} shown with solid curves. We see that \REINFORCE (score) does not do very well in the beginning, but eventually gets into a mode that allows it to optimize the objective. The methods \ARM and \PSA  perform similar and \ST is somewhat slower due to its bias for this small model.
}
\end{figure*}

\begin{figure*}[t]
\centering
\small
\begin{tabular}{ccc}
\small Layer 1 &  \small Layer 2 & \small Layer 3 \\  
\includegraphics[width=0.33\linewidth]{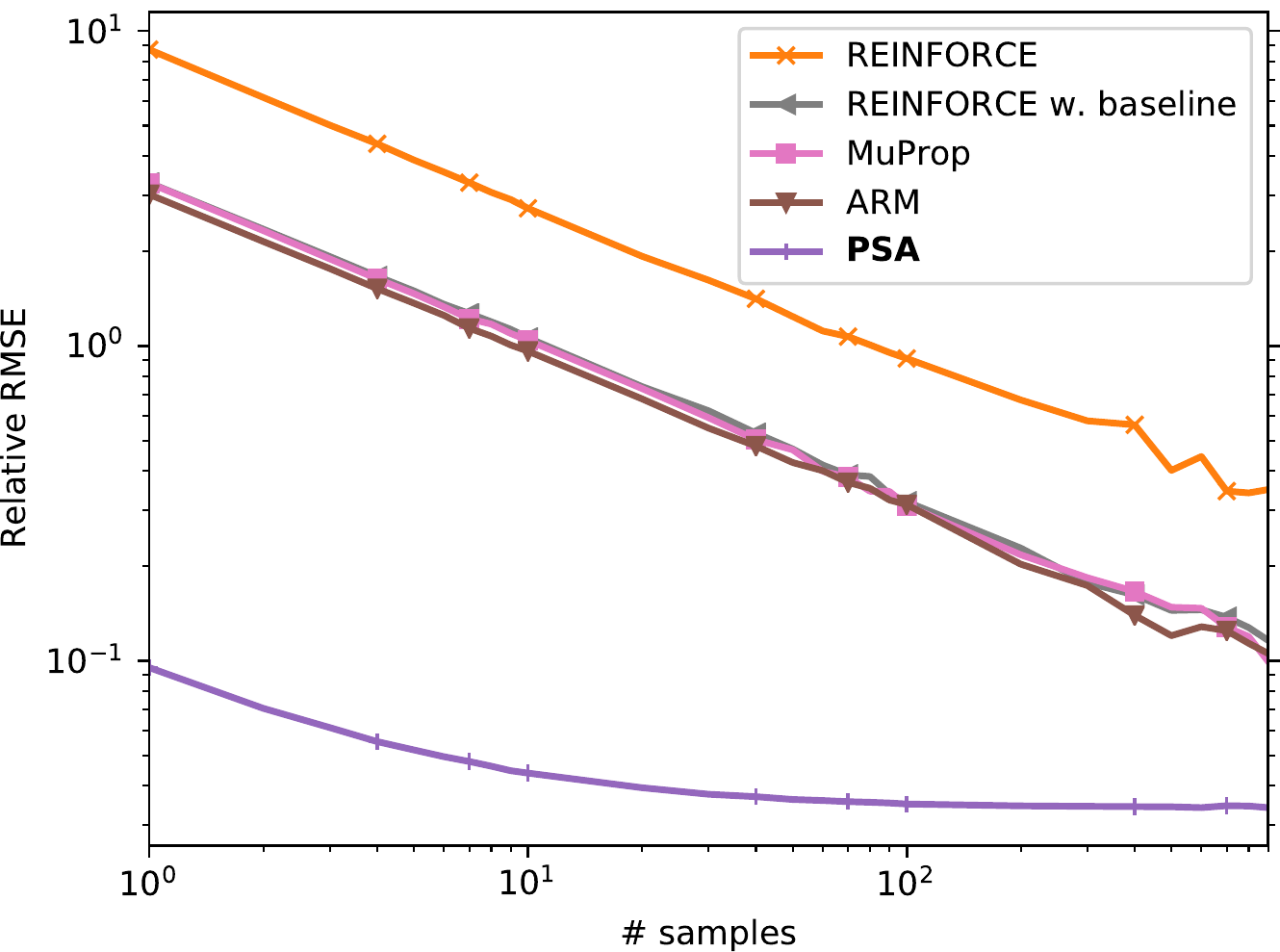}&
\includegraphics[width=0.33\linewidth]{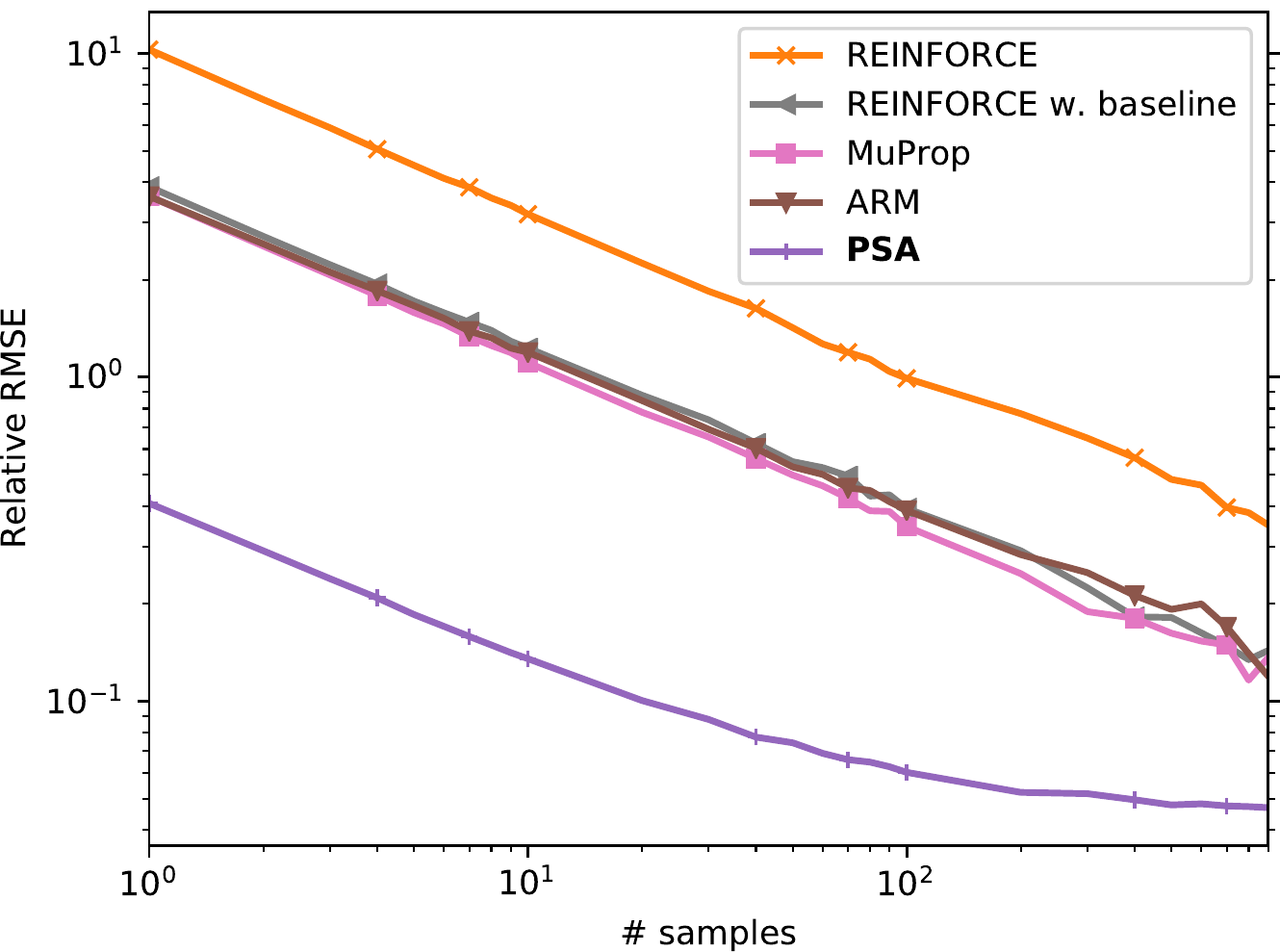}&
\includegraphics[width=0.33\linewidth]{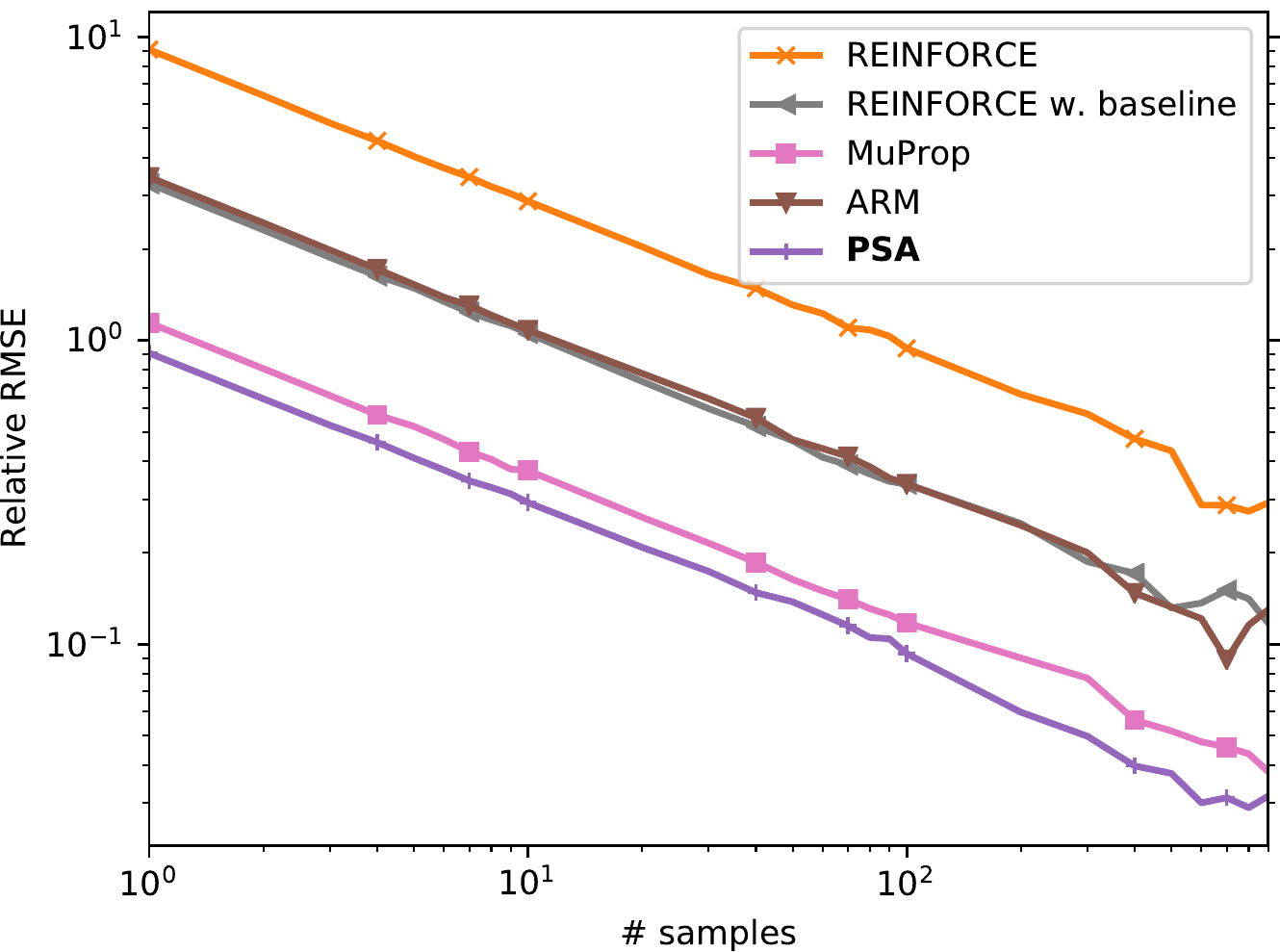}
\end{tabular}
\vskip-0.5\baselineskip
\caption{\label{fig:rmse1}
Root mean squared error of the gradient in layers 1 to 3 relative to the true gradient length after epoch 1.
Comparisons with more unbiased methods in the same setting as~\cref{fig:rmse}.
}
\end{figure*}
\begin{figure*}[t]
\setlength{\figwidth}{0.33\textwidth}
\centering
\begin{tabular}{ccc}
\small After 1 epoch &  \small After 100 epochs & \small After 2000 epochs \\  
\includegraphics[width=0.33\linewidth]{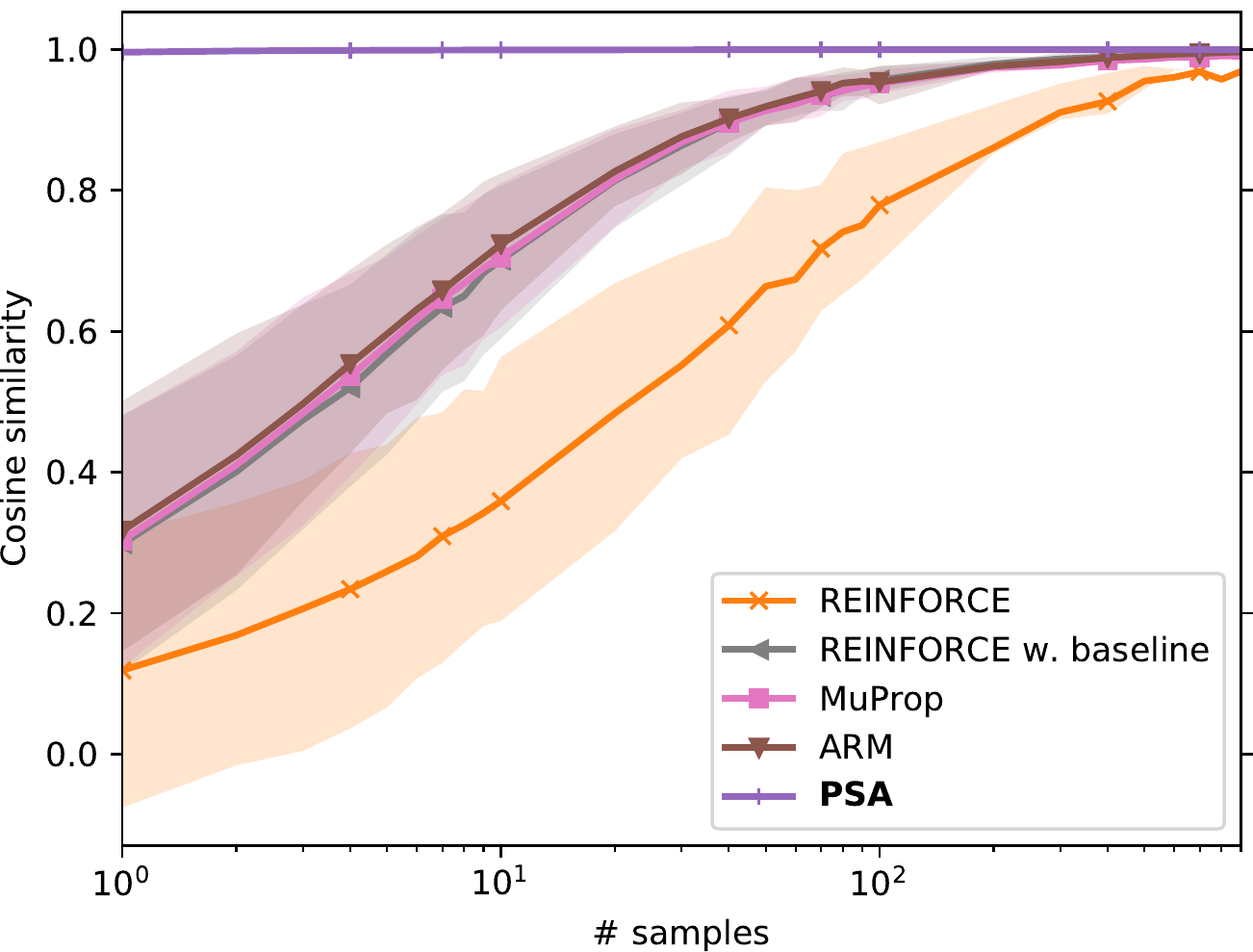}&
\includegraphics[width=0.33\linewidth]{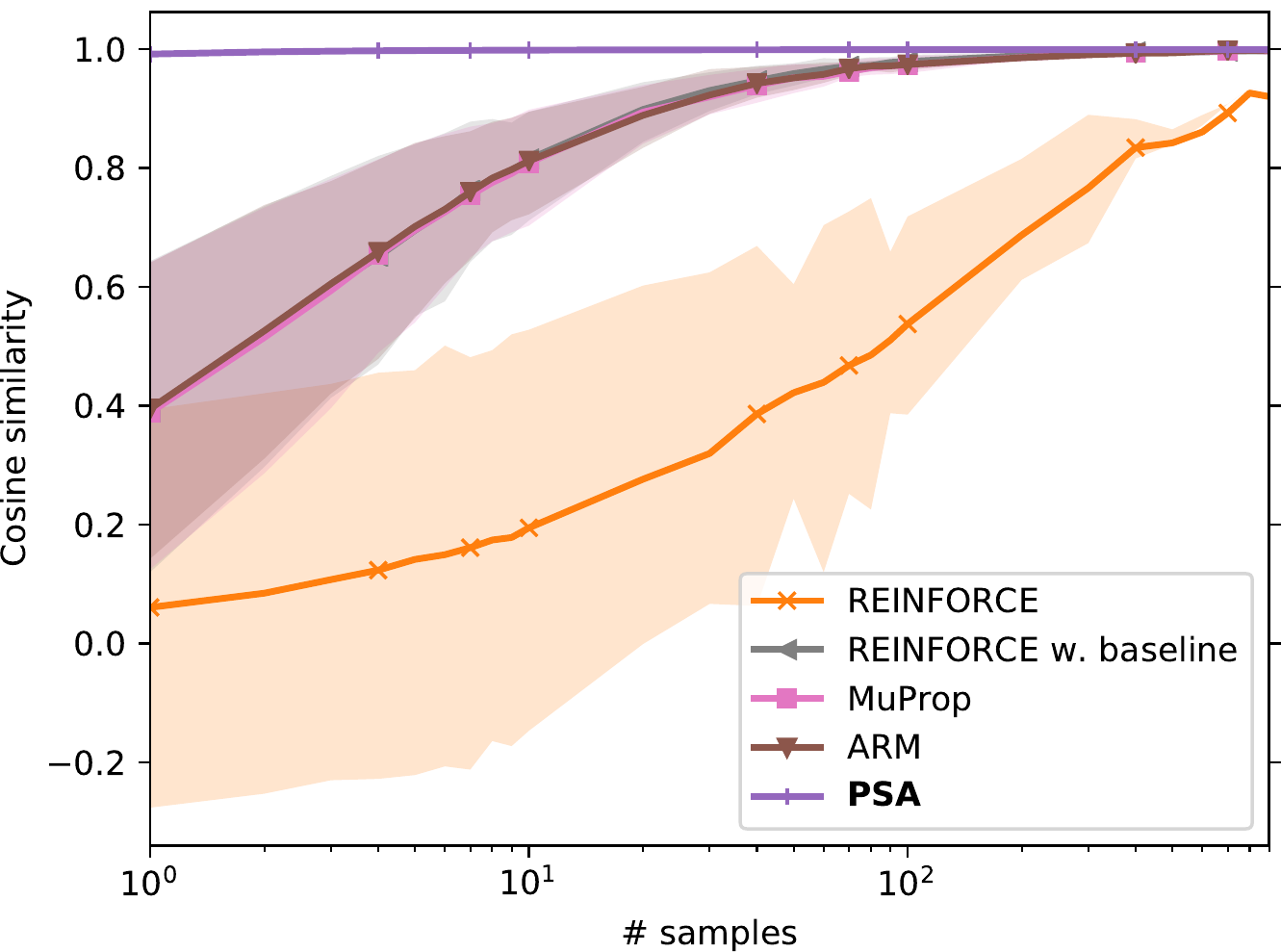}&
\includegraphics[width=0.33\linewidth]{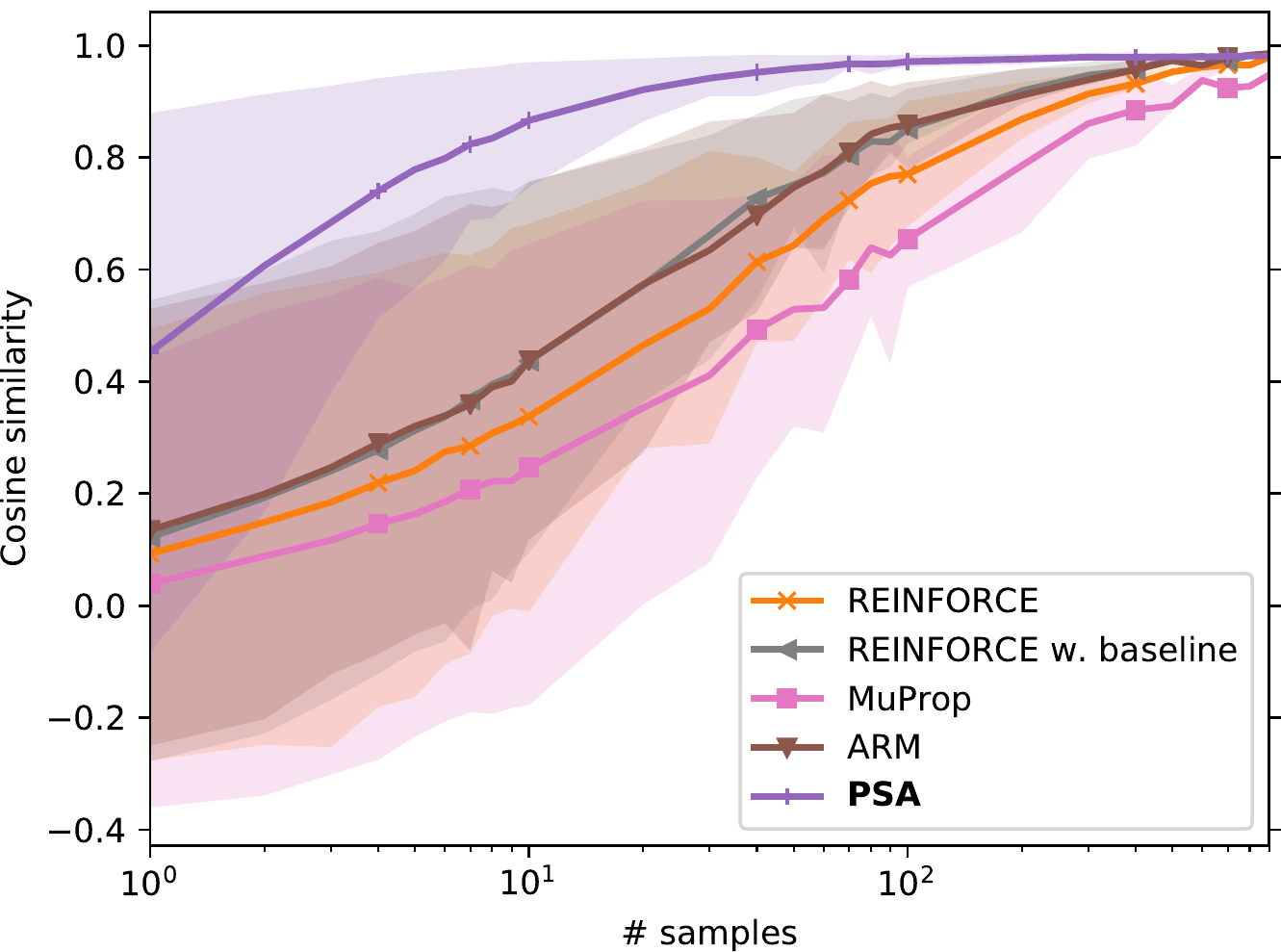}\\
\end{tabular}
\caption{\label{fig:accuracy1}
Cosine similarity of the estimated gradient to the true gradient in layer 1 at different points during training. 
Comparisons with more unbiased methods in the same setting as~\cref{fig:accuracy}.
}
\end{figure*}

\begin{figure*}[t]
\centering
\small
\begin{tabular}{ccc}
\small 3-3-3 &  \small 5-5-5 & \small 10-10-10 \\  
\includegraphics[width=0.33\linewidth]{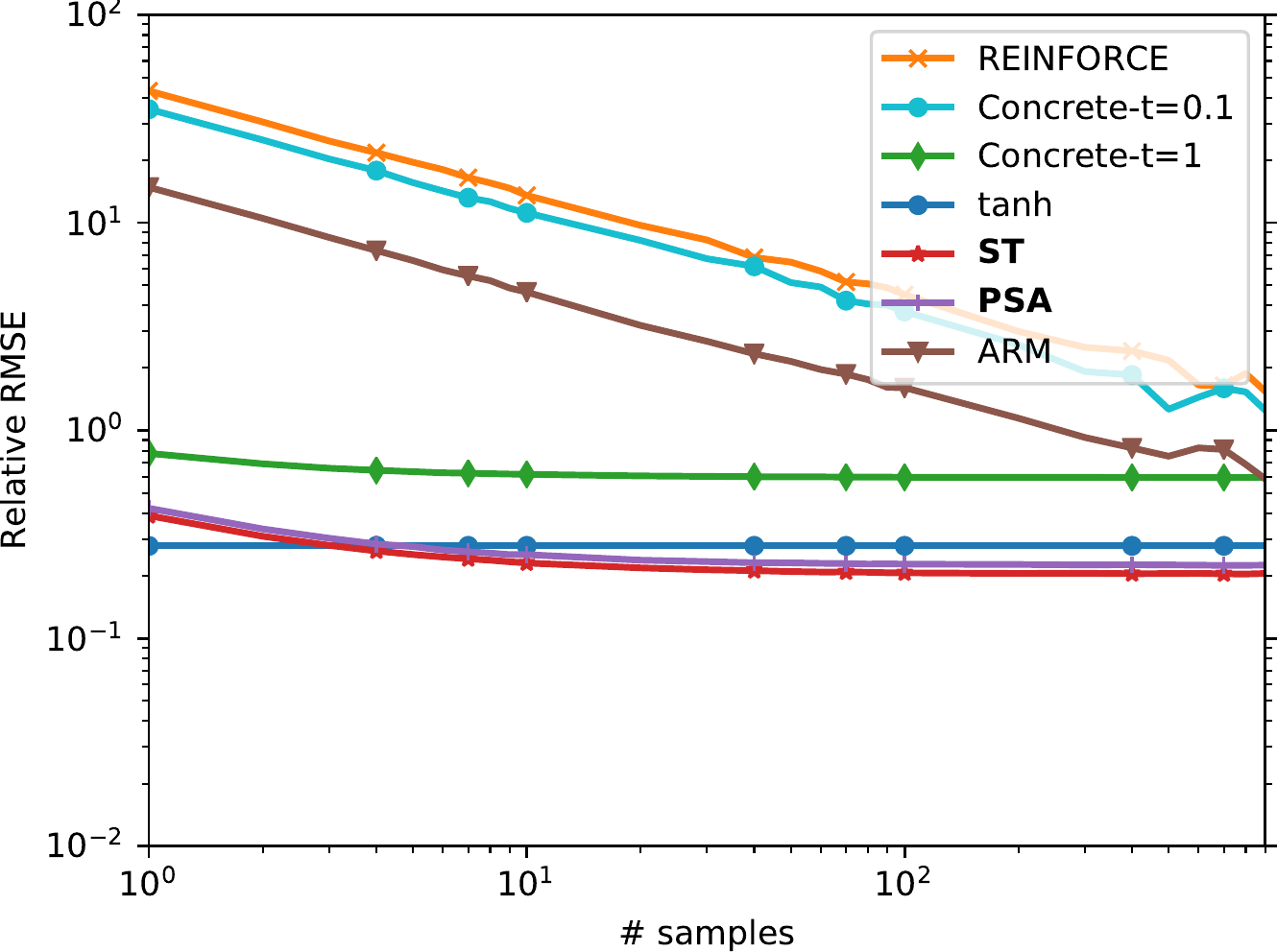}&
\includegraphics[width=0.33\linewidth]{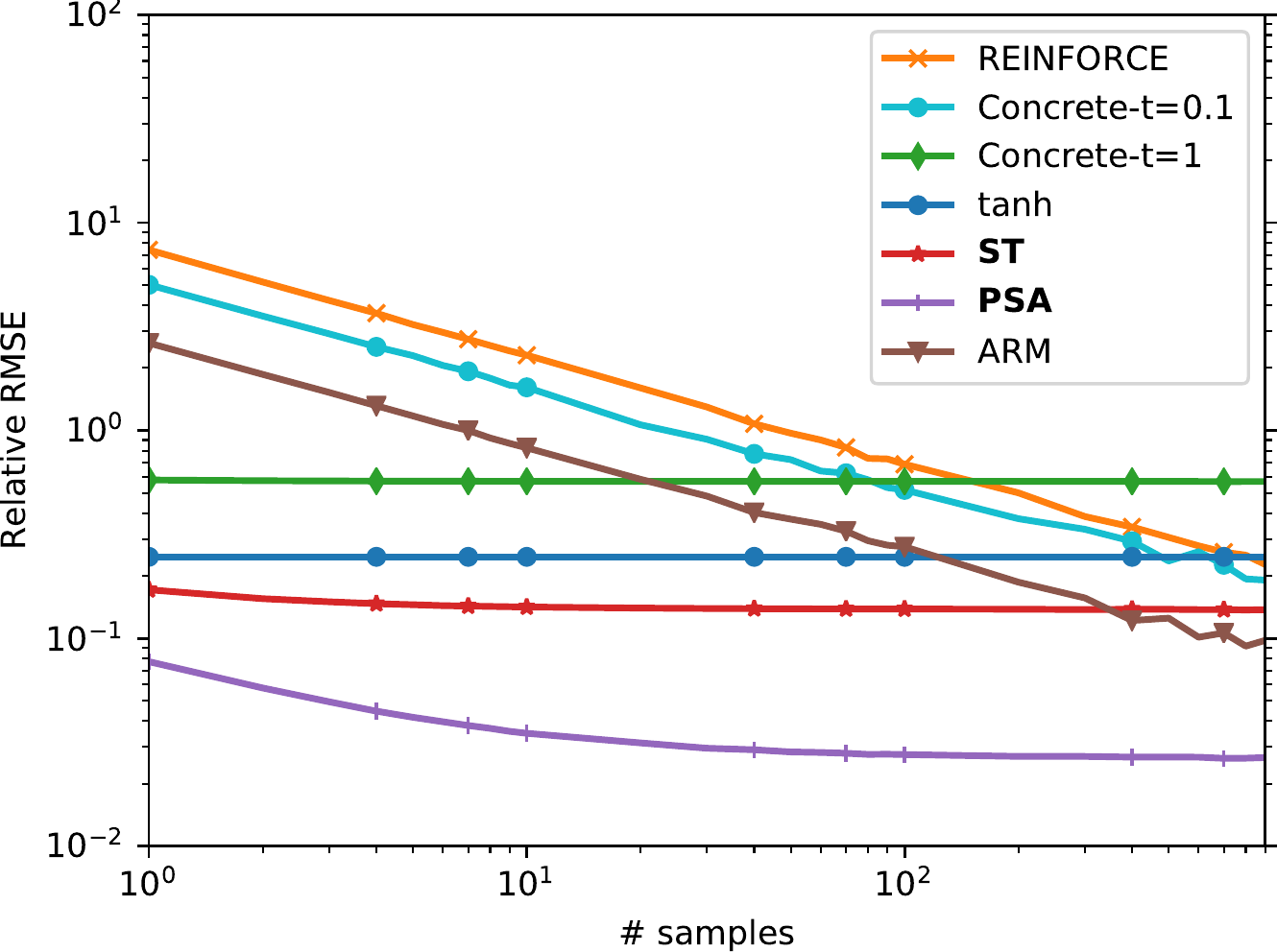}&
\includegraphics[width=0.33\linewidth]{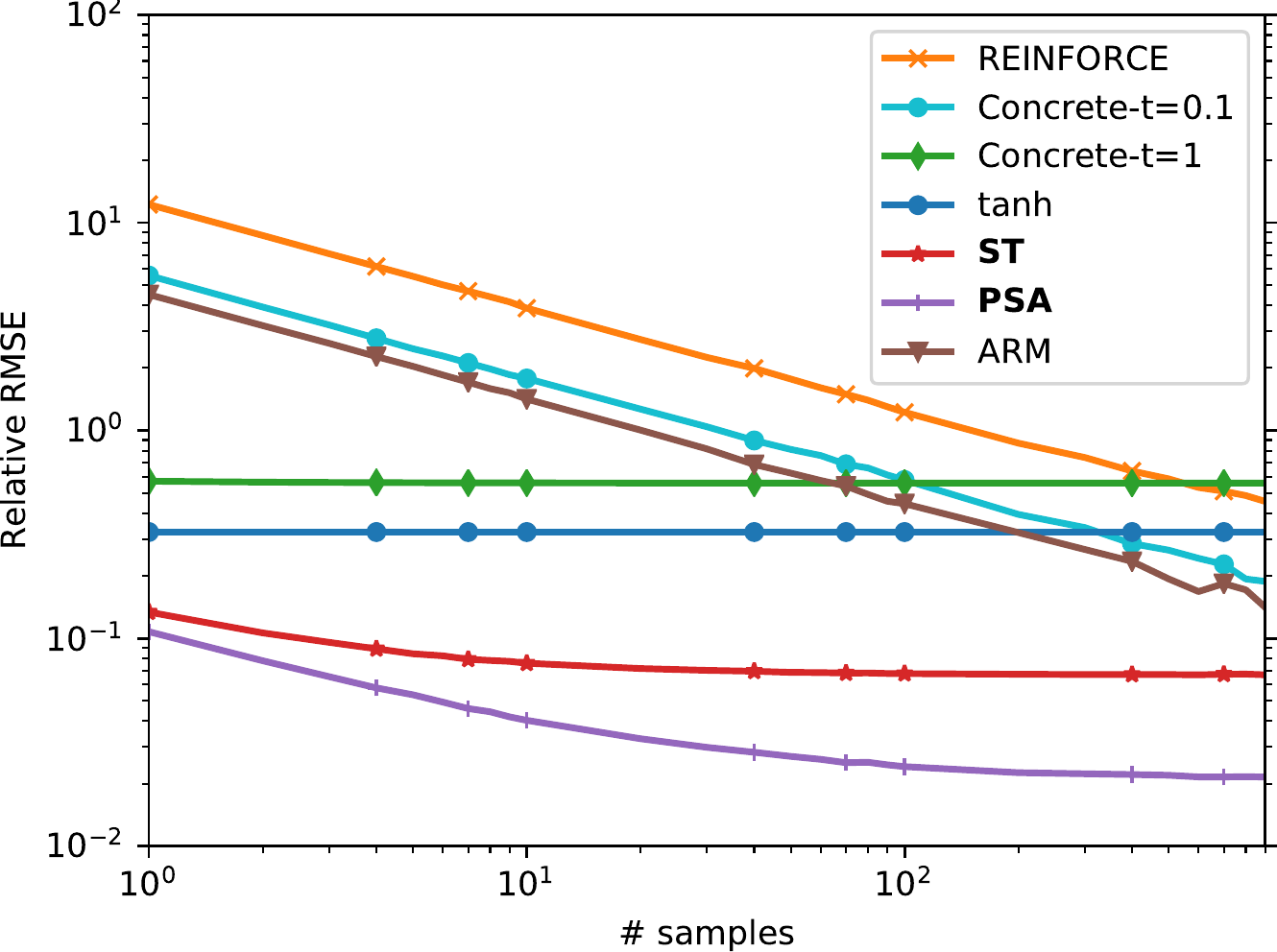}
\end{tabular}
\caption{Dependence on the network width. Shown gradient RMES in Layer 1 after epoch 1 (close to randomly initialized network). Bias of PSA and ST is more prominent in small models.\label{fig:acc-width}
}
\end{figure*}

\begin{figure*}[t]
\centering
\small
\begin{tabular}{ccc}
\small 2 layers &  \small 5 layers & \small 7 layers \\  
\includegraphics[width=0.33\linewidth]{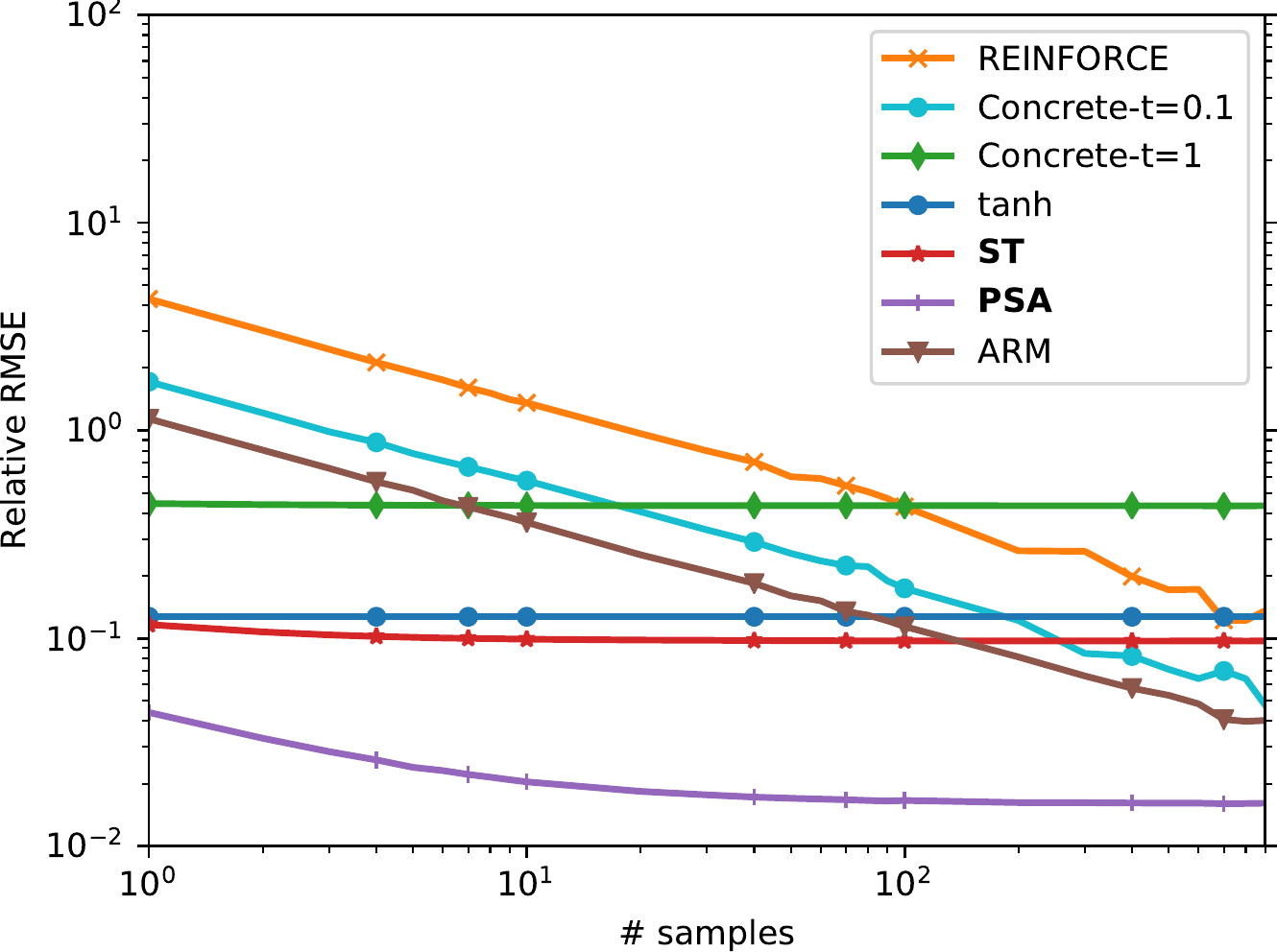}&
\includegraphics[width=0.33\linewidth]{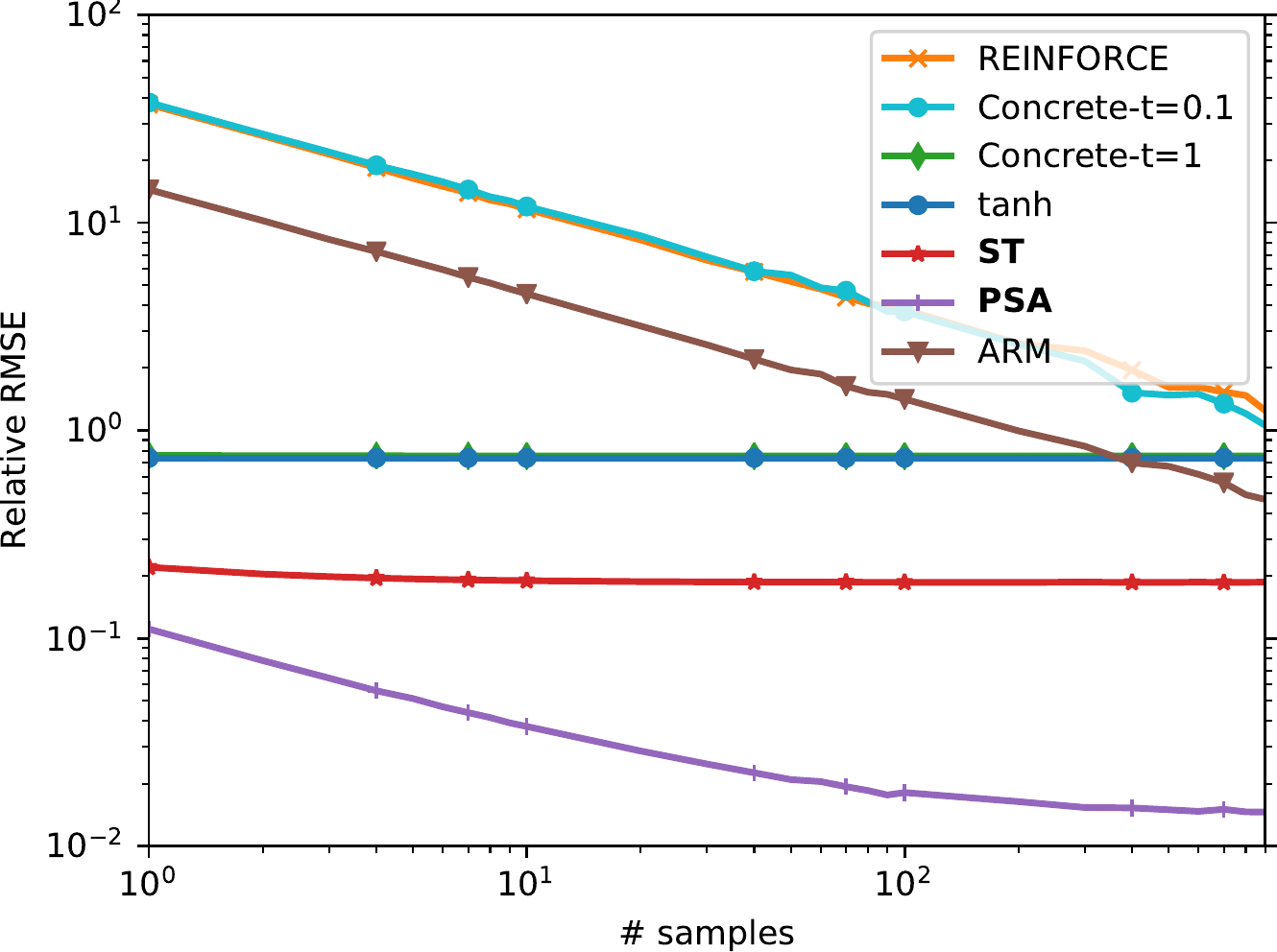}&
\includegraphics[width=0.33\linewidth]{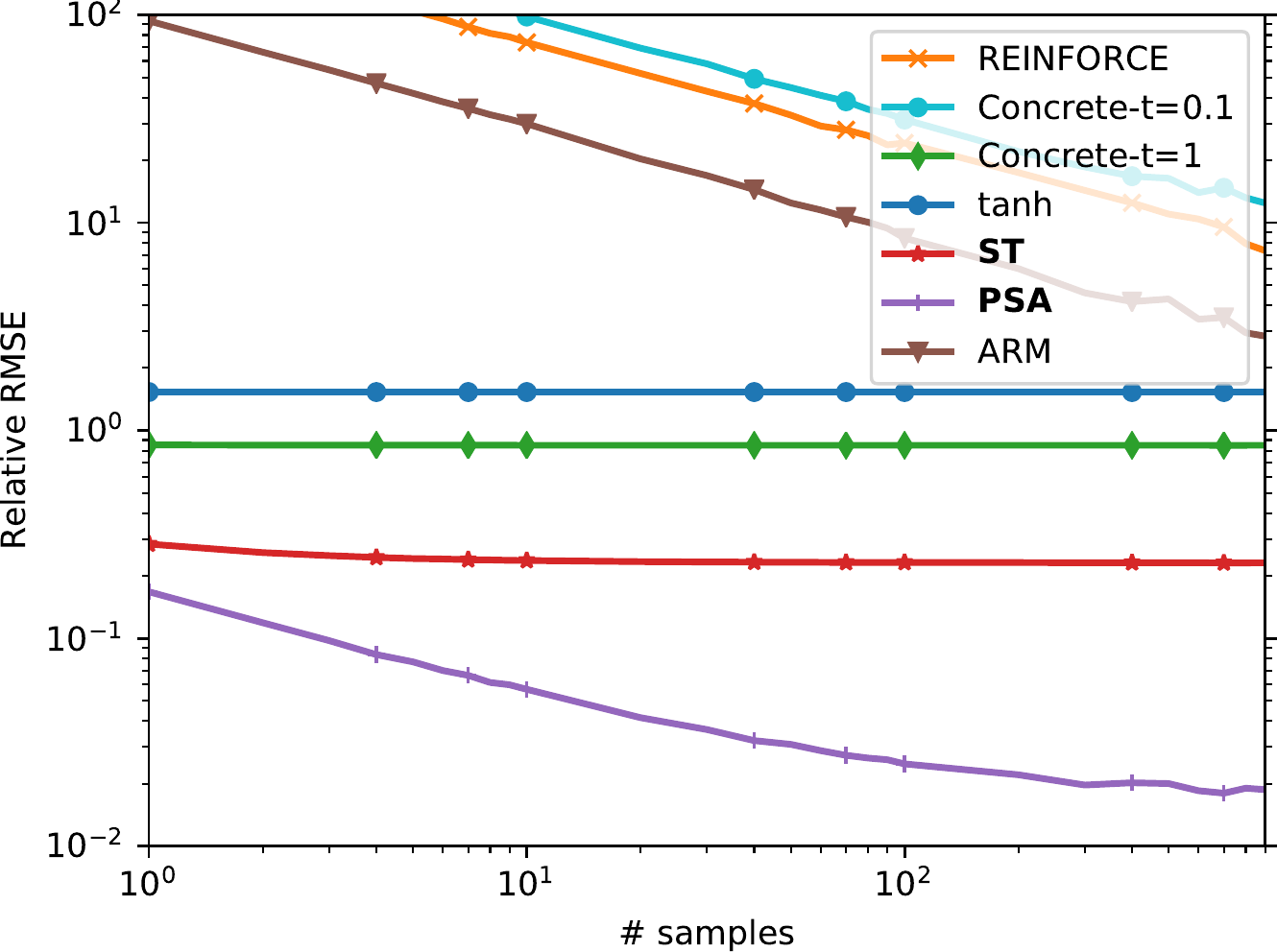}
\end{tabular}
\caption{\label{fig:acc-depth}
Dependence on the network depth. Shown gradient RMES in Layer 1 after epoch 1 (close to randomly initialized network). The variance of unbiased estimators grows much faster when increasing the network depth. The network width is fixed to 5 units per layer.
}
\end{figure*}

\begin{figure*}[t]
\small
\setlength{\tabcolsep}{0pt}
\setlength{\figwidth}{0.7\textwidth}
\centering
\begin{tabular}{cc}
& \small Flip \& Crop Augmentation \\
\parbox{1em}{\rotatebox[origin=c]{90}{\small Training Loss}}&
\begin{tabular}{c}
\includegraphics[width=\figwidth]{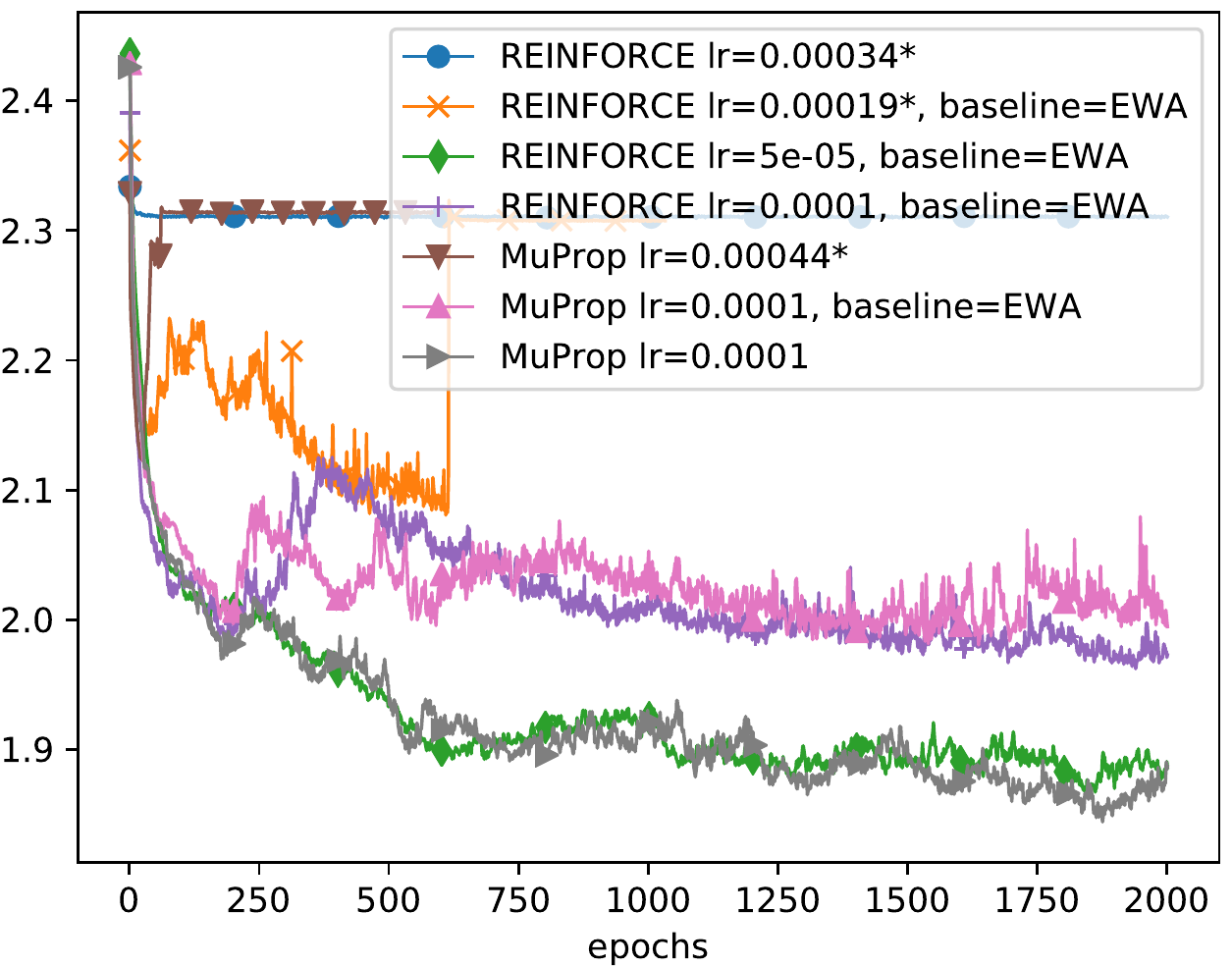}%
\end{tabular}
\end{tabular}
\caption{\label{fig:muprop}
Performance of \REINFORCE and MuProp with different learning rates and input-dependnet constant baselines estimated using a running averages.
No smoothing is applied across epochs in this plot. Note the limits of $y$-axis in comparison to~\cref{fig:accuracy}.
}
\end{figure*}

\end{document}